\documentclass[journal]{IEEEtran}  
 \usepackage{colortbl}
 \usepackage{tabulary}
\usepackage{mathtools}
\usepackage{pkgs/commath}

\usepackage{times} 
\usepackage{subfigure}

\usepackage{graphics} 
\usepackage{epsfig} 
\usepackage{mathtools} 
\usepackage[utf8]{inputenc} 
\usepackage{mleftright}
\usepackage{colortbl} 

\usepackage{array}
\usepackage{booktabs}
\newcommand{\PreserveBackslash}[1]{\let\temp=\\#1\let\\=\temp}
\newcolumntype{C}[1]{>{\PreserveBackslash\centering}m{#1}}
\newcolumntype{R}[1]{>{\PreserveBackslash\raggedleft}p{#1}}
\newcolumntype{L}[1]{>{\PreserveBackslash\raggedright}p{#1}}

\usepackage{graphicx}
\usepackage{graphics} 
\usepackage{epsfig}
\usepackage{algorithm}
\usepackage{color} 
\usepackage{multirow}
\usepackage{url}
\usepackage{bm} 
\usepackage{float} 
\usepackage{tikz}
\usetikzlibrary{positioning}
\usetikzlibrary{shapes,snakes}
\usepackage{empheq} 
\usepackage{xcolor} 
\usepackage{arydshln} 
\usepackage{longtable} 

\usepackage[format=plain,justification=justified]{caption}
\usepackage{fancyhdr}

\usepackage{url}
\usepackage[colorlinks,bookmarksopen,bookmarksnumbered,citecolor=red,urlcolor=red]{hyperref}
\makeatletter
\def\UrlAlphabet{%
	\do\a\do\b\do\c\do\d\do\e\do\f\do\g\do\h\do\i\do\j%
	\do\k\do\l\do\m\do\n\do\o\do\p\do\q\do\r\do\s\do\t%
	\do\u\do\v\do\w\do\x\do\y\do\z\do\A\do\B\do\C\do\D%
	\do\E\do\F\do\G\do\H\do\I\do\J\do\K\do\L\do\M\do\N%
	\do\O\do\P\do\Q\do\R\do\S\do\T\do\U\do\V\do\W\do\X%
	\do\Y\do\Z}
\def\UrlDigits{\do\1\do\2\do\3\do\4\do\5\do\6\do\7\do\8\do\9\do\0}
\g@addto@macro{\UrlBreaks}{\UrlOrds}
\g@addto@macro{\UrlBreaks}{\UrlAlphabet}
\g@addto@macro{\UrlBreaks}{\UrlDigits}
\makeatother

\definecolor{lightgray}{gray}{0.9}

\usepackage{cite,bm} 

\usepackage{amsmath}  
\allowdisplaybreaks[4]
\usepackage{amssymb} 
\usepackage{amsfonts}  
\usepackage{amsthm}
\newtheorem{thm}{Theorem}

\newtheorem{lem}[thm]{Lemma}
\usepackage{subfloat}

\newcommand{\clr}[1]{%
	\underline{#1}%
}

\begin{document}
	
\title{Visual-based Kinematics and Pose Estimation for Skid-Steering Robots}
	
\author{Xingxing Zuo$^{1,\dagger}$,  Mingming Zhang$^{2,\dagger}$, Mengmeng Wang$^3$, Yiming Chen$^2$,\\ Guoquan Huang$^4$, Yong Liu$^{3,\ast}$, and Mingyang Li$^{2,\ast}$
\thanks{$1$ Xingxing Zuo is with the Department of Informatics, Technical University of Munich, 80333 Munich, Germany. This work was done when Xingxing Zuo was a research intern in Alibaba Group.}
\thanks{$2$ Mingming Zhang, Yiming Chen, Mingyang Li are with Alibaba Group, Hangzhou 311121, China}
\thanks{$3$ Mengmeng Wang, Yong Liu are with the Institute of Cyber-Systems and Control, Zhejiang University, Hangzhou 310027, China.}
\thanks{$4$ Guoquan Huang is with the Department of Mechanical Engineering, University of Delaware, Newark,   DE 19716, USA.}
\thanks{$\dagger$ denotes equal contribution.}
\thanks{$\ast$ Mingyang Li and Yong Liu are the corresponding authors. (Email: {\tt\small yongliu@iipc.zju.edu.cn, mingyangli@alibaba-inc.com})}
}

\maketitle

\begin{abstract}
To build commercial robots, skid-steering mechanical design is of increased popularity due to its manufacturing simplicity and unique mechanism. However, these also cause significant challenges on software and algorithm design, especially for the pose estimation (i.e., determining the robot's rotation and position) of skid-steering robots, since  they change their orientation with an inevitable skid. To tackle this problem, we propose a probabilistic sliding-window estimator dedicated to skid-steering robots, using measurements from a monocular camera, the wheel encoders, and optionally an inertial measurement unit (IMU). Specifically, we explicitly model the kinematics of skid-steering robots by both track instantaneous centers of rotation (ICRs) and correction factors, which are capable of compensating for the complexity of track-to-terrain interaction, the imperfectness of mechanical design, terrain conditions and smoothness, etc. To prevent performance reduction in robots' long-term missions, the time- and location- varying kinematic parameters are estimated online along with pose estimation states in a tightly-coupled manner. More importantly, we conduct in-depth observability analysis for different sensors and design configurations in this paper, which provides us with theoretical tools in making the correct choice when building real commercial robots. In our experiments, we validate the proposed method by both simulation tests and real-world experiments, which demonstrate that our method outperforms competing methods by wide margins.

\bfseries\textit{Note to Practitioners}--- This paper was motivated by the problem of long-term pose estimation of the commonly commercial-used skid-steering robots with only low-cost sensors.  Skid-steering robots change their orientation with a skid, which poses a significant challenge for pose estimation when using the wheel encoders.  We propose to online estimate the robot's kinematics, which succeeds in compensating for the complexity of track-to-terrain interaction, due to the slippage, the imperfectness of mechanical design, terrain conditions and smoothness. It is critical to estimate the kinematics and poses jointly to prevent performance reduction in robots' long-term missions. We further theoretically analyze whether the kinematics parameters can be estimated under different sensor configurations, and find out the special degrade motions that make the parameters unobservable.
\end{abstract}


\IEEEpeerreviewmaketitle
\section{Introduction}
\label{sec:intro}
\begin{figure*}[tb]
	\centering
	\subfigure[]{ 
		\includegraphics[width=0.7\columnwidth]{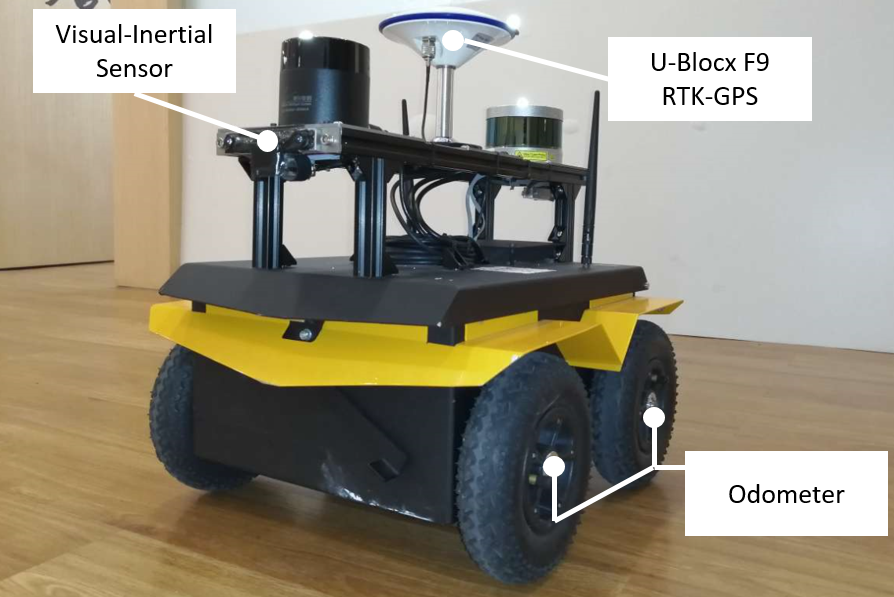} 
	} 
	\hfill
	\subfigure[]{ 
		\includegraphics[width=0.7\columnwidth]{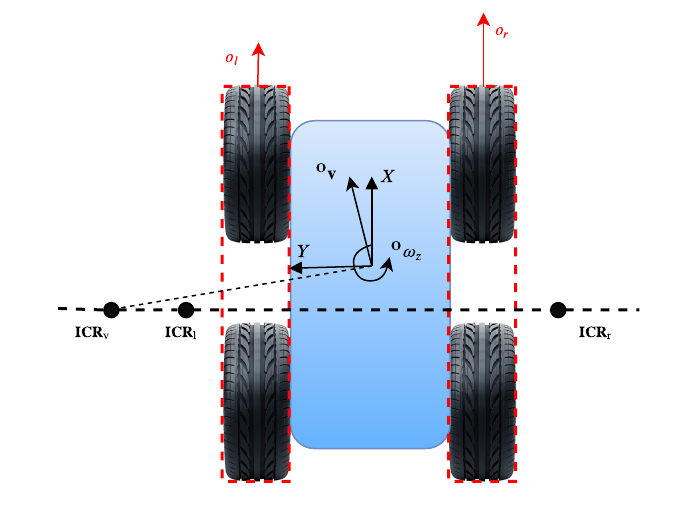} 
	}
	\captionsetup{justification=justified} 
	\caption{The skid-steering robotic platform used in our tests, as well as the corresponding kinematic model~\cite{zuo2019visual}. (a) Our testing robot, built on the Clearpath Jackal Platform~\cite{Clearpath}. The equipped low-cost sensors (i.e., a monocular camera, an imu, and wheel encoders) are leveraged in the proposed system, while the others (i.e., LiDAR and RTK-GPS) are not required in our system. (b) The odometer measurements and the instantaneous center of rotation (ICR, denoted by $[\mathrm{\mathbf{ICR}}_v, \mathrm{\mathbf{ICR}}_l, \mathrm{\mathbf{ICR}}_r]$) of a skid-steering robot. ${}^\mathbf{O}\mathbf{v}$ represents the robot velocity in odometer frame, and ${}^\mathbf{O}\omega_z$ is the angular velocity along the yaw direction.}
	\label{fig:robots_icr}
\end{figure*}
\IEEEPARstart{I}{n} recent years, the robotic community has witnessed a growing `go-to-market' trend, by not only building autonomous robots for scientific laboratory usage but also making commercial robots to create new business model and 
facilitate people's daily lives.
To date, a large amount of commercial outdoor robots, under either daily business usage or active trial operations and tests,  
are customized {\em skid-steering} robots ~\cite{skid-steer-loader-market,fedex-robot,amazon-robot}.
Instead of having an explicit mechanism of steering control, skid-steering robots rely on adjusting the speed of the left and right tracks to turn around. The simplicity of the mechanical design and the property of being able to turn around with zero-radius make skid-steering robots widely used in both the scientific research community as well as the commercial robotic industry. However, the mechanical simplicity of skid-steering robots has significantly challenged the software and algorithm design in robotic artificial intelligence, especially in autonomous localization~\cite{anousaki2004dead,martinez2005approximating,mandow2007experimental, yi2009kinematic,pentzer2014model,martinez2017inertia,sutoh2018motion,wang2018terrain}.

The localization system provides motion estimates, which is a key component for enabling any autonomous robot. To localize skid-steering robots, there is a large body of relevant literature~\cite{anousaki2004dead,martinez2005approximating,mandow2007experimental, yi2009kinematic,pentzer2014model,martinez2017inertia,sutoh2018motion,wang2018terrain}.
Early work by Anousaki et al.~\cite{anousaki2004dead} showed that the standard differential-drive two-wheel vehicle model could not be used to accurately model the motion of a skid-steering robot due to track and wheel slippage. To address this problem, Mart{\'\i}nez eat al.~\cite{martinez2005approximating} proposed an approach to approximate the kinematics of skid-steering robots based on instantaneous centers of rotation (ICRs).
Although some other kinematic models of skid-steering robots are also proposed~\cite{solc2008kinetic,wong2008theory,martinez2017inertia,sutoh2018motion}, ICR based kinematics is still popular due to its simplicity and feasibility~\cite{yi2009kinematic,pentzer2014model,wang2018terrain}, especially for real-time robotic applications.
In~\cite{yi2009kinematic}, IMU readings and wheel encoder measurements are fused in an EKF-based motion-estimation system for skid-steered robots. The ICR-based kinematics is utilized to compute virtual velocity measurements for robot motion estimation. However, the easily changed ICR parameters are not estimated in the estimator, which may lead to performance reduction. 
In~\cite{pentzer2014model}, ICR parameters and navigation states are estimated online in an EKF-based estimator for 3-DoF motion estimation of skid-steering robots. Wheel odometer measurements and GPS measurements are fused in the system.
It is also revealed that ICR-based kinematics is only valid in low dynamics, and will have a degraded performance when the vehicle is operated at a high speed.
The work~\cite{wang2018terrain} leveraged ICR-based kinematics and fused the readings from wheel encoders and a GPS-compass integrated sensor to estimate the ICR parameters and 3-DoF poses of the robot. 

In contrast to the above existing works, we estimate the kinematics (formulated by ICRs and correction factors),  and {\em full 6-DoF} poses (3-DoF rotations and 3-DoF translations) of the skid-steering robots jointly in a sliding-window bundle adjustment (BA) based estimator. Extracted visual features from camera and wheel encoder readings, optional IMU readings, are fused to optimize the estimated states in a tightly-coupled way and track the states of the skid-steering robot during its long-term mission.
In skid-steering robots, the track-to-terrain interaction is exceptionally complicated, and the conversion between wheel encoder readings and robot's motion depends on mechanical design, wheel inflation conditions, load and center of mass, terrain conditions, slippage, etc. In the long-term mission of the skid-steering robots, such as delivery, the kinematic parameters can be inevitably changed. Thus we estimate the kinematic parameters online to guarantee accurate pose estimation of the robots in complicated environments without performance reduction.

For a complicated estimator, it is critical to conduct observability analysis~\cite{huang2010observability,hesch2013towards,li2013high,pentzer2014model,li2014onlineabc,zuo2019visual} to study the identifiability of estimated states. 
Pentzer et al.~\cite{pentzer2014model} investigated 
the conditions that ICR parameters will be updated in a GPS-aided localization system, by demonstrating that the ICR parameters can be only updated when the robot is turning. However, this is just a glimpse of the observability property. The nature of the GPS measurements and the applicability of that algorithm are fundamentally different from our visual-based systems.
This paper is \textit{evolved} from our previous conference paper~\cite{zuo2019visual}, which performed observability analysis of localizing steering skid robot by using a monocular camera, wheel encoders, and an IMU, and showed that the skid-steering parameters are generally observable.
In this work, we extensively extend~\cite{zuo2019visual} and fully explore the observability properties of the visual-based kinematics and pose estimation system, by explicitly identifying the identifiable and non-identifiable parameters with and without using the IMU. Besides, we further investigate the suitability of kinematics estimation with joint online sensor extrinsic calibration, which is commonly required in state estimation systems with sensor fusion.
The results are with significant differences from the properties of other visual localization systems~\cite{li2014onlineabc, qin2018online,schneider2019observability}. This emphasizes the importance of the observability analysis in this paper. 

In summary, we focus on pose and kinematics estimation of steering-skid robots to enable their long-term mission in complicated environment, by using measurements from a monocular camera, wheel encoders, and optionally an IMU. The main contributions are as follows:
\begin{itemize}
	\item A visual-based estimator dedicated to skid-steering robots, which jointly estimates the ICR-based kinematic parameters of the robotic platform and 6-DoF poses in a tight-coupled manner. The formulation, error state propagation, and initialization of the kinematic parameters are presented in detail.
	
	\item Detailed observability analysis of the estimator under different sensor configurations, and the key results are as follows:
	(\romannumeral1) by using a monocular camera and wheel encoders, only the three ICR kinematic parameters are observable; (\romannumeral2) by introducing the additional IMU measurements, both the three ICR kinematic parameters and the two correction factors are observable under general motion; and (\romannumeral3) the 3-DoF extrinsic translation between the camera and odometer are unobservable with the online estimate of kinematic parameters, which prevents performing online sensor-to-sensor extrinsic calibration.  
	\item Extensive experiments including both simulation tests and real-world experiments  were conducted for evaluations. Ablation study is also investigated to standout the feasibility of the proposed method. In general, the proposed method i) shows high accuracy and great robustness under different environmental and mechanical conditions to enable the long-term mission of the robots and (ii) outperforms the competing methods that do and do not estimate the kinematics.
\end{itemize}

The rest of the paper is organized as follows. We introduce the kinematics model of skid-steering robots in Sec.~\ref{sec:icr-kinematics}. Subsequently, the framework of the tightly-coupled sliding-window estimator is introduced in Sec.~\ref{sec:estimator}, and we illustrate the kinematics estimation in detail. The observability analysis of the estimator under different configurations is performed in Sec.~\ref{sec:obser}. Experimental results are presented in Sec.~\ref{sec:exp}. Finally, the paper is concluded in Sec.~\ref{seq:conclusion}.

\section{ ICR-based Kinematics of Skid-Steering Robots}\label{sec:icr-kinematics}
\subsection{Notations}
In this paper, we consider a robotic platform navigating with respect to a global reference frame, $\lbrace \mathbf{G} \rbrace$. The platform is equipped with a camera, an IMU, and wheel odometers, whose frames are denoted by $\lbrace \mathbf{C} \rbrace$, $\lbrace \mathbf{I} \rbrace$, $\lbrace \mathbf{O}\rbrace$ respectively. To present transformation, we use $^{\mathbf A}\mathbf p_{\mathbf B}$ and 
$^{\mathbf A}_{\mathbf B} \mathbf R$ to denote position and rotation of frame $\lbrace\mathbf{B}\rbrace$ with respect to $\lbrace\mathbf{A}\rbrace$, and  $^{\mathbf A}_{\mathbf B}\mathbf{q}$ is the corresponding unit quaternion of $^{\mathbf A}_{\mathbf B} \mathbf R$. In addition, 
$\mathbf{I}$ denotes the identity matrix, and $\mathbf{0}$ denotes the zero matrix. 
We use $\hat{\mathbf x}$ and $\delta \mathbf x$ to represent the current estimated value and error state for variable $\mathbf x$.
Additionally, we reserve the symbol $\breve{z}$ to denote the inferred measurement value of $z$, which is widely used in observability analysis.
For the rotation matrix $^{\mathbf G}_{\mathbf O} \mathbf R$, we define the attitude error angle vector $\delta \boldsymbol{\theta}$ as follows~\cite{trawny2005indirect}:
\begin{align}
^{\mathbf G}_{\mathbf O} \mathbf R = {}^{\mathbf G}_{\mathbf O} \hat{\mathbf R} \left( \mathbf{I} + \lfloor \delta \boldsymbol{\theta} \rfloor \right)
\end{align}
We use $\lfloor \mathbf{v} \rfloor$ to denote the skew-symmetric matrix of a 3d vector $\mathbf{v}$:
\begin{align}
\lfloor \mathbf{v} \rfloor = 
\begin{bmatrix}
0 & -v_3 & v_2 \\
v_3 & 0 & -v_1 \\
-v_2 & v_1 & 0
\end{bmatrix}, 
\mathbf{v} = 
\begin{bmatrix}
v_1 \\
v_2 \\
v_3
\end{bmatrix}, 
\end{align}
\subsection{ICR-based Kinematics}
In order to design a general algorithm to localize skid-steering robots under different conditions, the corresponding kinematic models must be presented in a parametric format. 
In this work, we employ a model similar to the ones in~\cite{martinez2005approximating,pentzer2014model}, which contains five kinematic parameters: three ICR parameters and two correction factors, as shown in Fig.~\ref{fig:robots_icr}.
To describe the details,
we denote $\mathrm{\mathbf{ICR}}_v = \left( X_v, Y_v \right)$  the ICR position of the robot frame, and 
$\mathrm{\mathbf{ICR}}_l = \left( X_l, Y_l \right)$ and $\mathrm{\mathbf{ICR}}_r = \left( X_r, Y_r \right)$  the ones of the left and right wheels, respectively.
The relation between  the readings of wheel encoder measurements and the ICR parameters can be derived as follows:
\begin{align}
Y_l & =  - \frac{ o_l  - {}^{\mathbf  O}v_{x}}{{}^{\mathbf  O}\omega_z},\,\,\,Y_r  =  - \frac{ o_r  - {}^{\mathbf  O}v_{x}}{{}^{\mathbf  O}\omega_z} \notag \\
Y_v & =  \frac{{}^{\mathbf  O}v_{x}}{{}^{\mathbf  O}\omega_z},\,\,\,
X_v  = X_l = X_r=  - \frac{{}^{\mathbf  O}v_{y}}{{}^{\mathbf  O}\omega_z}
\label{eq:ICR}
\end{align}
where $o_l$ and $o_r$ are linear velocities of left and right wheels, ${}^{\mathbf  O}v_{x}$ and ${}^{\mathbf  O}v_{y}$ are robot's linear velocity along $x$ and $y$ axes represented in frame $\mathbf{O}$ respectively, and ${}^{\mathbf  O}\omega_z$ denotes the rotational rate about yaw also in frame $\mathbf{O}$. 
Those variables are also visualized in Fig.~\ref{fig:robots_icr}, 
and we use $\boldsymbol \xi_{ICR} = [X_v,\, Y_l,\, Y_r]^\top$ to represents the set of ICR parameters.
Moreover, we have used two scale factors,  $\boldsymbol \xi_{\alpha}=\left[\alpha_l, \alpha_r\right]^\top$, to compensate for effects which might cause changes in scales of wheel encoder readings. Representative situations include tire inflation, changes of road roughness, varying load of the robot, etc. 
With the ICR parameters and correction factors being defined, the skid-steering kinematic model can be written as:
\begin{align}~\label{eq:ICRmat_5par}
\begin{bmatrix} {}^{\mathbf  O}v_{x} \\ {}^{\mathbf  O}v_{y} \\ {}^{\mathbf  O}\omega_z \end{bmatrix} \!=\!
g(\boldsymbol \xi, o_l, o_r) \!= \!
\frac{1}{\Delta Y}\! 
\begin{bmatrix}
- Y_r &  Y_l\\ 
X_v   & - X_v\\
-1   &   1
\end{bmatrix}\!
\begin{bmatrix}
\alpha_l &0 \\ 0 &\alpha_r
\end{bmatrix}\!
\begin{bmatrix}
o_l \\ o_r
\end{bmatrix}
\end{align}
with
\begin{align}
\label{eq:icr def0}
\boldsymbol \xi \!\!=\!\! 
\begin{bmatrix}
\boldsymbol \xi_{ICR}^\top \!&
\boldsymbol \xi_{\alpha}^\top \end{bmatrix} \!\!=\!\! \begin{bmatrix}
X_v \!& Y_l \!& Y_r \!& \alpha_l \!& \alpha_r \end{bmatrix}^\top,
{\Delta Y}  \!\!=\!\!   Y_l \!-\!  Y_r
\end{align}
where $\boldsymbol \xi$ is the entire set of kinematic parameters.

Interestingly, as a special configuration when 
\begin{align}
\label{eq:special}
\boldsymbol \xi = 
\begin{bmatrix}
0, 0.5 b, -0.5 b, 1, 1
\end{bmatrix}^T
\end{align}
with $b$ being the distance between left and right wheels, Eq.~\eqref{eq:ICRmat_5par} can be simplified as:
\begin{align}
\label{eq:ordinarymodel}
{}^{\mathbf  O}v_{x} = \frac{o_l + o_r}{2}, ~~{}^{\mathbf  O}\omega_z = \frac{o_r - o_l}{b},~~ {}^{\mathbf  O}v_{y} =0
\end{align}
This is exactly the kinematic model for a wheeled robot moving without slippage (i.e., an ideal differential drive robot kinematics), 
and used by most existing work for localizing wheeled robots~\cite{yap2011particle,wu2017vins,quan2018tightly}. 
However, in the case of skid-steering robots, if Eq.~\eqref{eq:ordinarymodel} is employed directly in a localizer,
the pose estimation accuracy will be significantly reduced due to the incorrect conversion between wheel encoder readings and robot's motion estimates (also see experimental results in Sec.~\ref{sec:exp}).

\section{Visual-Inertial Kinematics and Pose Estimation}
\label{sec:estimator}

In this paper, we utilize a sliding-window bundle adjustment (BA) based estimator
for the kinematics and pose estimation of skid-steering robots using a monocular camera, wheel encoders, and optionally an IMU. For presentation simplicity, in this section, we describe our estimator by explicitly considering using the IMU. When the IMU is not included in the sensor system, our presented estimator can be straightforwardly modified by simply deleting the IMU related components. The architecture of our sliding-window estimator follows the pose estimation method for ground robots~\cite{zhang2021pose}, where visual constraints, IMU constraints, motion manifold (the profile of ground surface) constraints, ideal differential drive model induced odometer constraints are formulated and iteratively optimized. Notably, kinematics estimation is not touched in~\cite{zhang2021pose}, while it is the focus of this work.  
We also note that, compared to~\cite{zhang2021pose} and the other papers that focus on estimator architecture novelty, this work introduces methods to systematically handle skid-steering effects via online estimation. Our goal is to consistently and accurately estimate the motion of a moving robot as well as necessary observability-guided kinematic parameters of the robots.

\subsection{Estimator Formulation}
\subsubsection{State Vector}
To start with, we define the state vector of our estimator as:
\begin{align}
\label{eq:statevec}
\mathbf{x} =  \begin{bmatrix}
\boldsymbol{\chi}_{\mathbf{O}}^\top, {}^{\mathbf{G}}\mathbf{v}_{\mathbf{I}_k}^\top, \mathbf{b}_{a_k}^\top, \mathbf{b}_{\omega_k}^\top,
\mathbf{m}_k^\top,
\boldsymbol{\xi}_k^\top
\end{bmatrix}^\top
\end{align}
where
\begin{align} 
\boldsymbol{\chi}_{\mathbf{O}} \!\!=\!\! \begin{bmatrix}\!
{}^{\mathbf{G}}_{\mathbf{O}_{k-s}}\mathbf{q}^\top\!,\! {}^{\mathbf{G}}\mathbf{p}_{\mathbf{O}_{k-s}}^\top\!,\!\!~\dots~\!,\! {}^{\mathbf{G}}_{\mathbf{O}_{k-1}}\mathbf{q}^\top\!,\! {}^{\mathbf{G}}\mathbf{p}_{\mathbf{O}_{k-1}}^\top\!,\! {}^{\mathbf{G}}_{\mathbf{O}_{k}}\mathbf{q}^\top\!,\!{}^{\mathbf{G}}\mathbf{p}_{\mathbf{O}_{k}}^\top\!
\end{bmatrix}^\top
\end{align} 
denotes the sliding-window poses of odometer frame at times $\lbrace k-s, \dots, k\rbrace$ when keyframe images are captured.
$
{}^{\mathbf{G}}\mathbf{v}_{\mathbf{I}_k}, \mathbf{b}_{a_k}, \mathbf{b}_{\omega_k}$ are the IMU related states, including the IMU velocity in global frame, accelerometer bias, and gyroscope bias. 
If IMU is not available in the system, ${}^{\mathbf{G}}\mathbf{v}_{\mathbf{I}_k}, \mathbf{b}_{a_k}, \mathbf{b}_{\omega_k}$ will excluded from the state vector.
In addition, 
$\mathbf{m_k}$ denotes the parameters for modeling the local motion manifold of the skid-steering robots across current sliding window. The motion manifold is parameterized by a quadratic polynomial, and $\mathbf{m_k}$ is the 6-dimensional vectors to formulate the polynomial. This has been shown in~\cite{zhang2019v,zhang2021pose} to improve the estimation performance for ground robots, and we also adopt this design in our work. 
Finally, $\boldsymbol{\xi}_k$, as shown in Eq.~\eqref{eq:icr def0}, represents the skid-steering intrinsic parameter vector, which is
explicitly included in the state vector and thus estimated online.

\subsubsection{Bundle Adjustment Optimization}
Our optimization process follows the design of~\cite{zhang2021pose}. Specifically, 
the sliding-window bundle adjustment in our estimation algorithm seeks to iteratively minimize a cost function corresponding to a combination of sensor measurement constraints, kinematics constraints, and marginalized constraints.
\begin{align} \label{eq:cost}
\mathcal{C} = \mathcal{C}_{P} + \mathcal{C}_{V} + \mathcal{C}_{I} + \mathcal{C}_{O} +\mathcal{C}_{M} 
\end{align}

In what follows, we describe each of the cost terms. Firstly, 
the marginalized term $\mathcal{C}_{P}$ is critical to consistently keep the algorithm computational complexity bounded, by probabilistically removing the old states in the sliding window~\cite{Eckenhoff2019IJRR}.
The camera term $\mathcal{C}_{V}$, IMU term $\mathcal{C}_{I}$, and motion manifold term $\mathcal{C}_{M}$ used in this work are similar to that of existing literature~\cite{mourikis2007multi,li2013high,zhang2021pose}, respectively, but with dedicated design for ground robots. In general, the camera cost term models the geometrical reprojection error of point features in the keyframes (see Sec.~\ref{sec:camera cost}), the IMU term computes the error of IMU states between two consecutive keyframes, and the manifold cost term characterizes the motion smoothness across the whole sliding window.  Finally, $\mathcal{C}_{O}$ denotes 
the kinematic constraints induced by wheel odometer measurements. This term is a function of robot pose, measurement input, as well as skid-steering intrinsic kinematic parameters, and is discussed in detail in Sec.~\ref{sec:icr_pred}.
Due to limited space, we omit the details of $\mathcal{C}_{P}$,  $\mathcal{C}_{I}$, and $\mathcal{C}_{M}$ in this paper, and readers are encouraged to refer to the Appendix section. 

\subsection{Visual Constraints}\label{sec:camera cost}

In the sliding-window BA with visual constraints, the poses of only visual keyframes are optimized for computational saving. We use a simple heuristic for keyframe selection: the odometer prediction has a translation or rotation over a certain threshold (in all the experiments, 0.2 meter and 3 degrees). Since the movement form of the ground robot is simple, and it can be well predicted by the odometer in a short period of time. Unlike existing methods~\cite{qin2018vins,leutenegger2015keyframe}, which extract features and analyze the distribution of the features for keyframe selection, the non-keyframe will be dropped immediately without any extra operations in our framework. Among keyframes selected into the sliding-window, corner feature points are extracted in a fast way~\cite{rosten2006machine} and tacked with FREAK~\cite{alahi2012freak} descriptors.

The successfully tracked features across multiple keyframes will be initialized in the 3D space by triangulation.
By denoting $\mathbf{z}_{i,j}$ the visual measurement of a 3D feature ${}^{\mathbf G}\mathbf{p}_{f_j}$ observed by the $\mathbf{C}_i$th camera keyframe, the visual reprojection error~\cite{hartley2003multiple} in normalized image coordinate is given by:
\begin{align}\label{eq:repro err}
	&\mathcal{C}_{V}({}^{\mathbf G}_{\mathbf{O}_i}\mathbf{R}, {}^{\mathbf G}\mathbf{p}_{\mathbf{O}_i}) \!=\! \big{|}\big{|} \mathbf{z}_{i,j} \!- \! \pi( {}^{\mathbf G}_{\mathbf{C}_i}\mathbf{R}, {}^{\mathbf G}\mathbf{p}_{\mathbf{C}_i}, {}^{\mathbf G}\mathbf{p}_{f_j}) \big{|}\big{|}_{\bm \Lambda_{V}}^2, \notag \\ &{}^{\mathbf G}_{\mathbf{C}_i}\mathbf{R} = {}^{\mathbf G}_{\mathbf{O}_i}\mathbf{R} {}^{\mathbf O}_{\mathbf{C}}\mathbf{R}, ~~ {}^{\mathbf G}\mathbf{p}_{\mathbf{C}_i} = {}^{\mathbf G}\mathbf{p}_{\mathbf{O}_i} + {}^{\mathbf G}_{\mathbf{O}_i}\mathbf{R} {}^{\mathbf{O}}\mathbf{p}_{\mathbf{C}} 
\end{align}
In the above expressions, $\pi( \cdot )$ denotes the perspective function of an intrinsically calibrated camera, 
and ${\bm \Lambda_{V}}$ represents the inverse of noise covariance in the observation $\mathbf{z}_{i,j}$.
${}^{\mathbf O}_{\mathbf{C}}\mathbf{R}$ and ${}^{\mathbf{O}}\mathbf{p}_{\mathbf{C}}$ are the extrinsic transformation between camera and odometer. 
%
Furthermore, we choose not to incorporate visual features into the state vector due to the limited computational resources. Thus ${}^{\mathbf G}\mathbf{p}_{f_j}$ and the pose of oldest keyframe over the sliding window will be marginalized immediately after the iterative minimization.

\subsection{ICR-based Kinematic Constraints}
\label{sec:icr_pred}

This section provides details on formulating ICR-based kinematic constraints $\mathcal{C}_{O}$.
Specifically, 
by assuming the supporting manifold of the robot is locally planar between $t_k$ and $t_{k+1}$, 
the local linear and angular velocities, ${}^{\mathbf O{(t)}}{\mathbf{v}}$ and ${^{\mathbf O(t)} \boldsymbol \omega}$, 
are a function of the wheel encoders' measurements of the left and right wheels $o_{lm}(t)$ and $o_{rm}(t)$ 
as well as the skid-steering kinematic parameters $\boldsymbol \xi$ [see Eq.~\eqref{eq:ICRmat_5par}]:
\begin{align}
\label{eq:map to vw}
\begin{bmatrix}
{}^{\mathbf O_{(t)}}{\mathbf{v}} \\
{^{\mathbf O(t)} \boldsymbol \omega}
\end{bmatrix} &\!=\! \boldsymbol{\Pi} \, g(\boldsymbol{\xi}(t), o_l(t), o_r(t) ) \notag \\
&\!=\! \boldsymbol{\Pi} \,
g(\boldsymbol{\xi}(t), o_{lm}(t) \!- \!n_l(t), o_{rm}(t) \!- \!n_r(t) )
\end{align}
where $\boldsymbol{\Pi} = 
\begin{bmatrix}
\mathbf e_1 &
\mathbf e_2 &
\mathbf 0 &
\mathbf 0 &
\mathbf 0 &
\mathbf e_3
\end{bmatrix}^T$ is the selection matrix 
with    $\mathbf e_i$ being  a $3\times1$ unit vector with the $i$th element of 1,
$n_l(t)$ and $n_r(t)$ are the odometry noise  modeled as zero-mean white Gaussian. With slight abuse of notation, we define $ \mathbf{n}_o = \begin{bmatrix} n_l & n_r\end{bmatrix}^\top$.

By using ${}^{\mathbf O{(t)}}{\mathbf{v}}$ and ${^{\mathbf O(t)} \boldsymbol \omega}$, the wheel odometry based kinematic equations are given by:
\begin{subequations}
	\label{eq:odom predict evo}
	\begin{align}
	{}^{\mathbf G}\dot{\mathbf{p}}_{\mathbf O(t)} & = 
	{}^{\mathbf G}_{\mathbf O(t)}\mathbf{R} \cdot
	{}^{\mathbf O_{(t)}}{\mathbf{v}} \\
	{}^{\mathbf G}_{\mathbf O(t)}\dot{\mathbf{R}} &=
	{}^{\mathbf G}_{\mathbf O(t)}\mathbf{R} \cdot 
	\lfloor {^{\mathbf O(t)} \boldsymbol \omega} \rfloor \\
	\dot{\boldsymbol{\xi}}(t) & = \mathbf n_{ \xi}(t)
	\end{align}
\end{subequations}
where we model the noise of the ICR kinematic parameter $\boldsymbol \xi$ by using a random walk process, and $\mathbf n_{ \xi}$ is characterized by zero-mean white Gaussian noise. 
The motivation of using $\mathbf n_{ \xi}$ is to capture time-varying characteristics of $\boldsymbol \xi$, 
caused by changes in road conditions, tire pressures, center of mass. It is important to point out that, unlike sensor extrinsic calibration in which parameters can be modeled as constant parameters, e.g., $^C \dot{\mathbf p}_I = \mathbf 0$ in~\cite{kelly2011visual}, $\boldsymbol{\xi}$ must be modeled as a time-varying variable.

To propagate pose estimates in a stochastic estimator, we describe the process starting 
from the estimates $\hat{\mathbf{x}}_{ O_{k-1}} = \left[ {}^{\mathbf{G}}\hat{\mathbf{p}}^T_{\mathbf{O}_{k-1}},  {}^{\mathbf G}_{\mathbf{O}_{k-1}}\hat{ \mathbf{q}}^T, \hat{ \boldsymbol{\xi}}^T_{k-1} \right]^T$. Once the instantaneous local velocities of the robot (see Eq.~\eqref{eq:map to vw}) are available, we integrate the differential equations in Eq.~\eqref{eq:odom predict evo} over the time interval $t\in \left(t_{k-1}, t_{k} \right)$ by all the intermediate  odometer measurements $\mathcal{O}_m$, and obtain the predicted robot pose and kinematic parameters at the newest keyframe time $t_k$. This process can be characterized by $ \hat{\mathbf{x}}_{ O_k} = f \left(  {\hat{\mathbf{x}}}_{{O}_{k-1}}, \mathcal{O}_m  \right)$. 
Therefore, the odometer-induced kinematic constraint can be generically written in the following form:
\begin{align}
\label{eq:odom cost function}
& \mathcal{C}_{O}  (\mathbf{x}_{O_k}, \mathbf{x}_{O_{k-1}})= \big{|}\big{|} \mathbf{x}_{O_k} \boxminus f({\mathbf{x}}_{O_{k-1}},\mathcal{O}_m) \big{|}\big{|}_{\bm \Lambda_{O}}^2
\end{align}
%
%
where ${\bm \Lambda_{O}} $ represents the inverse covariance (information)  obtained via propagation process, and ``$\boxminus$" denotes the minus operation on manifold~\cite{barfoot2017state}. 
To formulate Eq.~\eqref{eq:odom cost function} in the stochastic estimator, error state characteristics also need to be computed since linearization of the propagation function $f$ consists of Jacobian matrices with respect to error states. We start with the continuous-time error state model, by linearizing Eq.~\eqref{eq:odom predict evo}:
\begin{align}
\dot{\delta {}^{\mathbf{G}} \mathbf p_{\mathbf O}}  &\simeq {}^{\mathbf{G}}_{\mathbf O}\hat{\mathbf R}  \left(\mathbf{I}+\lfloor\delta \boldsymbol{\theta}\rfloor \right) \left( {}^{\mathbf O } \hat{\mathbf{v}} + \mathbf{J}_{v \xi} \delta \boldsymbol{\xi}+ \mathbf{J}_{v o} \mathbf{n}_o \right) \!\! -\!\! {}^{\mathbf{G}}_{\mathbf O}\hat{\mathbf R} {}^{\mathbf O } \hat{\mathbf{v}} \notag  \\ 
&\simeq -{}^{\mathbf{G}}_{\mathbf O}\hat{\mathbf R} \lfloor {}^{\mathbf O } \hat{\mathbf{v}}  \rfloor {\delta \boldsymbol \theta} + {}^{\mathbf{G}}_{\mathbf O}\hat{\mathbf R} \mathbf{J}_{v \xi} \delta \boldsymbol{\xi} + {}^{\mathbf{G}}_{\mathbf O}\hat{\mathbf R} \mathbf{J}_{v o} \mathbf{n}_o
\\ 
\label{error evolution of theta}
\dot{\delta \boldsymbol \theta} &\simeq - \lfloor {}^{\mathbf O} \hat{\boldsymbol{\omega}} \rfloor {\delta \boldsymbol \theta} + \mathbf{J}_{\omega \xi} \delta \boldsymbol{\xi}+ \mathbf{J}_{\omega o} \mathbf{n}_o\\ 
\dot{\delta \boldsymbol \xi} &\simeq \mathbf{n}_{\xi}
\end{align}
%
We here point out that Eq.~\eqref{error evolution of theta} can be  obtained similar to Eq. (156) of \cite{trawny2005indirect}. 
In above equations, $\mathbf{J}_{v \xi}, \mathbf{J}_{\omega \xi}, \mathbf{J}_{v o}, \mathbf{J}_{\omega o}$ are the linearized Jacobian matrices, originated from:
\begin{align}
{}^{\mathbf O} {\mathbf{v}} &= {}^{\mathbf O } \hat{\mathbf{v}} + \mathbf{J}_{v \xi} \delta \boldsymbol{\xi}+ \mathbf{J}_{v o} \mathbf{n}_o \\
{}^{\mathbf O} {\boldsymbol{\omega}} &= {}^{\mathbf O } \hat{\boldsymbol{\omega}} + \mathbf{J}_{\omega \xi} \delta \boldsymbol{\xi}+ \mathbf{J}_{\omega o} \mathbf{n}_o 
\end{align} 
and
\begin{align}\setlength{\arraycolsep}{4.0pt}
\mathbf{J}_{v \xi} &= \frac{\hat{\alpha}_l o_l - \hat{\alpha}_r o_r }{\Delta \hat{Y} ^2}  \begin{bmatrix} 
0 & \hat{Y}_r & - \hat{Y}_l &  0& 0 \\
\Delta \hat{Y} &  -\hat{X}_v &  \hat{X}_v&  0& 0\\
0&0&0&0&0  \end{bmatrix} \notag\\
&~~~~~+  \frac{1}{\Delta \hat{Y}}  \begin{bmatrix} 
0 & 0& 0& - \hat{Y}_r o_l &  \hat{Y}_l o_r\\
0 & 0& 0&   \hat{X}_v o_l & - \hat{X}_v o_r \\
0 & 0& 0& 0& 0
\end{bmatrix} \\
\mathbf{J}_{v o} &= - \frac{1}{\Delta \hat{Y}} \begin{bmatrix}
-\hat{\alpha}_l \hat{Y}_r & \hat{\alpha}_r \hat{Y}_l \\
\hat{X}_v \hat{\alpha}_l & -\hat{X}_v \hat{\alpha}_r \\
0 & 0
\end{bmatrix} \\
\mathbf{J}_{\omega \xi} &= \frac{1}{\Delta \hat{Y} ^2} \begin{bmatrix} \begin{smallmatrix}
0 & 0& 0& 0& 0\\
0 & 0& 0& 0& 0\\
0& -\hat{\alpha}_r o_r + \hat{\alpha}_l o_l& \hat{\alpha}_r o_r - \hat{\alpha}_l o_l& - \Delta \hat{Y} o_l & \Delta \hat{Y} o_r 
\end{smallmatrix}\end{bmatrix}\\
\mathbf{J}_{\omega o} &= - \frac{1}{\Delta \hat{Y}} \begin{bmatrix}
0 & 0 \\
0 & 0 \\
-\hat{\alpha}_l & \hat{\alpha}_r \end{bmatrix}
\end{align}
Once continuous-time error-state equations are given, the discrete-time state transition matrices needed when minimizing Eq.~\eqref{eq:odom cost function}, can be straightforwardly calculated by numerical integration.
In general, Eq.~\eqref{eq:odom cost function} encapsulates all information related to the skid-steering effect and enables online estimation of the skid-steering parameters.
After the constraint in Eq.~\eqref{eq:odom cost function} is minimized, $\boldsymbol{\xi}_{k-1}$ will be marginalized immediately, ensuring low computational complexity of the system.

\subsection{Initialization of Kinematic Parameters}
\label{sec:init kinematic}

To allow the estimation of skid-steering kinematic parameters online, an initial estimate of the parameter vector $\boldsymbol \xi$ is required. 
Specifically, we use a simply while effective method by setting:
\begin{align}
\label{eq:special1}
\boldsymbol \xi_{initial} = 
\begin{bmatrix}
0, 0.5 b^{\dagger}, -0.5 b^{\dagger}, 1, 1
\end{bmatrix}^T
\end{align}
We emphasize that Eq.~\eqref{eq:special1} is similar but different from Eq.~\eqref{eq:special}. The parameter $b$ represents wheel distance in Eq.~\eqref{eq:special}, which can be correctly used for robot without slippage. However, skid-steering robots are designed to have slippery behaviors, and thus $b^{\dagger}$ should not be simply the wheel distance. To compute $b^{\dagger}$, we rotate the skid-steering robots and use the fact that rotational velocity reported by the IMU and odometry should be identical, which leads to the follow equation:
\begin{align}
b^{\dagger} = \frac{1}{N} \sum_{i=1}^N 
\frac{||o_{lm}(t_i) - o_{rm}(t_i)||}{
	|| \boldsymbol \omega_m(t_i)||}
\end{align}
where $\boldsymbol \omega_m(t_i)$ is gyroscope measurement, and $N$ is the number of measurements. Although this is not of high precision and the road condition of computing $b^{\dagger}$ is different from that of the testing time, this simple initialization method in combination of the proposed online calibration algorithm is able to yield accurate localization results (see our experimental results). 

\newcommand{\bO}{\mathbf{O}}
\newcommand{\bL}{\mathbf{L}}
\newcommand{\bT}{\mathbf{T}}
\newcommand{\bp}{\mathbf{p}}
\newcommand{\bl}{\mathbf{l}}
\newcommand{\ba}{\mathbf{a}}
\newcommand{\bomega}{\boldsymbol{\omega}}
\newcommand{\bo}{\mathbf{o}}
\newcommand{\bc}{\mathbf{c}}
\newcommand{\bz}{\mathbf{z}}
\newcommand{\bt}{\mathbf{t}}
\newcommand{\bQ}{\mathbf{Q}}
\newcommand{\bC}{\mathbf{C}}
\newcommand{\bA}{\mathbf{A}}
\newcommand{\bB}{\mathbf{B}}
\newcommand{\bG}{\mathbf{G}}
\newcommand{\bn}{\mathbf{n}}
\newcommand{\bI}{\mathbf{I}}
\newcommand{\bR}{\mathbf{R}}
\newcommand{\bD}{\mathbf{D}}
\newcommand{\bv}{\mathbf{v}}
\newcommand{\bJ}{\mathbf{J}}
\newcommand{\be}{\mathbf{e}}
\newcommand{\btheta}{\boldsymbol{\theta}}
\newcommand{\bepsilon}{\boldsymbol{\epsilon}}
\newcommand{\bgamma}{\boldsymbol{\gamma}}
\newcommand{\boldeta}{\boldsymbol{\eta}}
\newcommand{\bzeta}{\boldsymbol{\zeta}}

\section{Observability Analysis}
\label{sec:obser}
A critical prerequisite condition for a well-formulated estimator is to {\em only} include locally observable  (or identifiable\footnote{Since the derivative of $\boldsymbol \xi$ is modelled by zero-mean Gaussian, we here use observability and identifiability interchangeably.})~\cite{Bar-Shalom1988} states in the online optimization stage. In the skid-steering robot localization system, a subset of estimation parameters inevitably become unobservable under center circumstances, which will be analytically characterized in this section and avoided in a real-world deployment. 

Specifically, in this section, we first conduct our analysis by assuming the extrinsic parameters between sensors are perfectly known, and analyze the observability properties of the skid-steering parameters in different sensor system setup. Specifically, we consider three
cases: (\romannumeral1) monocular camera and odometer with the 3 ICR parameters and 2 correction factors; (\romannumeral2) monocular camera and odometer with 3 ICR parameter only; (\romannumeral3)  monocular camera, odometer and an IMU with 3 ICR parameters and 2 correction factors.
Subsequently, we perform the analysis under the case that extrinsic parameters between sensors are {\em unknown}. 
Since estimating extrinsic parameters online is a common estimator design choice in robotics community~\cite{guo2012analytical,heng2013camodocal,geiger2012automatic}, we also investigate the possibility of doing that for skid-steering robots.  
%

%

\subsection{Methodology Overview}
To investigate the observability properties, the analysis can be either conducted in the original nonlinear continuous-time system~\cite{li2013high} or the corresponding linearized discrete-time system~\cite{guo2013imu,yang2019observability}. 
As shown in~\cite{li2013high,hesch2013towards}, the dimension of the nullspace of the observability matrix might subject to changes due to linearization and dicrestization, and thus we conduct our analysis in the nonlinear continuous-time space in this work.

To conduct the observability analysis, we are inspired by~\cite{li2014onlineabc}, in which information provided by each sensor is investigated and subsequently combined together for deriving the final results. 
By doing this, `abstract' measurements instead of the `raw' measurements are used for analysis, which significantly simplifies our derivation and is helpful for intuitive understanding. 
Specifically, the observability analysis consists of three main steps. Firstly, we investigate the information provided by each sensor, and derive inferred `abstract' measurements from the raw measurements. Secondly, we use kinematic and measurement constraints to derive equations that indistinguishable trajectories must follow. Finally, the observability matrix is constructed by computing the derivatives of the previous derived equations with respect to the states of interests. The observability of the states can be determined by examining the rank and nullspace of the observability matrix~\cite{van2009identifiability}. 

\subsection{Inferred Measurement Model}
We first analyze the information provided by a monocular camera. It is well-known that
a monocular camera is able to provide information on rotation and up-to-scale position with respect to the initial camera frame~\cite{hartley2003multiple,li2014onlineabc} under general motion. Equivalently, the information characterized by a monocular camera can be given by: (\romannumeral1) camera's angular velocity  and (\romannumeral2) its up-to-scale linear velocity:
\begin{subequations}
	\label{eq:camera_measurement}
	\begin{align}
	\breve{{\bomega}}_C (t) &= {}^{\mathbf{C}_{(t)}}\bomega + \mathbf{n}_{\omega} (t)\\
	\breve{{\bv}}_C (t) &= s^{-1}\cdot  {}^{\mathbf C} _{\mathbf G} \mathbf R \cdot {}^{\mathbf{G}}\bv_{\mathbf{C}_{(t)}} + \mathbf{n}_v (t)
	\end{align}
\end{subequations}
where $\mathbf{n}_{\omega} (t)$ and $\mathbf{n}_{v} (t)$ are the measurement noises, ${}^{\mathbf{C}_{(t)}}{\bomega}$ denotes true local angular velocity expressed in camera frame, and ${}^{\mathbf{G}}\bv_{\mathbf{C}_{(t)}}$ is the linear velocity of camera with respect to global frame, and finally $s$ is an unknown scale factor. Additionally, $\breve{{\bomega}}_C (t)$ and $\breve{{\bv}}_C (t)$ denote the inferred rotational and linear velocity measurements.
Moreover, to make our later derivation simpler, we also introduce the rotated inferred measurements as follows:
\begin{align}
\label{eq:camera_measurement1}
\breve{{\bomega}}(t) \triangleq {}^{\mathbf O} _{\mathbf C} \mathbf R \cdot \breve{{\bomega}}_C(t),\,\,
\breve{{\bv}}(t) \triangleq {}^{\mathbf O} _{\mathbf C} \mathbf R \cdot  \breve{{\bv}}_C(t)
\end{align}
It is important to point out that in the cases when extrinsic parameter calibration between sensors is not considered in the online estimation stage, $\breve{{\bomega}} (t)$ and $\breve{{\bv}} (t)$ can be uniquely computed from the camera measurement and also treated as the inferred measurement.

\subsection{Observability of \texorpdfstring{$\boldsymbol \xi$}{} with Monocular Camera and Odometer}
\label{sec:mono_ICR_5}

We first investigate the case when a system is equipped with a monocular camera and odometers, and their extrinsic parameters are known in advance.
To perform observability analysis, we derive system equations that indistinguishable trajectories must satisfy. To start with, we note that the following geometric relationships hold for any camera-odometer system:
\begin{align}
\label{eq:angular_velocity_relation}
^{\mathbf{O}}\bomega & = {}^{\mathbf{O}}_{\mathbf{C}} \bR \cdot {}^{C}\bomega
\end{align}
which allows us to derive the following equations: 
\begin{subequations}
	\begin{align}
	^{\mathbf{G}} \bp _{\mathbf{O}}   &= 
	{} ^{\mathbf{G}} _{\mathbf{C}} \bR \cdot  ^{\mathbf{C}} \bp _{\mathbf{O}} + {}
	^{\mathbf{G}} \bp_{\mathbf{C}}\\
	^{\mathbf{G}} \bv _{\mathbf{O}}   &= ^{\mathbf{G}}\dot{\bp} _{\mathbf{O}} =
	{} ^{\mathbf{G}} _{\mathbf{C}} \bR 
	\lfloor {^{\mathbf{C}}\bomega}  \rfloor
	^\mathbf{C} \bp _{\mathbf{O}} + {}
	^{\mathbf{G}} \bv_{\mathbf{C}}\\
	{} ^{\mathbf{O}} _{\mathbf{G}} \bR {} ^{\mathbf{G}} \bv _{\mathbf{O}}   &= 
	{} ^{\mathbf{O}} _{\mathbf{C}} \bR 
	\lfloor {^{\mathbf{C}}\bomega}  \rfloor 
	^{\bC} \bp _{\bO} + {}
	^{\bO} _{\bC} \bR {} ^{\bC} _{\bG} \bR {} ^{\bG} \bv_{\bC}\\
	^{\bO}\bv  &= 
	-\lfloor { ^{\bO}\bomega} \rfloor
	^\bO \bp_{\bC}+ {}
	^{\bO} _{\bC} \bR {} ^{\bC} _{\bG} \bR  {}^\bG \bv _\bC
	\label{eq:velocity for obser}
	\end{align}
\end{subequations}
Substituting Eq.~\eqref{eq:camera_measurement1} to Eq.~\eqref{eq:velocity for obser}, we obtain the following equation:
\begin{align}
\label{eq:linear_velocity_relation}
^{\mathbf{O}}\bv  &=     -\lfloor{ { \breve\bomega}} \rfloor
^{\mathbf O} \bp_{\mathbf C}+  s \cdot \breve{\bv} 
\end{align}
where $^{\mathbf O} \bp_{\mathbf C}$ is known and $^{\mathbf{O}}\bv$ is velocity expressed in the odometer frame.
We also note that, during the observability analysis, the noise terms are ignored, following the standard procedure of performing the observability analysis.

On the other hand, as mentioned in Sec.~\ref{sec:icr-kinematics}, odometer provides observations for the speed of left and right wheels, i.e., $o_{l}$ and $o_{r}$ respectively. By linking $o_{l}$,  $o_{r}$, $\breve{{\bomega}}(t) $, $\breve{{\bv}}(t)$, and kinematic parameter vector $\boldsymbol \xi$ together, specifically substituting Eq.~\eqref{eq:linear_velocity_relation} into 
Eq.~\eqref{eq:ICRmat_5par}, we obtain:
\begin{align}\setlength{\arraycolsep}{4.0pt}
\begin{bmatrix}
\begin{bmatrix}
\breve{\omega} {}^{\mathbf O} y_{\mathbf C} \\
-\breve{\omega} {}^{\mathbf O} x_{\mathbf C} 
\end{bmatrix} \!+\! 
s \begin{bmatrix}
\breve{v} _{x} \\
\breve{v} _{y}
\end{bmatrix} \\ 
\breve{\omega}
\end{bmatrix} \!=\! \frac{1}{\Delta Y} \!
\begin{bmatrix}
- Y_r &  Y_l\\ 
X_v  & - X_v\\
-1 & 1
\end{bmatrix}\!
\begin{bmatrix}
\alpha_l &0 \\ 0 &\alpha_r
\end{bmatrix}\!
\begin{bmatrix}
o_l \\ o_r
\end{bmatrix} \notag \\
\!= \!\begin{bmatrix}
\breve{\omega} Y_l \\
\!-\! \breve{\omega} X_v \\
\frac{1}{\Delta Y}
\begin{bmatrix}
-1 & 1
\end{bmatrix} \!
\begin{bmatrix}
\alpha_l & 0 \\
0 & \alpha_r
\end{bmatrix}\!
\begin{bmatrix}
o_l \\ o_r 
\end{bmatrix} 
\end{bmatrix}
\!+\!
\begin{bmatrix}
\alpha_l o_l \\ 0 \\ 0
\end{bmatrix}
\label{eq:constraints for obser}
\end{align} 
where ${}^{\mathbf O} x_{\mathbf C}, {}^{\mathbf O} y_{\mathbf C}$ are the first and second element of ${}^{\mathbf O} \bp _{\mathbf C}$, and 
$\breve{v} _{x} , \breve{v} _{y} $ are the first and second element of $\breve{\bv}$. For brevity, we use $\breve{\omega}$ to denote the third element of $\breve{ \bomega}$.
By defining $\beta_r = \Delta Y^{-1} \alpha _r$, and
$\beta_l = \Delta Y^{-1} \alpha _l$,  we can write
\begin{align}
\label{eq:constraints_mono_5}
\begin{bmatrix}
\begin{bmatrix}
\breve{\omega} {}^{\mathbf O} y_{\mathbf C} \\
-\breve{\omega} {}^{\mathbf O} x_{\mathbf C} 
\end{bmatrix} + 
s \begin{bmatrix}
\breve{v} _{x} \\
\breve{v} _{y}
\end{bmatrix} \\ 
\breve{\omega}
\end{bmatrix} = \begin{bmatrix}
\breve{\omega} Y_l \\
- \breve{\omega} X_v \\
- \beta_l o_l + \beta_r o_r
\end{bmatrix}
\!+\!
\begin{bmatrix}
\beta_l \Delta Y o_l \\ 0 \\ 0
\end{bmatrix}
\end{align}
Note that, this equation only contains 1) sensor measurements, and 2) a combination of vision scale factor and skid-steering kinematics:
\begin{align}
\bepsilon = \begin{bmatrix}
X_v& Y_l& Y_r& \alpha_l& \alpha_r& s
\end{bmatrix}^\top \notag
\end{align}
which allows us to analyze whether indistinguishable sets of $\bepsilon$ exist subject to the provided measurement constraint equations.

The identifiability of $\epsilon$ can be described as follows:
\begin{lem}
	By using measurements from a monocular camera and wheel odometers, $\bepsilon$ is not locally identifiable.
\end{lem}
\begin{proof}
	$\bepsilon$ is locally identifiable if and only if $\bar{\bepsilon}$ is locally identifiable:
	\begin{align}
	\label{eq:simple}
	\bar{\bepsilon} = \begin{bmatrix} Y_l& \Delta Y&  X_v& \beta_l& \beta_r& s \end{bmatrix}^\top \notag
	\end{align} 
	By expanding Eq.~\eqref{eq:constraints_mono_5}, we can write the following constraints:
	\begin{subequations}
		\label{eq:constraints_mono_5_ex}
		\begin{align}
		c_x(\bar{\bepsilon}, t) &= \breve{\omega}(t) {}^{\mathbf O} y_{\mathbf C} + s \breve{v} _{x}(t) - \breve{\omega}(t) Y_l - \beta_l \Delta Y o_l(t) = 0\\
		c_y(\bar{\bepsilon}, t) &=  -\breve{\omega}(t) {}^{\mathbf O} x_{\mathbf C} + s \breve{v} _{y}(t) + \breve{\omega}(t) X_v = 0\\
		c_{\omega}(\bar{\bepsilon}, t) &= \breve{\omega}(t) + \beta_l o_l(t) - \beta_r o_r(t) =0
		\end{align}
	\end{subequations}
	A necessary and sufficient condition of $\bar{\bepsilon}$ to be locally identifiable is following observability matrix has full column rank, 
	over a set of time instants $\mathcal{S}  = \lbrace  t_0, t_1, \dots, t_s \rbrace$,:
	\begin{align}
	\label{eq:occ}
	\mathbf{M} = \begin{bmatrix}
	\mathbf{D}(t_0)^\top & \mathbf{D}(t_1)^\top& \dots& \mathbf{D}(t_s)^\top
	\end{bmatrix}^\top
	\end{align}
	where
	\begin{align}\setlength{\arraycolsep}{0.1pt}
	\mathbf{D}(t) &\!=\! \begin{bmatrix}
	\frac{\partial c_x(\bar{\bepsilon}, t)}{\partial \bar{\bepsilon}} & \frac{\partial c_y(\bar{\bepsilon}, t)}{\partial \bar{\bepsilon}} & \frac{\partial c_{\omega}(\bar{\bepsilon}, t)}{\partial \bar{\bepsilon}}
	\end{bmatrix} ^\top \notag \\
	\label{eq:d}
	&\!= \!\begin{bmatrix}\begin{smallmatrix}
	- \breve{\omega}(t)& -\beta_l o_l(t)& 0& -\Delta Y o_l(t)& 0& \breve{v} _{x}(t)\\
	0& 0& \breve{\omega}(t)& 0& 0& \breve{v} _{y}(t)\\
	0& 0& 0& o_l(t)& -o_r(t)& 0
	\end{smallmatrix}\end{bmatrix} 
	\end{align}
	%
	Substituting Eq.~\eqref{eq:d} back into Eq.~\eqref{eq:occ} leads to:
	\begin{align}\setlength{\arraycolsep}{4.0pt}
	\label{eq:obser_matrix_5icr}
	\mathbf{M} = \begin{bmatrix}\begin{smallmatrix}
	- \breve{\omega}(t_0)& -\beta_l o_l(t_0)& 0& -\Delta Y o_l(t_0)& 0& \breve{v} _{x}(t_0)\\
	0& 0& \breve{\omega}(t_0)& 0& 0& \breve{v} _{y}(t_0)\\
	0& 0& 0& o_l(t_0)& -o_r(t_0)& 0 \\
	\vdots& \vdots& \vdots& \vdots& \vdots& \vdots\\
	- \breve{\omega}(t_s)& -\beta_l o_l(t_s)& 0& -\Delta Y o_l(t_s)& 0& \breve{v} _{x}(t_s)\\
	0& 0& \breve{\omega}(t_s)& 0& 0& \breve{v} _{y}(t_s)\\
	0& 0& 0& o_l(t_s)& -o_r(t_s)& 0
	\end{smallmatrix}\end{bmatrix}
	\end{align}
	By defining  $\mathbf{M}(:,i)$ the $i$th block columns of $\mathbf{M}$, the following equation holds:
	\begin{align}
	& (-{}^{\mathbf O} y_{\mathbf C} \!+\! Y_l)\cdot\mathbf{M}(:,1) \!+\! \Delta Y \cdot \mathbf{M}(:,2)  \! \notag \\ &~~~ +\! (X_v - {}^{\mathbf O} x_{\mathbf{C}}) \cdot \mathbf{M}(:,3)\!  +\! s \cdot \mathbf{M}(:,6)= \mathbf 0 \notag
	\end{align}
	The above equation demonstrates that $\mathbf{M}$ is {\em not} of full column rank, indicating that $\bepsilon$ is not identifiable.
	
	To further investigate the indistinguishable states that cause the unobservable situations, we note that for a vector $\bar{\bepsilon}_1 = [ Y_l, \Delta Y,  X, \beta_l, \beta_r, s ]^\top$ that satisfies Eq.~\eqref{eq:constraints_mono_5_ex}, another vector
	\begin{align}
	\bar{\bepsilon}_2 = [ (1+\lambda/s)Y_l- (\lambda/s) {}^\bO y_\bC, (1+\lambda/s)\Delta Y, \notag \\~~~  (1+\lambda/s)X - (\lambda/s) {}^\bO x_\bC, \beta_l, \beta_r, s+\lambda ]^\top \notag
	\end{align}
	for any $\lambda \in \mathbb{R}$ is always valid for the constraints Eq.~\eqref{eq:constraints_mono_5_ex}. Thus $\bar{\bepsilon}_1$ and $\bar{\bepsilon}_2$ are indistinguishable, and $\bar{\bepsilon}$ is not locally identifiable. This completes the proof.
\end{proof}

\subsection{Observability of \texorpdfstring{$\boldsymbol \xi_{ICR}$}{} with Monocular Camera and Odometer}
Since the full kinematic parameters $\boldsymbol{\xi} = 
[\boldsymbol{\xi}_{ICR}^T,\boldsymbol{\xi_{\alpha}}^T]^T$ with monocular camera and odometer are not locally identifiable, we look into the case of that only the 3 ICR parameters, i.e., $\boldsymbol{\xi}_{ICR}$, are estimated without the correction factors. Similar to Eq.~\eqref{eq:constraints for obser}, the following equation holds:
\begin{align}
\label{eq:constraints_mono_3}
\begin{bmatrix}
\begin{bmatrix}
\breve{\omega} {}^\bO y_\bC \\
-\breve{\omega} {}^\bO x_\bC 
\end{bmatrix} + 
s \begin{bmatrix}
\breve{v} _{x} \\
\breve{v} _{y}
\end{bmatrix} \\ 
\breve{\omega}
\end{bmatrix} 
= \begin{bmatrix}
\breve{\omega} Y_l \\
- \breve{\omega} X_v \\
\frac{1}{\Delta Y}
\left(o_r - o_l\right) 
\end{bmatrix}
+
\begin{bmatrix}
o_l \\ 0 \\ 0
\end{bmatrix}
\end{align} 
The above expression is a function of the odometer and inferred visual measurements $\breve{\omega}, \breve{v} _{x}, \breve{v} _{y}, o_l, o_r$, as well as the kinematic intrinsic parameters $\boldsymbol{\xi}_{ICR}$ and visual scale factor $s$: \begin{align}
\bgamma = 
\begin{bmatrix}
\boldsymbol{\xi}^T_{ICR} & s
\end{bmatrix}^T
= 
\begin{bmatrix}
X_v& Y_l& Y_r& s
\end{bmatrix}^\top \notag
\end{align}
The local identifiability of $\bgamma$ can be stated as follows:
\begin{lem}
	By using the monocular and odometer measurements, and the 3 ICR parameter vector $\boldsymbol{\xi}_{ICR}$ to model the kinematics, $\bgamma$ is locally identifiable except for the following degenerate cases: (\romannumeral1) the odometer linear velocity $o_l(t)$ keeps zero; (\romannumeral2) the angular velocity $\breve{\omega}(t)$ keeps zero; (\romannumeral3)  $o_r (t)$, $o_l (t)$, and $\breve{\omega}(t)$ are all constants;  (\romannumeral4) the linear velocities of two wheels $o_l(t), o_r(t)$  keeps identical to each other; (\romannumeral5) the angular velocity $\breve{\omega}(t)$ is consistently proportional to $o_l(t)$.
\end{lem}

\begin{proof}
	We first note that the local identifiability of $\bgamma$ is equivalent to that of $\bar{\bgamma}$,
	$$\bar{\bgamma} =\begin{bmatrix}
		Y_l & \Delta Y& X_v& s
	\end{bmatrix}^T $$ 
	By expanding~\ref{eq:constraints_mono_3} and considering all the measurements at different time $t$, we can derive following system constraints:
	\begin{subequations}
		\label{eq:constraints_mono_3_ex}
		\begin{align}
			c_x(\bar{\bgamma}, t) &= \breve{\omega}(t) {}^\bO y_\bC + s \breve{v}_x(t) - \breve{\omega} Y_l - o_l(t) = 0\\
			c_y(\bar{\bgamma}, t) &=  -\breve{\omega}(t) {}^\bO x_\bC + s \breve{v}_y(t) + \breve{\omega}(t) X_v = 0\\
			c_{\omega}(\bar{\bgamma}, t) &= \breve{\omega}(t) + \frac{o_l(t) - o_r(t)}{\Delta Y} =0
		\end{align}
	\end{subequations}
	Similar to the case of using full kinematic parameters $\boldsymbol{\xi}$ in Section.~\ref{sec:mono_ICR_5}, we derive the following observability matrix for $\bar{\bgamma}$ (using Eq.~\eqref{eq:occ} and.~\ref{eq:d}):
	\begin{align}
		\label{eq:obser_matrix_3icr}
		\mathbf{M} = \begin{bmatrix}
			- \breve{\omega}(t_0)& 0& 0& \breve{v}_x(t_0)\\
			0& 0& \breve{\omega}(t_0)& \breve{v}_y(t_0)\\
			0& \frac{o_r(t_0) - o_l(t_0)}{(\Delta Y)^2}& 0& 0 \\
			\vdots& \vdots& \vdots& \vdots\\ 
			- \breve{\omega}(t_s)& 0& 0& \breve{v}_x(t_s)\\
			0& 0& \breve{\omega}(t_s)& \breve{v}_y(t_s)\\
			0& \frac{o_r(t_s) - o_l(t_s)}{(\Delta Y)^2}& 0& 0     
		\end{bmatrix}
	\end{align}
	To simplify the structure of the observability matrix, we apply the following linear operations without changing the observability properties: 
	\begin{align}
		\mathbf{M}(:,2) &\!\leftarrow \!(\Delta Y)^2  \mathbf{M}(:,2) \notag \\ \mathbf{M}(:,4) &\!\leftarrow\! s  \mathbf{M}(:,4) \!\!+\!\! (Y_l \!\!-\!\! {}^{\bO}y_{\bC}) \mathbf{M}(:,1) \!\!+\!\! (X_v\!\! -\!\! {}^{\bO}x_\bC)  \mathbf{M}(:,3) \notag
	\end{align} 
	where $(\cdot)\leftarrow(\cdot)$ represents the operator to replace the left side by the right side. As a result, Eq.~\eqref{eq:obser_matrix_3icr} can be simplified as:
	\begin{align}
		\label{eq:obser_matrix_3icr_af}
		\mathbf{M} = \begin{bmatrix}
			- \breve{\omega}(t_0)& 0& 0& o_l(t_0)\\
			0& 0& \breve{\omega}(t_0)& 0\\
			0& {o_r(t_0) - o_l(t_0)}& 0& 0 \\
			\vdots& \vdots& \vdots& \vdots\\ 
			- \breve{\omega}(t_s)& 0& 0& o_l(t_s)\\
			0& 0& \breve{\omega}(t_s)& 0\\
			0& o_r(t_s) - o_l(t_s)& 0& 0     
		\end{bmatrix}
	\end{align}
	To investigate the observability of the matrix in Eq.~\eqref{eq:obser_matrix_3icr_af}, 
	we inspect the existence of the non-zero vector $\mathbf k$ such that $\mathbf{M} \, \mathbf k = \mathbf 0, \mathbf k = \begin{bmatrix}
		k_1&k_2&k_3&k_4
	\end{bmatrix}^\top \neq\mathbf 0$. 
	If such a vector $\mathbf k$ exists, all of the following conditions must be satisfied:
	\begin{align}
		- \breve{\omega}(t)  k_1 + o_l(t) k_4 = 0
		,(o_r(t) - o_l(t)) k_2 = 0
		,\breve{\omega}(t) k_3 = 0 \notag
	\end{align}
	To allow the above equations to be true, one of the following conditions is required:  \begin{itemize}
		\item $o_l(t)$ keeps constantly zero, $\mathbf k = \begin{bmatrix}
			0&0&0&\rho
		\end{bmatrix}^\top $,
		\item $\breve{\omega}(t)$ keeps constantly zero, $\mathbf k = \begin{bmatrix}
			0&0&\rho&0
		\end{bmatrix}^\top $,
		\item $o_r (t)$, $o_l (t)$, and $\breve{\omega}(t)$ are all constants, $\mathbf k = \begin{bmatrix}
			0&\rho&0&0
		\end{bmatrix}^\top $,
		\item $o_l(t)$ keeps identical to $ o_r(t)$, $\mathbf k = \begin{bmatrix}
			0&\rho&0&0
		\end{bmatrix}^\top $,
		\item $\breve{\omega}(t)$ keeps proportional to $o_l(t)$, $\mathbf k =
		\begin{bmatrix}
			\rho o_l/\breve{\omega}&0&0&\rho
		\end{bmatrix}^\top$.
	\end{itemize}    
	where $\rho$ can be any non-zero value that is used to generate valid non-zero vector $\mathbf k$ such that $\mathbf{M} \, \mathbf k = \mathbf 0$. We note that, all above cases are {\em special} conditions. When a robot moves under general motion, none of those conditions can be satisfied. Therefore, in a camera and odometers only skid-steering robot localization system, $\boldsymbol{\xi}_{ICR}$ is observable unless entering the specified special conditions listed above.
\end{proof}

\subsection{Observability of \texorpdfstring{$\boldsymbol \xi$}{} with a Monocular Camera, an IMU, and Odometer}

So far, we have shown that when a robotic system is equipped with a monocular camera and wheel odometer, estimating $\boldsymbol \xi$ is not feasible, and the alternative solution is to include $\boldsymbol \xi_{ICR}$ in the online stage only. However, it is not an ideal solution to calibrate $\boldsymbol \xi_{\alpha}$ offline and fixed in the online stage since it is subject to the changes in road conditions and tire conditions, etc.

To tackle this problem, we investigate the observability of $\boldsymbol \xi$ when an IMU and camera are fused with odometer measurements.
Once IMU  and camera are used, similar to the previous analysis, we start by introducing the `inferred measurement'. Instead of focusing on visual measurement only, we provide `inferred' measurement by considering the visual-inertial system together. As analyzed in rich existing literature,
visual-inertial estimation provides: camera's local (\romannumeral1) angular velocity and (\romannumeral2) linear velocity, 
similar to vision only case (Eq.~\eqref{eq:camera_measurement1}) without having the unknown scale factor~\cite{hesch2013towards,li2014onlineabc,schneider2019observability}.
Similarly to Eq.~\eqref{eq:simple}, to simplify the analysis, we prove identifiability of $\bar{\boldsymbol \xi}$
instead of ${\boldsymbol \xi}$, since properties of $\bar{\boldsymbol \xi}$
and ${\boldsymbol \xi}$ are interchangeable:
\begin{align}
\bar{\boldsymbol \xi} = \begin{bmatrix} Y_l& \Delta Y&  X_v& \beta_l& \beta_r\end{bmatrix}^\top \notag
\end{align}

\begin{lem}
	By using measurements from a monocular camera, an IMU, and wheel odometer, $\bar{\boldsymbol \xi}$ is locally identifiable, except for following degenerate cases: (\romannumeral1) velocity of one of the wheels, $o_l(t)$ or $o_r(t)$, keeps zero; (\romannumeral2) $\breve{\omega}(t)$ keeps zero; (\romannumeral3) $o_r (t)$, $o_l (t)$, and $\breve{\omega}(t)$ are all constants; (\romannumeral4) $o _{l}(t)$ is always proportional to $o_r(t)$;(\romannumeral5) $\breve{\omega}(t)$ is always proportional to $o_l(t)$.
\end{lem}
\begin{proof}
	Similarly to Eq.~\eqref{eq:constraints_mono_5_ex}, by removing the scale factor, the corresponding system constraints can be derived as:
	\begin{subequations}
		\label{eq:constraints_stereo_5_ex}
		\begin{align}
			c_x(\bar{\boldsymbol \xi}, t) &= \breve{\omega}(t) {}^{\mathbf O} y_{\mathbf C} +  \breve{v} _{x}(t) - \breve{\omega}(t) Y_l - \beta_l \Delta Y o_l(t) = 0\\
			c_y(\bar{\boldsymbol \xi}, t) &=  -\breve{\omega}(t) {}^{\mathbf O} x_{\mathbf C} +  \breve{v} _{y}(t) + \breve{\omega}(t) X_v = 0\\
			c_{\omega}(\bar{\boldsymbol \xi}, t) &= \breve{\omega}(t) + \beta_l o_l(t) - \beta_r o_r(t) =0
		\end{align}
	\end{subequations}
	Therefore, the observability matrix for $\bar{\boldsymbol \xi}$ can be computed by:
	\begin{align}\setlength{\arraycolsep}{4.0pt}
		\label{eq:obser_matrix_5icr_stereo}
		\mathbf{M} = \begin{bmatrix}
			- \breve{\omega}(t_0)& -\beta_l o_l(t_0)& 0& -\Delta Y o_l(t_0)& 0\\
			0& 0& \breve{\omega}(t_0)& 0& 0\\
			0& 0& 0& o_l(t_0)& -o_r(t_0) \\
			\vdots& \vdots& \vdots& \vdots& \vdots\\
			- \breve{\omega}(t_s)& -\beta_l o_l(t_s)& 0& -\Delta Y o_l(t_s)& 0\\
			0& 0& \breve{\omega}(t_s)& 0& 0\\
			0& 0& 0& o_l(t_s)& -o_r(t_s)
		\end{bmatrix}
	\end{align}
	After the following linear operations:
	\begin{align}
		\mathbf{M}(:,2) &\!\leftarrow\! -\mathbf{M}(:,2)/\beta_l \notag\\
		\mathbf{M}(:,4) &\!\leftarrow\! \mathbf{M}(:,4) + \Delta Y \mathbf{M}(:,2) \notag
	\end{align} 
	Eq.~\eqref{eq:obser_matrix_5icr_stereo} can be simplified as:
	\begin{align}\setlength{\arraycolsep}{4.0pt}
		\label{eq:obser_matrix_5icr_stereo_af}
		\mathbf{M} = \begin{bmatrix}
			- \breve{\omega}(t_0)& o_l(t_0)& 0& 0& 0\\
			0& 0& \breve{\omega}(t_0)& 0& 0\\
			0& 0& 0& o_l(t_0)& -o_r(t_0) \\
			\vdots& \vdots& \vdots& \vdots& \vdots\\
			- \breve{\omega}(t_s)& o_l(t_s)& 0& 0& 0\\
			0& 0& \breve{\omega}(t_s)& 0& 0\\
			0& 0& 0& o_l(t_s)& -o_r(t_s)         
		\end{bmatrix}
	\end{align}
	Similarly, we investigate the existence of the non-zero vector $\mathbf{k}$ such that $\mathbf{M} \mathbf{k} = 0$. If such a vector $\mathbf{k} = \begin{bmatrix}    k_1& k_2& k_3& k_4& k_5 \end{bmatrix}$ exists, all of the following must be satisfied:
	\begin{align}
		- \breve{\omega}(t)  k_1 + o_l(t) k_2 \!=\! 0, \breve{\omega}(t) k_3 \!=\! 0 
		,o_r(t) k_4 - o_l(t) k_5 \!=\! 0
		\notag
	\end{align}
	which requires one of the following conditions to be true:
	\begin{itemize}
		\item $o_l(t)$ is constantly zero, $\mathbf k = \begin{bmatrix}
			0& \rho_1& 0& \rho_2& 0
		\end{bmatrix}^\top $, or $o_r(t)$ is constantly zero, $\mathbf k = \begin{bmatrix}
			0& 0& 0& 0& \rho
		\end{bmatrix}^\top $
		\item $\breve{\omega}(t)$ is constantly zero, $\mathbf k = \begin{bmatrix}
			\rho_1& 0& \rho_2& 0& 0
		\end{bmatrix}^\top$,
		\item $o_r (t)$, $o_l (t)$, and $\breve{\omega}(t)$ are all constants, $\mathbf k = \begin{bmatrix}
			0& 0& \rho& 0& 0
		\end{bmatrix}^\top$,
		\item $o_l(t)$ keeps proportional to $ o_r(t)$, $\mathbf k = \begin{bmatrix}
			0& 0& 0& \rho& 0
		\end{bmatrix}^\top$,
		\item $\breve{\omega}(t)$ keeps proportional to $o_l(t)$, $\mathbf k = \begin{bmatrix}
			\rho o_l/\breve{\omega}& \rho& 0& 0& 0
		\end{bmatrix}^\top$.
	\end{itemize}    
	where $\rho, \rho_1, \rho_2$ can be any non-zero value that is used to generate valid non-zero vector $\mathbf k$ such that $\mathbf{M} \, \mathbf k = \mathbf 0$. All the above cases are {\em special} conditions. Therefore, in a  skid-steering robot localization system equipped with a camera, an IMU, and odometers, $\boldsymbol{\xi}$ is observable unless entering the specified special conditions listed above. This completes the proof.
	%
	%
\end{proof}

\subsection{Observability of \texorpdfstring{$\boldsymbol \xi$}{} with a Monocular Camera, an IMU, an Odometer and with Online Extrinsics Calibration}
\label{sec:obser with extrin}
%

It is essential to know the extrinsic transformations between different sensors in a multi-sensor fusion system. 
Since IMU-camera extrinsic parameters are widely investigated in the existing literature, and practically the IMU-camera system is frequently manufactured as an integrated sensor suite, we here focus on camera-odometer extrinsic parameters.
%
%
We first define the parameter state when camera-odometer extrinsics are included:
\begin{align}\setlength{\arraycolsep}{4.0pt}
\boldeta \!=\! \begin{bmatrix}
X_v& Y_l& Y_r& \alpha_l& \alpha_r& {}^{\bO} x_\bC& {}^{\bO} y_\bC& {}^{\bO} z_\bC& {}^{\bO}_{\bC}\delta\btheta^\top 
\end{bmatrix}^\top \notag
\end{align}
where ${}^{\bO}\bp_\bC = \begin{bmatrix}    {}^{\bO} x_\bC& {}^{\bO} y_\bC& {}^{\bO} z_\bC    \end{bmatrix}^\top$ and $^{\bO}_{\bC}\delta\btheta \in \mathbb{R}^3$ are the translational and rotational part of the extrinsic transformation between odometer and camera. $^{\bO}_{\bC}\delta\btheta$ is error state (or Lie algebra increment) of the 3D rotation matrix $^{\bO}_{\bC}\bR$.

Since extrinsic translation and rotation components might be subject to different observability properties, we also define state parameters that contain each of them separately:
\begin{align}\setlength{\arraycolsep}{4.0pt}
\boldeta_p \!=\! \begin{bmatrix}
X_v& Y_l& Y_r& \alpha_l& \alpha_r& {}^{\bO} x_\bC& {}^{\bO} y_\bC& {}^{\bO} z_\bC
\end{bmatrix}^\top \notag
\end{align}
and
\begin{align}\setlength{\arraycolsep}{4.0pt}
\boldeta_{\theta} \!=\! \begin{bmatrix}
X_v& Y_l& Y_r& \alpha_l& \alpha_r& {}^{\bO}_{\bC}\delta\btheta^\top 
\end{bmatrix}^\top \notag
\end{align} 
To summarize, the objective of this section is to demonstrate the observability properties of $\boldeta$, $\boldeta_p$, and $\boldeta_{\theta}$.
\begin{lem}
	By using measurements from a monocular camera, IMU and wheel odometers, $\boldeta_p$ and $\boldeta$ are not identifiable. Specifically, the vertical direction of translation in the extrinsics, $ {}^{\bO} z_\bC$ is always unidentifiable for any type of ground robot, and $ {}^{\bO} x_\bC$ and ${}^{\bO} y_\bC$ become unidentifiable if the skid-steering kinematic parameters are estimated online.
\end{lem}
\begin{lem}
	By using measurements from a monocular camera, IMU and wheel odometers, $\boldeta_{\theta}$ is identifiable in general motion, which allows an estimation algorithm to perform online calibration on extrinsic rotational parameters.
\end{lem}
\begin{proof}
	First of all, $\boldeta$ is locally identifiable if and only if $\bar{\boldeta}$ is locally identifiable:
	\begin{align}\setlength{\arraycolsep}{4.0pt}
		\bar{\boldeta} = \begin{bmatrix} Y_l& \Delta Y&  X_v& \beta_l& \beta_r & {}^{\bO} x_\bC& {}^{\bO} y_\bC& {}^{\bO} z_\bC& {}^{\bO}_{\bC}\delta\btheta^\top \end{bmatrix}^\top \notag
	\end{align}
	By substituting $\breve{{\bv}}(t) \triangleq {}^{\mathbf O} _{\mathbf C} \mathbf R \cdot  \breve{{\bv}}_C(t)$ in Eq.~\eqref{eq:camera_measurement1}, we are able to derive constraints similar to Eq.~\eqref{eq:constraints_stereo_5_ex}, as
	\begin{subequations}
		\label{eq:constraints_stereo_5_excalib}
		\begin{align}
			c_x(\bar{\boldeta}, t) &\!=\! \breve{\omega}(t) {}^\bO y_\bC \!+\!  {\be_1}^\top {}^\bO _\bC \bR \breve{\bv}_C (t) \!-\! \breve{\omega} Y_l \!-\! \beta_l \Delta Y o_l(t) \!=\! 0\\
			c_y(\bar{\boldeta}, t) &\!=\!  -\breve{\omega}(t) {}^\bO x_\bC +  {\be_2}^\top {}^\bO _\bC \bR \breve{\bv}_C (t) + \breve{\omega}(t) X_v \!=\! 0\\
			c_{\omega}(\bar{\bzeta}, t) &\!=\! \breve{\omega}(t) + \beta_l o_l(t) - \beta_r o_r(t) \!=\!0
		\end{align}
	\end{subequations}
	Considering the constraints in a set of time instants $\mathcal{S}  = \lbrace  t_0, t_1, \dots, t_s \rbrace$, we compute the following observability matrices for the systems with calibrating $\bar{\boldeta}_p$ and $\bar{\boldeta}_\theta$, given by Eq.~\eqref{eq:obser_matrix_5icr_excalib_p} and Eq.~\eqref{eq:obser_matrix_5icr_excalib_theta_1},  respectively.
	Eq.~\eqref{eq:obser_matrix_5icr_excalib_theta_1} can be converted to Eq.~\eqref{eq:obser_matrix_5icr_excalib_theta_2} by linear operations:
	\begin{align}
		\mathbf{M}(:,4) &\!\leftarrow\! \mathbf{M}(:,4) - \Delta Y/\beta_l \mathbf{M}(:,2)  \notag\\
		\mathbf{M}(:,2) &\!\leftarrow\! - \mathbf{M}(:,2)/\beta_l \notag    
	\end{align} 
	\newcounter{mytempeqncnt0}
	\begin{figure*}[!ht]
		\normalsize
		\setcounter{mytempeqncnt0}{\value{equation}}
		\begin{align}\setlength{\arraycolsep}{4.0pt}
			\label{eq:obser_matrix_5icr_excalib_p}
			\mathbf{M} \!=\! \begin{bmatrix}
				- \breve{\omega}(t_0)& -\beta_l o_l(t_0)& 0& -\Delta Y o_l(t_0)& 0& 0& \breve{\omega}(t_0)& 0\\
				0& 0& \breve{\omega}(t_0)& 0& 0& - \breve{\omega}(t_0)& 0& 0\\
				0& 0& 0& o_l(t_0)& -o_r(t_0)& 0& 0& 0\\
				\vdots& \vdots& \vdots& \vdots& \vdots& \vdots& \vdots& \vdots \\
				- \breve{\omega}(t_s)& -\beta_l o_l(t_s)& 0& -\Delta Y o_l(t_s)& 0& 0& \breve{\omega}(t_s)& 0 \\
				0& 0& \breve{\omega}(t_s)& 0& 0&  - \breve{\omega}(t_s)& 0& 0\\
				0& 0& 0& o_l(t_s)& -o_r(t_s)& 0& 0& 0
			\end{bmatrix}
		\end{align}    
		\setcounter{equation}{\value{mytempeqncnt0}+1}
	\end{figure*}
	\newcounter{mytempeqncnt1}
	\begin{figure*}[!ht]
		\normalsize
		\setcounter{mytempeqncnt1}{\value{equation}}
		\begin{subequations}
			\begin{align}\setlength{\arraycolsep}{4.0pt}
				\label{eq:obser_matrix_5icr_excalib_theta_1}
				\mathbf{M} \!=\! \begin{bmatrix}
					- \breve{\omega}(t_0)& -\beta_l o_l(t_0)& 0& -\Delta Y o_l(t_0)& 0\textsf{}& -{\be_1}^\top {}^\bO _\bC  \bR \lfloor \breve{\bv}_C (t_0) \rfloor\\
					0& 0& \breve{\omega}(t_0)& 0& 0& -{\be_2}^\top {}^\bO _\bC  \bR \lfloor \breve{\bv}_C (t_0) \rfloor\\
					0& 0& 0& o_l(t_0)& -o_r(t_0)& \mathbf{0}_{3 \times 1} \\
					\vdots& \vdots& \vdots& \vdots& \vdots& \vdots\\
					- \breve{\omega}(t_s)& -\beta_l o_l(t_s)& 0& -\Delta Y o_l(t_s)& 0& -{\be_1}^\top {}^\bO _\bC  \bR \lfloor \breve{\bv}_C (t_s) \rfloor\\
					0& 0& \breve{\omega}(t_s)& 0& 0& -{\be_2}^\top {}^\bO_\bC\bR \lfloor \breve{\bv}_C (t_s) \rfloor\\
					0& 0& 0& o_l(t_s)& -o_r(t_s)& \mathbf{0}_{3 \times 1}     
				\end{bmatrix}
			\end{align} \\
			\begin{align}\setlength{\arraycolsep}{4.0pt}
				\label{eq:obser_matrix_5icr_excalib_theta_2}
				\mathbf{M} \!=\! \begin{bmatrix}
					- \breve{\omega}(t_0)&  o_l(t_0)& 0& 0& 0\textsf{}& -{\be_1}^\top {}^\bO _\bC  \bR \lfloor \breve{\bv}_C (t_0) \rfloor\\
					0& 0& \breve{\omega}(t_0)& 0& 0& -{\be_2}^\top {}^\bO _\bC  \bR \lfloor \breve{\bv}_C (t_0) \rfloor\\
					0& 0& 0& o_l(t_0)& -o_r(t_0)& \mathbf{0}_{3 \times 1} \\
					\vdots& \vdots& \vdots& \vdots& \vdots& \vdots\\
					- \breve{\omega}(t_s)& - o_l(t_s)& 0& 0& 0& -{\be_1}^\top {}^\bO _\bC  \bR \lfloor \breve{\bv}_C (t_s) \rfloor\\
					0& 0& \breve{\omega}(t_s)& 0& 0& -{\be_2}^\top {}^\bO_\bC\bR \lfloor \breve{\bv}_C (t_s) \rfloor\\
					0& 0& 0& o_l(t_s)& -o_r(t_s)& \mathbf{0}_{3 \times 1}     
				\end{bmatrix}
			\end{align}
		\end{subequations}
		
		\setcounter{equation}{\value{mytempeqncnt1}+1}
	\end{figure*}

	Similar to previous proofs, to look into the properties of $\mathbf{M}$ in Eq.~\eqref{eq:obser_matrix_5icr_excalib_p}, we investigate non-zero vector $\mathbf k$ such that $\mathbf{M} \, \mathbf k = \mathbf 0$. We can easily find the following non-zero solutions: 
	\begin{align}
		\mathbf{k}_1 &= \begin{bmatrix} \rho& 0& 0& 0& 0& 0& \rho& 0    \end{bmatrix}^\top \notag\\ \mathbf{k}_2 &= \begin{bmatrix} 0& 0& \rho& 0& 0& \rho& 0& 0    \end{bmatrix}^\top \notag\\
		\mathbf{k}_3 &= \begin{bmatrix} 0& 0& 0& 0& 0& 0& 0& \rho \end{bmatrix}^\top \notag
	\end{align}
	where $\rho$ can be any non-zero value.
	We can find that $\mathbf{k_1},~\mathbf{k_2}$ are related with the kinematic parameters, while $\mathbf{k_3}$ always holds, which results from no constraints on ${}^{\bO}z_{\bC}$ for ground robots and it has no matter with the kinematic parameters. Through the found null spaces, we can draw the following conclusions: (\romannumeral1) $Y_l$ and $ {}^{\bO}y_{\bC}$ are indistinguishable; (\romannumeral2) $X_v$ and $ {}^{\bO}x_{\bC}$ are indistinguishable; (\romannumeral3) the vertical direction of extrinsic parameters ${}^{\bO}z_{\bC}$ is always unidentifiable for skid-steering robot moving on ground, no matter whether the kinematic parameters are calibrated online.

	However, $\mathbf{M}$ in Eq.~\eqref{eq:obser_matrix_5icr_excalib_theta_2} is under quite different properties. 
	Similarly, we investigate the non-zero $\mathbf{k}$ that satisfies  $\mathbf{M} \, \mathbf k = \mathbf 0$, which requires the all of the following to be true:
	\begin{align}
		& - \breve{\omega}(t)  k_1 + o_l(t) k_2 -{\be_1}^\top {}^\bO _\bC  \bR \lfloor \breve{\bv}_C (t_0) \rfloor \cdot   \begin{bmatrix} k_6 & k_7 & k_8 \end{bmatrix}^\top   \!=\! 0 \notag\\ & \breve{\omega}(t) k_3 -{\be_2}^\top {}^\bO _\bC  \bR \lfloor \breve{\bv}_C (t_0) \rfloor \cdot  \begin{bmatrix} k_6 & k_7 & k_8 \end{bmatrix}^\top \!=\! 0 \notag \\
		& o_r(t) k_4 - o_l(t) k_5 \!=\! 0 \notag
	\end{align}
	However, since $\breve{\omega}(t), \breve{\bv}_C (t), o_l(t), o_r(t)$ are time variant under general motion, we can not find such a non-zero vector $\mathbf{k}$.
	We can draw the conclusion: (\romannumeral4) the rotation between camera and odometer are identifiable under general motion.     
\end{proof}

Based on the derived observability properties, it is important to point out the following algorithm design issues: (\romannumeral1) Unlike online calibration algorithms in other literature~\cite{li2014onlineabc,qin2018online,schneider2019observability}, extrinsic parameters between camera and odometer are not observable and cannot be calibrated online. (\romannumeral2) The extrinsic rotation between sensors can be observable in general motion and thus can be included in the state vector of a state estimation algorithm for online calibration. Since extrinsic rotation usually does not undergo severe changes, and its calibration also suffers from degenerate motions, we do not recommend calibrating it online.

\section{Experimental Results}
\label{sec:exp}

In this section, we provide experimental results that support our claims in both algorithm design and theoretical analysis. 
Specifically, we conducted real-world experiments and simulation tests to demonstrate: (\romannumeral1) the advantages and necessities of online estimating kinematic parameters in visual (inertial) localization systems for skid-steering robots, (\romannumeral2) the observability and convergence properties of the skid-steering kinematic parameters under different settings, (\romannumeral3) the robustness of the proposed kinematics and pose estimation method against the changes of mass center, tire inflation condition, terrains, etc., to enable the long-term mission completeness of the robots without performance reduction.
%
%
%
%
In our experiments, we used the adapted skid-steering robots based on the commercially available Clearpath Jackal robot~\cite{Clearpath} (see Fig.~\ref{fig:robots_icr}), with both `kinematics and pose estimation' sensors and `ground-truth sensors' equipped.
For `kinematics and pose estimation' sensors, we used a $10$Hz monocular global shutter camera at a resolution of $640 \times 400$,  a $200$Hz Bosch BMI160 IMU, and $100$Hz wheel encoders~\footnote{We point out that, we used customized wheel encoder hardware instead of the on-board one on Clearpath robot, to allow accurate hardware synchronization between sensors.}.
The `ground truth' sensor mainly relies on RTK-GPS with centimeter-level precision. 
All sensors used in our experiment are synchronized by hardware and calibrated offline via~\cite{yiming2019}.
We note that the offline calibration procedure is an important prerequisite in our experiments since the extrinsic translation between the odometer and camera has shown to be constantly unobservable.
All the experiments are conducted on an Intel Core i7-8700 @ 3.20GHz CPU for comparisons. 
To fully examine the proposed method, we only evaluate the poses from the odometry system without loop closures and prior maps to standout the performance of our proposed kinematic and pose estimation method.
The sliding-window size of the estimator is $8$ in all the evaluations presented in this paper.
\subsection{Real-world Experiment}
\begin{figure} [htb]
		\vspace*{-1em} 
	\centering
	\includegraphics[width = \columnwidth]{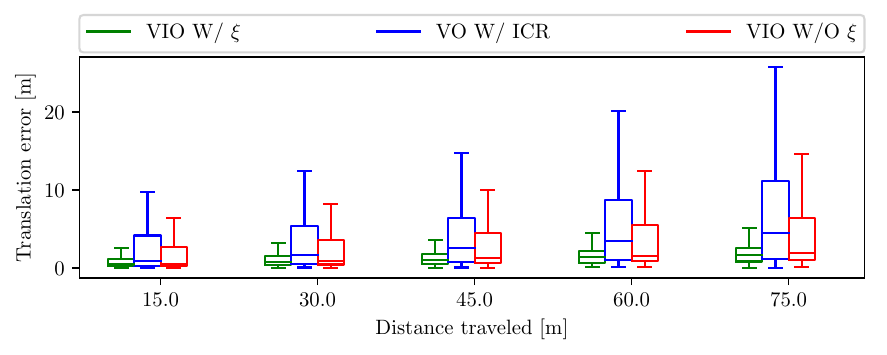}
	\captionsetup{justification=justified}
	\caption{Boxplot of the relative trajectory error (RPE) statistics over all the sequences where RTK-GPS measurements are available. This plot best seen in color.
	}
	\vspace*{-1em} 
	\label{fig:gps rpe} 
\end{figure}
In the first set of experiments, we focus on validating the effectiveness of the proposed skid-steering model as well as the localization algorithm. Specifically, we investigated the localization accuracy by estimating skid-steering kinematic parameters $\boldsymbol \xi$ (Eq.~\eqref{eq:ICRmat_5par}) online and compared that to the competing methods.
To demonstrate the generality of our method, we conducted experiments under various environmental conditions.
The environments involved in our robotic data collection include (a) lawn, (b) cement brick, (c) wooden bridge, (d) muddy road, (e) asphalt road, (f) ceramic tiles, (g) carpet, and (h) wooden floor. The representative figures in 8 types of terrains are also shown in the experiments of~\cite{zuo2019visual}.
%

We note that since GPS signal is not always available in all tests (e.g., indoor tests), we use both final drift and root-mean-squared error (RMSE) of absolute translational error (ATE)~\cite{zhang2018tutorial} as our metrics.
It is also important to point out that, in the research community, it is preferred to use publicly-available datasets to conduct experiments to facilitate comparison between different researchers. However, most localization datasets publicly available either utilize passenger cars (KITTI~\cite{Geiger2013IJRR}, Kaist Complex Urban~\cite{jeong2019complex}, Oxford Robotcar~\cite{RobotCarDatasetIJRR}) or lack of one or multiple synchronized low cost sensors (NCLT~\cite{carlevaris2016university}, and Canadian 3DMap~\cite{tong2013canadian}). To this end, we also plan to release a comprehensive dataset specifically with low-cost sensors, as our future work.
\begin{table*}[th]
	\renewcommand{\arraystretch}{1.5}
	\caption{ RMSE of ATE (m) on the Sequences with RTK-GPS Measurements. }
	\label{tb:gpsseq_ate}
	\vspace{-1em}
	\begin{center}
		\resizebox{0.8\textwidth}{!}
		{  \begin{tabular}{C{3.5cm}C{1.0cm}C{1.0cm}C{1.0cm}C{1.0cm}C{1.0cm}C{1.0cm}C{1.0cm}C{1.0cm}C{1.0cm}C{1.0cm}C{1.0cm}} \toprule
				& SEQ21-CP01 & SEQ22-CP01 & SEQ23-CP01 & SEQ3-CP01 & SEQ4-CP01 & SEQ5-CP01 & SEQ8-CP01 & SEQ12-CP01 & SEQ14-CP01 & SEQ17-CP01& \textbf{Mean} \\ \hline
				{{\textbf{Length (m)}}}& 629.16& 633.53& 651.94& 632.64& 629.96& 626.83& 204.81& 436.19& 372.15& 110.55\\
				{{\textbf{Terrain}}}& (b,f)& (b,f)& (b)& (b,f)& (b,f)& (b,f)& (e)& (e)& (b)& (b)\\ \hline
				{\textbf{VIO W/ $\boldsymbol \xi$ }} &     \textbf{1.36} & \textbf{2.35} & 1.50 & \textbf{2.81} & \textbf{1.17} & \textbf{2.70} & \textbf{0.26} & \textbf{0.80} & 1.65 &  0.32& \textbf{1.492}\\
				{\textbf{VO W/ $\mathrm{\mathbf{ICR}}$  }} &    \clr{ 43.12} & \clr{38.49} & \textbf{0.61} & \clr{48.04} & \clr{39.62} & \clr{40.60} & 0.44 & 3.12 &\textbf{0.74} &  \textbf{0.20}& \clr{21.498}\\
				{\textbf{VIO W/O $\boldsymbol \xi$}~\cite{zhang2021pose}} &     5.88 & 6.02 & 1.62 & 9.82 & 10.43 & 8.61 & 0.50 & 1.86 & 4.97 &  0.45& 5.016\\
				{\textbf{VINS-on-Wheels }~\cite{wu2017vins}} & 6.52  & 8.37 & 3.81 & 8.97 & \clr{13.39} & 8.93 & 0.52 & 1.78 & 5.69    &  0.56 & 5.854   \\
				{\textbf{VINS-Mono }~\cite{qin2018vins}}  & 10.36 & \clr{14.84} & 3.19 & \clr{12.53} & \clr{18.29} & 11.37 & 0.62 & 3.04 & 5.04 &  0.52 & 7.980  \\
				{\textbf{Wheel Odo. W/ $\boldsymbol \xi$}}  & \clr{36.38}  & \clr{48.26} & 4.07 & \clr{63.33} & \clr{43.91} & \clr{54.94} & 1.85 & 4.62 & 7.71 &  0.85 & \clr{26.592}   \\
				{\textbf{Inertial Aided Wheel Odo.}}   & \clr{39.35}  & \clr{45.38} & 4.29 & \clr{58.21} & \clr{40.05} & \clr{48.92} & 1.63 & 3.63 & 5.93 &  0.69 & \clr{24.808}   \\
			 \hline
			\end{tabular}
		}
	\end{center}
	\vspace{-2em}
\end{table*}
\begin{figure*}[t]
	\centering 
	\subfigure{ 
		\includegraphics[width=0.63\columnwidth]{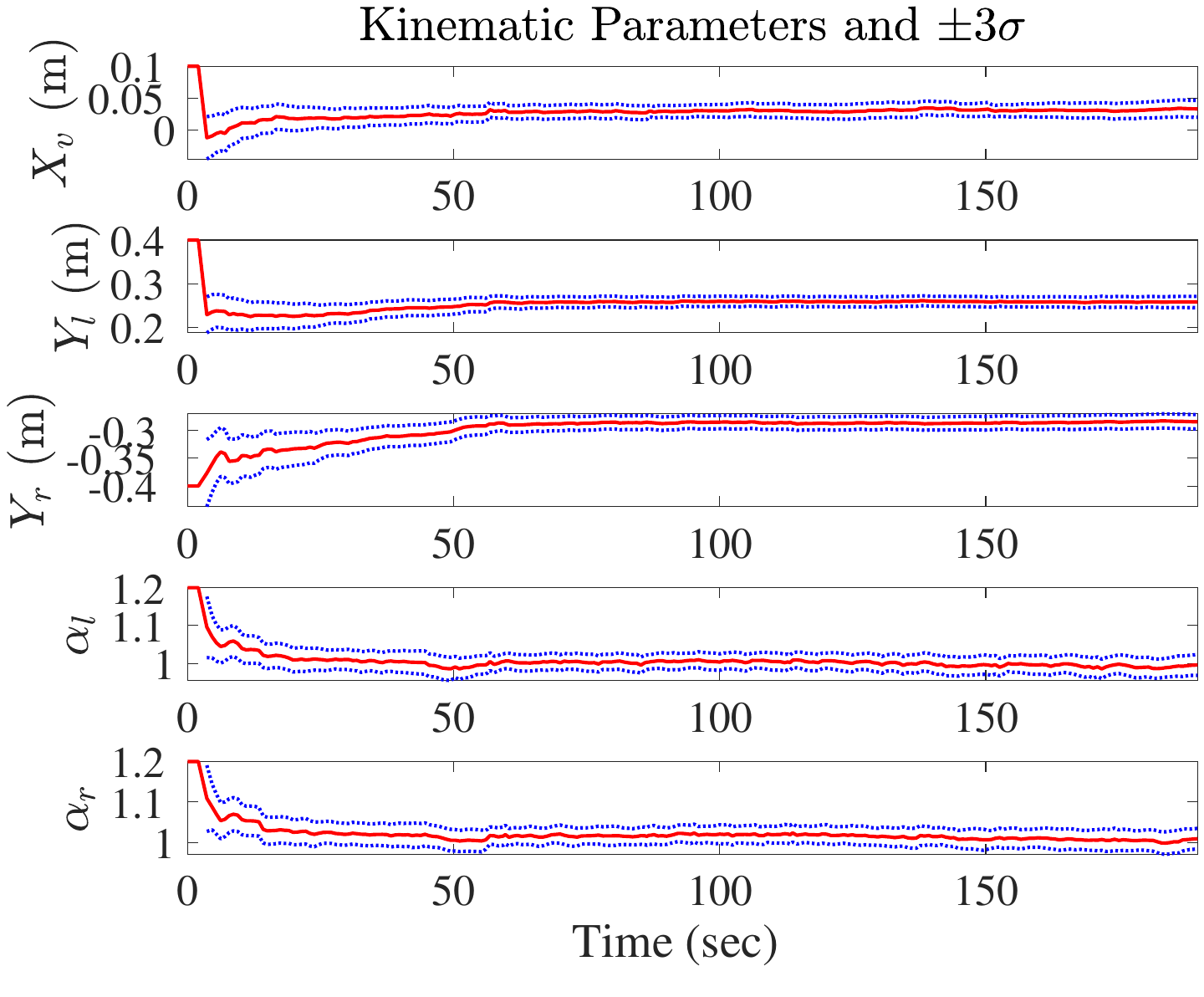} 
	} 
	\subfigure{ 
		\includegraphics[width=0.63\columnwidth]{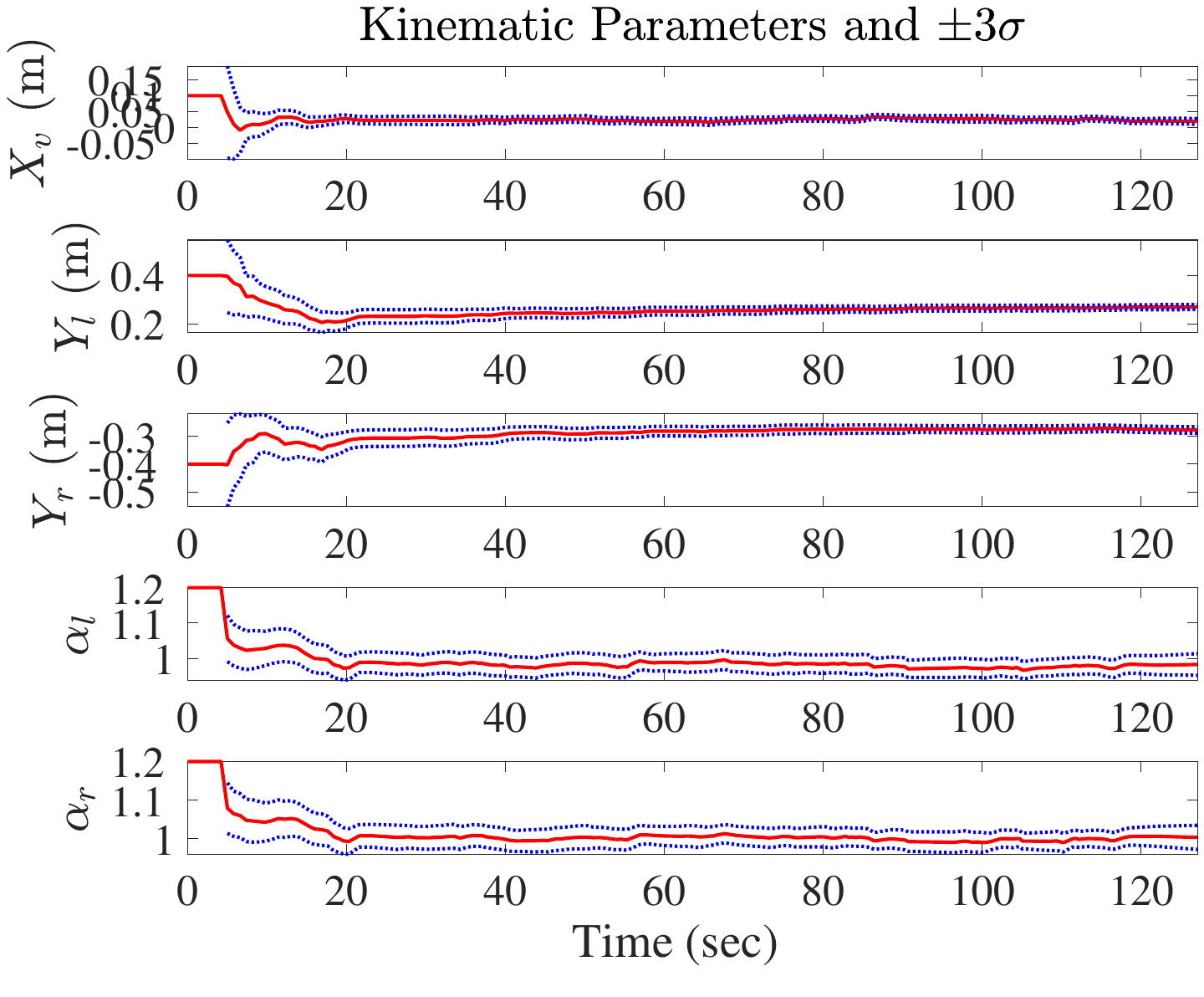} 
	}
	\subfigure{ 
		\includegraphics[width=0.66\columnwidth]{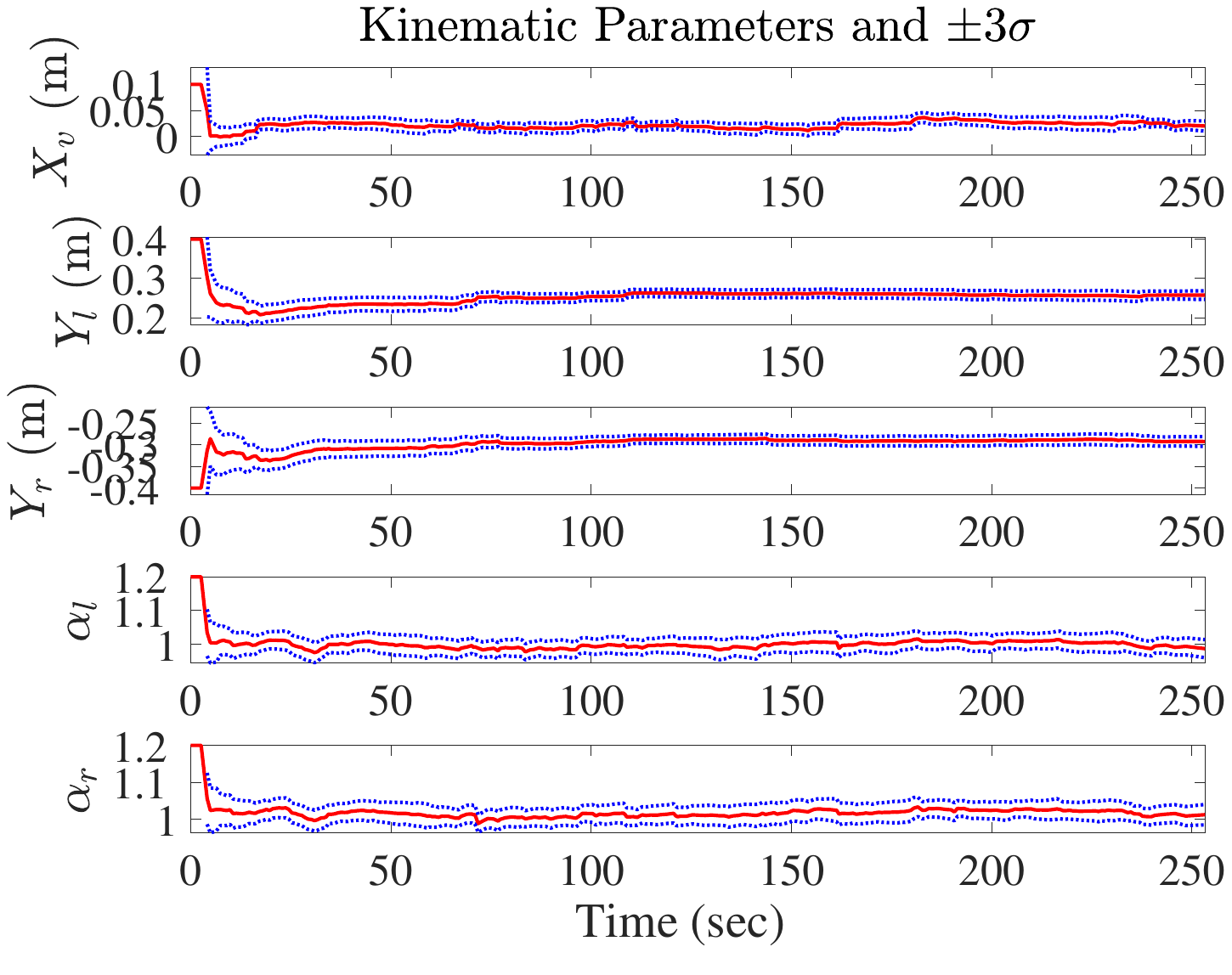} 
	}
	\captionsetup{justification=justified}    
		\vspace{-1em}      
	\caption{While given bad initial values, the kinematic parameters $\boldsymbol{\xi}$ are able to converge to the reasonable values in the visual-inertial navigation system. The online estimated kinematic parameters (red) and the associated $\pm 3 \sigma$ envelopes (blue) are shown on sequences `` SEQ8-CP01, SEQ18-CP01, SEQ19-CP01" from left to right.
	}
	~\label{fig:icr_converge_vio_xi}
	\vspace{-1em} 
\end{figure*}
\begin{figure*}[t]
	\centering 
	\subfigure{ 
		\includegraphics[width=0.63\columnwidth]{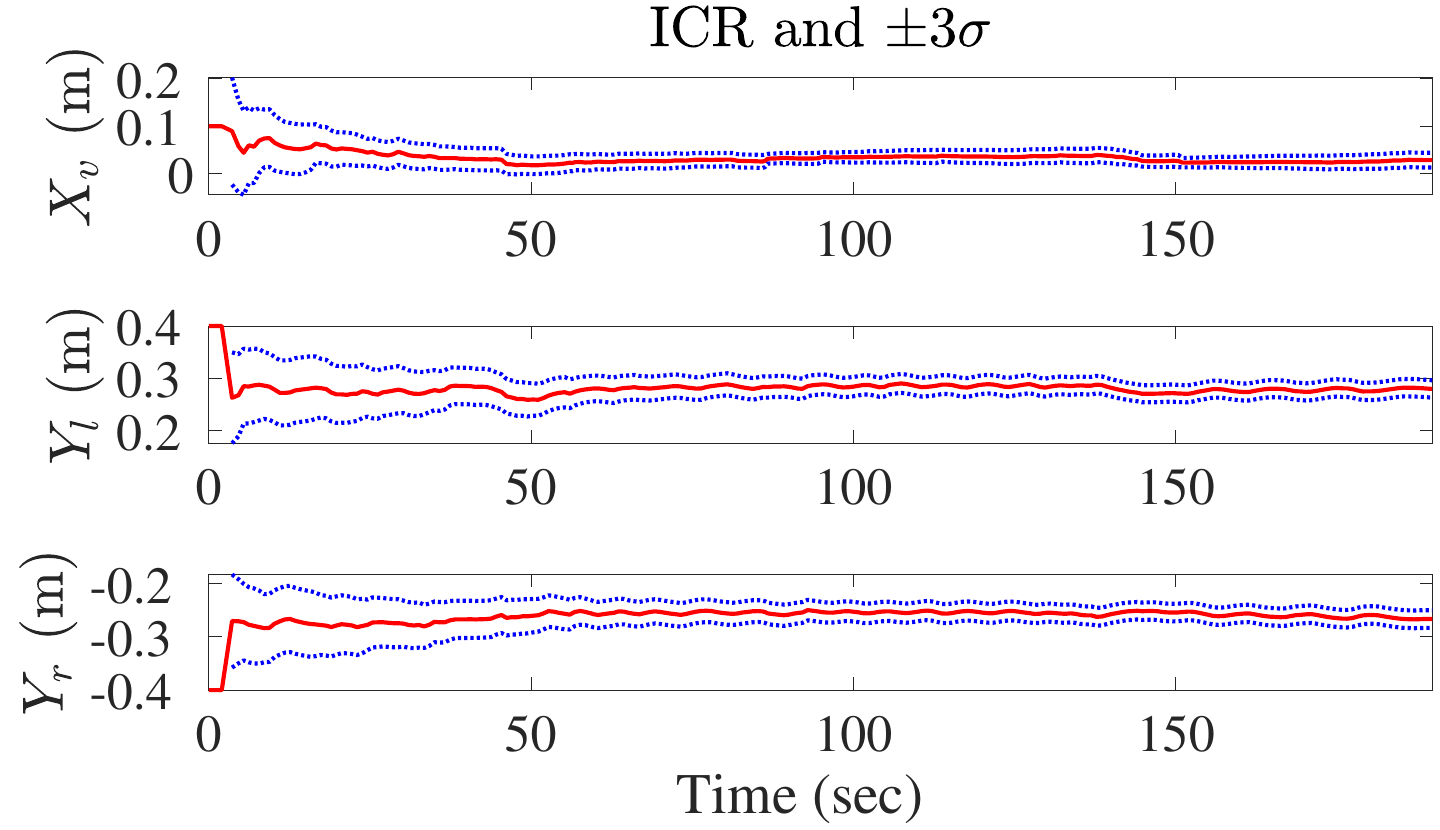} 
	} 
	\subfigure{ 
		\includegraphics[width=0.63\columnwidth]{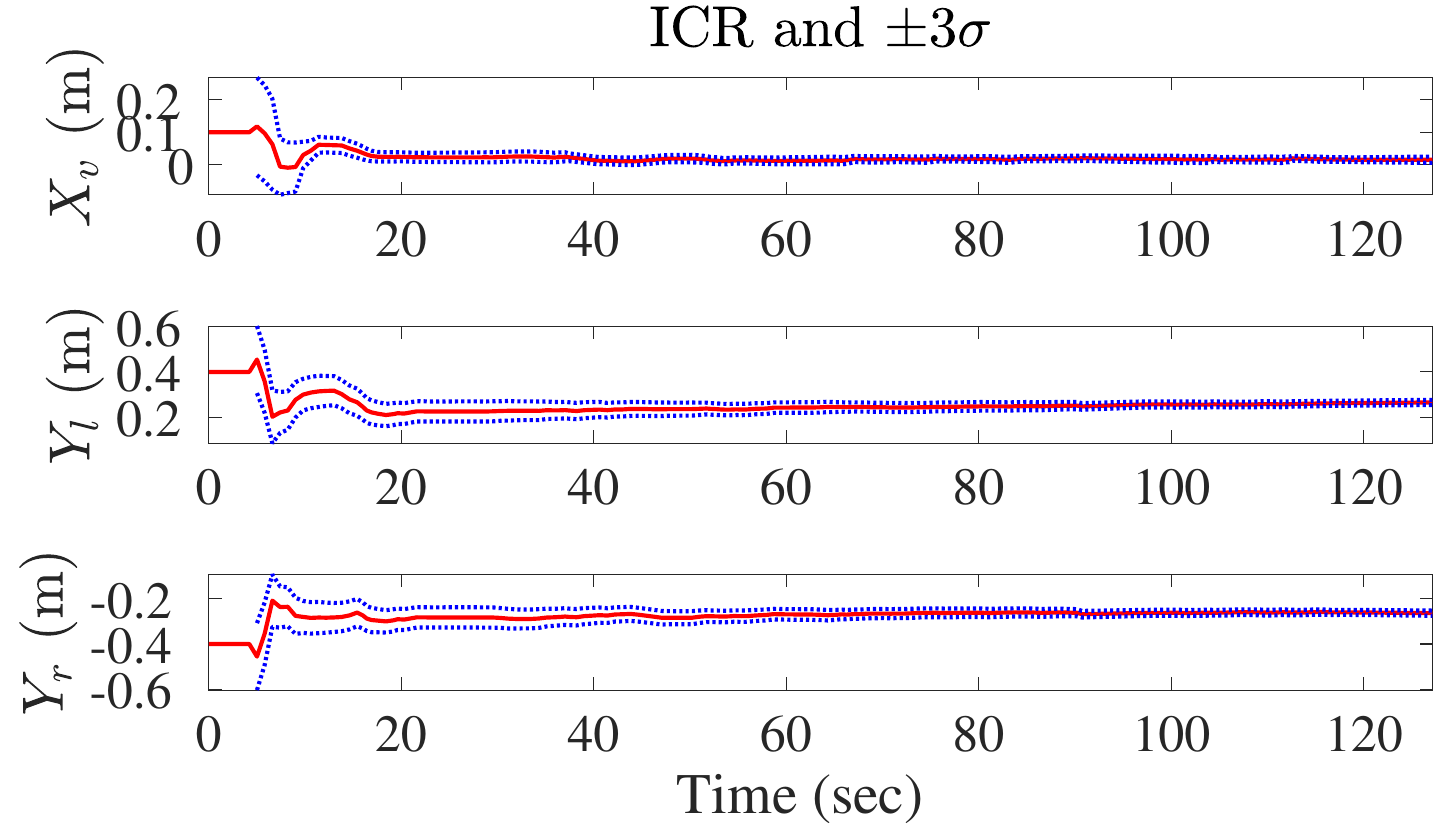} 
	}
	\subfigure{ 
		\includegraphics[width=0.63\columnwidth]{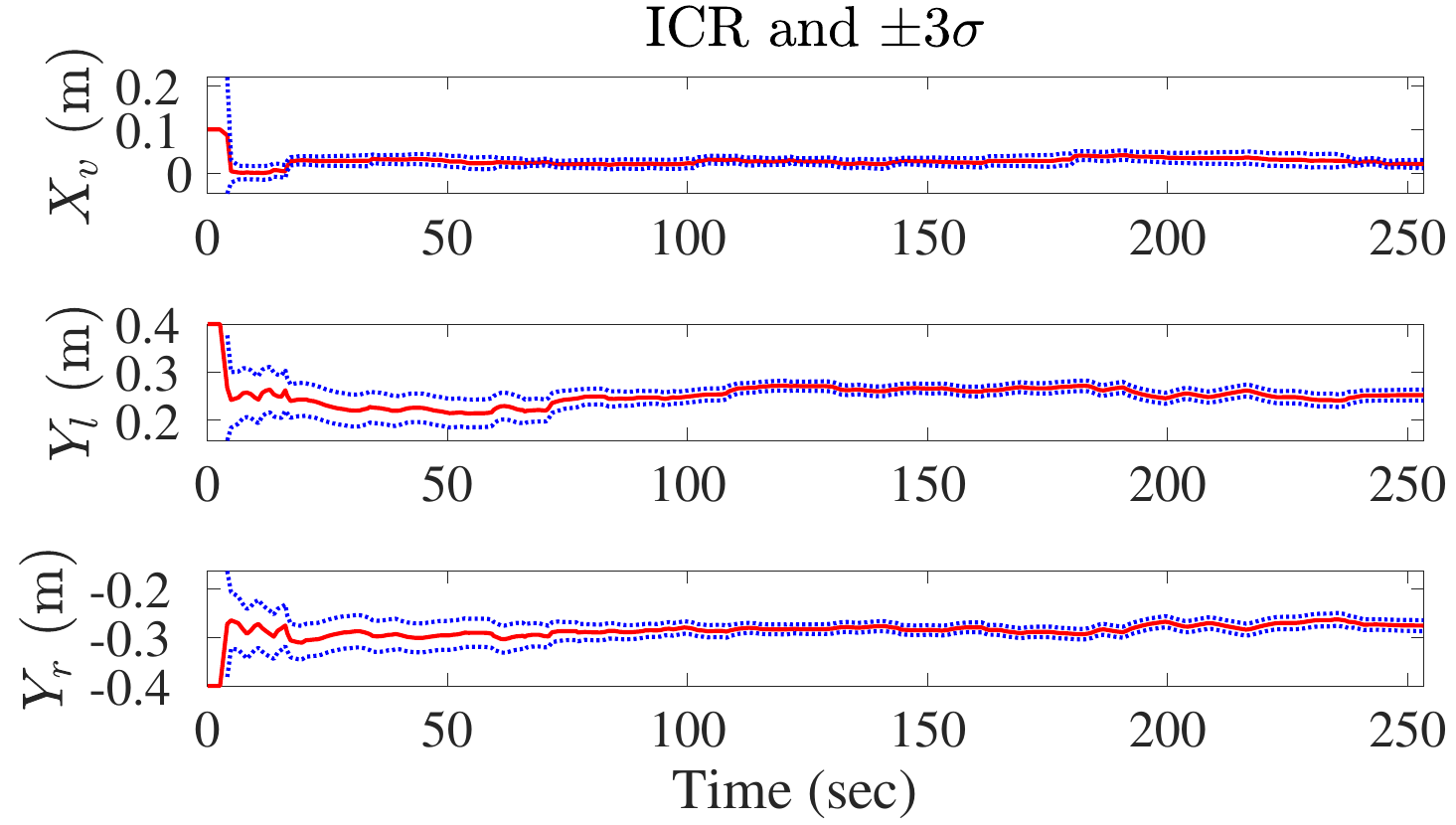} 
	}
	\captionsetup{justification=justified} 
		\vspace{-1em}
	\caption{While given bad initial values, the ICR parameters $\boldsymbol{\xi}_{ICR}$ are able to converge to the reasonable values in the kinematics-constrained VO system. The online estimated $\boldsymbol{\xi}_{ICR}$ parameters (red) and the associated $\pm 3 \sigma$ envelopes (blue) are shown on sequences `` SEQ8-CP01, SEQ18-CP01, SEQ19-CP01" from left to right.}
	~\label{fig:icr_converge_vo_icr}
	\vspace{-2em}
\end{figure*}

\subsubsection{Pose Estimation Accuracy}\label{sec:pose accuracy}
We first conducted an experiment to show the benefits gained by modeling and estimating skid-steering parameters online. In this experiment, three setups are compared, i.e., two provably observable methods and one baseline method. Specifically, those methods are 1) \textbf{VIO (visual-inertial odometry) W/ $\boldsymbol \xi$ :}using measurements from a monocular camera, an IMU, and odometer via the proposed estimator by estimating the full 5 skid-steering kinematic parameters $\boldsymbol \xi$ online; 2) \textbf{VO (visual odometry) W/ $\mathrm{\mathbf{ICR}}$ :} using monocular camera and odometer measurements (without an IMU), and performing localization by estimating the 3 ICR parameters $\boldsymbol{\xi}_{ICR}$ online; 3) \textbf{VIO W/O $\boldsymbol \xi$ :} using measurements from monocular camera, an IMU and odometer, and utilizing differential drive kinematics in Eq.~\eqref{eq:ordinarymodel} for localization~\cite{zhang2021pose} without explicitly modeling $\boldsymbol \xi$. 
%
Notably, the configuration \textbf{VIO W/O  $\boldsymbol \xi$} is exactly our baseline method~\cite{zhang2021pose} for ground robots. In~\cite{zhang2021pose}, similar to the optimized cost function in Eq.~\eqref{eq:cost}, the constraints from prior term $\mathcal{C}_{P}$, camera term $\mathcal{C}_{V}$, IMU term  $\mathcal{C}_{I}$, and motion manifold term $\mathcal{C}_{M}$ are taken into account, while the odometer constraint $\mathcal{C}_{O}$ is simply induced from ideal differential drive model (see Eq.~\eqref{eq:ordinarymodel}) instead of leveraging the ICR-based kinematics constraint (see Section.~\ref{sec:icr_pred}) proposed in this work.

We show the final drift errors on 23 representative sequences in the Table~\ref{tb:drift} of the Appendix section, which cover all eight types of terrains (a)-(g).
Notably,  some sequences also cover multiple types of terrains. 
In the sequences where GPS signals were available across the entire data sequence, we also evaluated the root mean square errors (RMSE)~\cite{Bar-Shalom1988} of absolute translational error (ATE)~\cite{zhang2018tutorial}. To compute that, we interpolated the estimated poses to get the ones corresponding to the timestamp of the GPS measurements.
In addition to the aforementioned configurations, \textbf{VIO W/ $\boldsymbol \xi$},  \textbf{VO W/ $\mathrm{\mathbf{ICR}}$}, and \textbf{VIO W/O $\boldsymbol \xi$}, we also compare to the state-of-the-art wheel odometer aided VIO methtod \textbf{VINS-on-Wheels}~\cite{wu2017vins}, the state-of-the-art VIO method \textbf{VINS-Mono}~\cite{qin2018vins}, and also the vision-free methods, including \textbf{Wheel Odo. W/ $\boldsymbol \xi$} and \textbf{Inertial Aided Wheel Odo.}.
In the method, \textbf{Wheel Odo. W/ $\boldsymbol \xi$}, only the wheel odometer measurements are used by propagating poses forward based on the ICR kinematics with the full 5 kinematic parameters $\boldsymbol \xi$ (see Sec.~\ref{sec:icr_pred}). The kinematic parameters are kept fixed at the given initial values, and will not be updated due to the lack of constraints from the other sensor modalities. The fixed values can not reflect the real-time robot kinematic status, and can lead to significant errors in pose estimation.
As for the method, \textbf{ Inertial Aided Wheel Odo.}, both the measurements from IMU and wheel odometers are fused. Besides the forward propagation based on the full ICR-based kinetic model with wheel odometers measurements, IMU measurements are also propagated forward to formulate the relative pose constraints between two virtual keyframes. Here we selected virtual keyframes once the wheel odometer pose prediction has a translation over 0.2 meters or rotation over 3 degrees. The IMU velocity and biases, as well as the kinematic parameters, are optimized in this method.

The RMSE errors of compared methods are shown in Table.~\ref{tb:gpsseq_ate}, where we highlight the best results in bold, while the bad results (RMSE of ATE is over 12m) by underlines.
The results clearly demonstrate that when skid-steering kinematic parameter $\boldsymbol \xi$ is estimated online, the localization accuracy can be significantly improved. 
This validates our claim that, in order to use odometer measurements of skid-steering robots, the complicated mechanism must be explicitly modelled to avoid accuracy loss. 
We also note that, the method of using an IMU and estimating the full 5 kinematic parameters, \textbf{VIO W/ $\boldsymbol \xi$},  performs best among those methods, by modeling the time-varying scale factors. 
In fact, the method of estimating only 3 ICR parameters with visual and odometer sensors works well for a portion of the dataset while fails in others (e.g., the datasets under (b,f) categories). This is due to the fact that those datasets involve terrain conditions changes, and the scale factor also changes. If those factors are not model, the performance will drop. Moreover, we note that, under those conditions (e.g., (b,f)), the best performing method still works not as good as the performance in other data sequences. This is due to the fact that we used `random walk' process to model the `environmental condition' changes, which is not the `best' assumption when there are rapid road surface changes.   
We will also leave the terrain detection as future work. 
The pure VIO method without wheel odometers, \textbf{VINS-Mono}~\cite{qin2018vins} has a poor performance because of the degeneration motion of ground robots. The state-of-the-art odometer aided VIO method, \textbf{VINS-on-Wheels}~\cite{wu2017vins}, shows relatively good pose estimation results, while it is inferior to the proposed method, due to the planar ground assumption and the ideal differential drive model assumption, which does not address the slippage issue. The vision-free methods, \textbf{Wheel Odo. W/ $\boldsymbol \xi$} and \textbf{Inertial Aided Wheel Odo.}, shows large pose estimation error on most of the sequences, which stands out the substantial help from visual measurements.
Representative Trajectory estimates on representative sequences are also shown in the Appendix.
In order to provide insight into how the error of each algorithm grows with the trajectory length, we also calculate the calculated relative pose
error (RPE) averaged over all the sequences when GPS measurements are available. The RPE results are shown in Table.~\ref{tb:gpsseq_rpe} and Fig.~\ref{fig:gps rpe}, which also support our algorithm claims. 

To further examine the advantages of online estimating the full kinematic parameters $\boldsymbol{\xi}$ in the kinematics-constrained VIO systems, we conduct ablation study and compare the three variants of the proposed method:  1) estimating 3 parameters 
$\boldsymbol{\xi}_{ICR}$ only;  2) estimating the full $\boldsymbol{\xi}$ with 5 parameters; 3) used fixed  $\boldsymbol{\xi}$ with a relatively good initial guess. 
Besides, in the ablation study, the skid-steering robot is tested under the following practically commonly-seen mechanism configurations during its life-long service: (\romannumeral1) normal; (\romannumeral2) carrying a package with the weight around 3 kg; (\romannumeral3) under low tire pressure; (\romannumeral4) carrying a 3-kg package and with low tire pressure. The experimental results can be found in the Appendix, which demonstrates the advantage of estimating the full  kinematic parameters $\boldsymbol{\xi}$ and the effectiveness of the proposed methods under the aforementioned mechanism configurations.

\begin{table}[ht]
	\centering
	\caption{
		Mean of RPE (m) for Different Segment Length on the Sequences with RTK-GPS Measurements.  
	}
	\resizebox{0.9\columnwidth}{!}
	{
		\begin{tabular}{C{2.0cm}C{1.5cm}C{1.5cm}C{1.5cm}} \toprule
			\textbf{Segment Length (m)} &  {\textbf{VIO W/ $\boldsymbol \xi$ }} & {\textbf{VO W/ $\mathrm{\mathbf{ICR}}$  }} & {\textbf{VIO W/O $\boldsymbol \xi$}} \\ \hline
			15.00 &      0.78& 10.27&  2.37 \\
			30.00 &     1.11& 11.33&   2.82 \\
			45.00 &     1.42& 13.14&   3.40 \\
			60.00 &     1.75& 15.61&   4.07 \\
			75.00 &     2.00& 17.76&   4.66 \\ \hline
		\end{tabular}
	}
	\label{tb:gpsseq_rpe}
	\vspace{-1em}
\end{table}

\subsubsection{Convergence of Kinematic Parameters}
In this section, we show experimental results to demonstrate the convergence properties of $\boldsymbol{\xi}$ and $\boldsymbol{\xi}_{ICR}$, in systems that we theoretically claim observable. 
Unlike the experiments in the previous section, which utilized the method described in Sec.~\ref{sec:init kinematic} for kinematic parameter initialization, we manually added extra errors to the kinematic parameter for the tests in this section, to better demonstrate the observability properties.
Specifically, for the kinematics-constrained VIO system, we added the following extra error terms to initial kinematic parameters
\begin{align}
\delta X_v \!=\! 0.08, \delta Y_l \!=\! 0.14, \delta Y_r \!=\! -0.1, \delta  \alpha_l \!=\!0.2, \delta  \alpha_r \!=\! 0.2  \notag
\end{align}
For the kinematics-constrained VO system, we only add error terms to $\boldsymbol \xi_{ICR}$.
%
To show details in parameter convergence properties,
We carried out experiments on representative indoor and outdoor sequences, ``SEQ8-CP01, SEQ18-CP01, SEQ19-CP01".
%
In Fig.~\ref{fig:icr_converge_vio_xi}, the estimates of the full kinematic parameters $\boldsymbol \xi$ in VIO are shown, along with the corresponding $\pm 3\sigma$ uncertainty envelopes ( where $\sigma$ is the square root of the corresponding diagonal component of the estimated covariance matrix).
The convergence of $\boldsymbol \xi_{ICR}$ in VO are also shown in Fig.~\ref{fig:icr_converge_vo_icr}.
The results demonstrate that the kinematic parameters $\boldsymbol \xi$ in the VIO quickly converge to stable values, and remains slow change rates for the rest of the trajectory. 
Similar behaviours can also be observed for $\boldsymbol \xi_{ICR}$ when only a monocular camera and odometer sensors are used. 
The results exactly meet our theoretical expectations that $\boldsymbol \xi$ in VIO  and  $\boldsymbol{\xi}_{ICR}$ in VO are both locally identifiable under general motion. 
We also note that, since it is not feasible for obtaining high-precision ground truth for $\boldsymbol \xi$, the correctness of those values cannot be `directly' verified. Instead, they can be evaluated either based on the overall estimation results shown in the previous section or simulation results in Sec.~\ref{sec:sim} where ground truth $\boldsymbol \xi$ is known.

\subsubsection{Time efficiency}
The proposed ICR-kinematic based method is computationally effective to run in real time, thus is applicable in various robotic applications. In this section, we showcase the running time of the proposed method on a typical sequence ``SEQ21-CP01" with a traversed distance of 629.16 meters, and there are about 2400 cycles of sliding-window optimization occurring in our estimators.
As noted previously, the run time is also evaluated on a desktop with Intel Core i7-8700 @ 3.20GHz CPU. The implementation of our method is in C++ and adequately optimized for high efficiency. 
We disclose the average running time and its standard deviation of four main processing stages: feature detection, feature tracking, BA optimization, and marginalization. The ICR-based kinematic constraints are formulated and leveraged in BA optimization, as well as the other constraints.  The results are shown in Table~\ref{table:time}, including the three configurations  \textbf{VIO W/ $\boldsymbol \xi$},  \textbf{VO W/ $\mathrm{\mathbf{ICR}}$}, and \textbf{VIO W/O $\boldsymbol \xi$}. Comparing the runtime of \textbf{VIO W/ $\boldsymbol \xi$ } to the one of {\textbf{VIO W/O $\boldsymbol \xi$}}, it is obvious to find that introducing the  ICR-based kinematic constraints into VIO has a negligible burden on the computation. Considering the significant improvement in the pose estimation accuracy by incorporating kinematic constraints (see experiments in Sec.~\ref{sec:pose accuracy}), it is quite worthwhile to estimate the kinematic parameters and poses jointly by the proposed method.

\begin{table}[t] 
	\centering  
		\caption{Mean and (Standard Deviation) of the Running Time for the Main Processing Stages.}  
		\label{table:time}
		\setlength{\tabcolsep}{4pt}
		\resizebox{1.0\columnwidth}{!}{
		\begin{tabular}{c c c c 
			}  
			\toprule
			\multicolumn{1}{c}{\multirow{2}*{\bf Stages} }
			& \multicolumn{3}{c}{Running Time (ms)} \\
			\cline{2-4} 
			\multicolumn{1}{c}{}
			& \multicolumn{1}{c} {\textbf{VO W/ $\mathrm{\mathbf{ICR}}$  }}
			& \multicolumn{1}{c} {\textbf{VIO W/ $\boldsymbol \xi$ }}
			& \multicolumn{1}{c} {\textbf{VIO W/O $\boldsymbol \xi$}} \\ 
			
			\midrule
			{  { Feature Detection} } 
			&2.09(0.67) & 2.06(0.66) & 2.09(0.67) \\

			\midrule
			{  { Feature Tracking } } 
			&4.35(1.04) & 4.29(1.00) & 4.40(1.06)  \\

			\midrule
			{  { BA Optimization } } 
			& 4.09(1.33) & 4.77(2.11) & 4.35(1.54) \\
			
			\midrule
			{  { Marginalization } } 
			& 2.22(0.83) & 2.71(0.90) & 2.63(0.91) \\
			\bottomrule  
		\end{tabular} }
\end{table} 

\subsection{Simulation Experiments}
\label{sec:sim}
We also perform Monte-Carlo simulations to investigate our proposed method specifically for parameter calibration precision, since this cannot be verified in real-world tests.
The synthetic trajectory is generated by simulating a real-world trajectory with a length of $205.4 m$, using the method introduced in~\cite{li2014vision}. 
To generate noisy sensory measurements, we have used zero-mean Gaussian vector for all sensors with the following standard deviation (std) values. 
Pixel std for visual measurements is $0.6$ pixels, odometer stds for the left and right wheels are both $0.0245$ m/s, gyroscope and accelerometer measurement stds are $9\cdot10^{-4}$ rad/s and $1\cdot10^{-2}$ m/$s^2$, and finally the stds representing the random walk behavior of gyroscope and accelerometer biases are $1\cdot10^{-2}$ rad/$s^2$ and $1\cdot10^{-2}$ m/$s^3$ respectively. 
Additionally, since skid-steering kinematic parameters can not be known in advance, we initialize $\boldsymbol \xi$ in our simulation tests by adding an error vector to the ground truth values. The noise vector is sampled from zero-mean Gaussian distribution with std $8\cdot10^{-2}$ for all elements in $\boldsymbol \xi$.

To collect algorithm statistics, we conducted 15 Monte-Carlo tests and compute parameter estimation results for $\boldsymbol \xi$. Specifically, we computed the mean and std of calibration errors for all elements in $\boldsymbol \xi$, averaged from the Monte-Carlo tests. The estimation error for each element in $\boldsymbol \xi$ are:
$-0.0211 \pm 0.0095,~ 0.0102 \pm 0.0030,~ -0.0081 \pm  0.0026,~    0.0212 \pm  0.0109,~    0.0216 \pm 0.0108$. 
Those results indicate that, the skid-steering parameters can be accurately calibrated by significantly reducing uncertainty values.
It is also interesting to look into a representative run, in which the initial estimate of $\boldsymbol{\xi}$ is subject to the following error vector $ \delta\boldsymbol{\xi} = \begin{bmatrix} 0.15& 0.15& -0.15& 0.1& 0.1 \end{bmatrix} ^\top$. In this case, the calibration errors averaged over the second half of the trajectory are: $ -0.0276 \pm 0.0067,~   0.0199 \pm 0.0118,~     0.0054 \pm  0.0026,~     0.0192 \pm 0.0157,~    0.0189 \pm 0.0157 $.
%

Since simulation tests provide absolute ground truth, it is also interesting to investigate the accuracy gain by estimating $\boldsymbol \xi$ online. Fig.~\ref{fig:simu traj} demonstrates the estimated trajectory when $\boldsymbol \xi$ is estimated online, or $\boldsymbol \xi$ is fixed during estimation as well as the ground truth. This clearly demonstrates that, by the online estimation process, the localization accuracy can be significantly improved. The averaged RMSE of rotation and translation for those two competing methods in this Monte-Carls tests are
$ 0.042 \pm 0.023  rad,~ 2.051 \pm 0.830 m$ and $0.154 \pm 0.0635 rad,~ 4.617 \pm 2.563 m $, respectively.

\begin{figure}[th]
	\centering
	\includegraphics[width = 0.5\columnwidth]{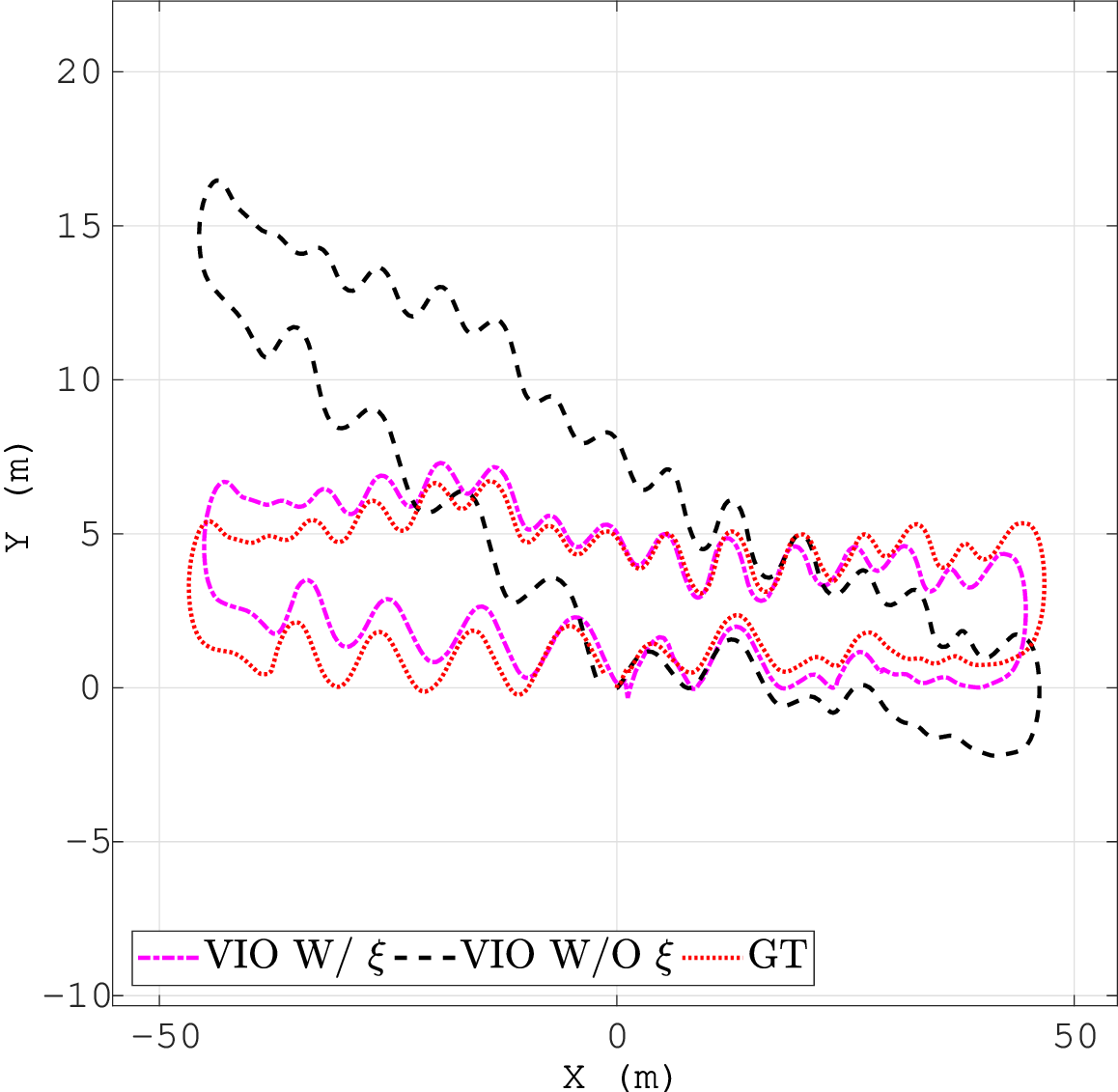}
	\captionsetup{justification=justified}
	\caption{ In simulation experiments, estimated Trajectories aligned with the ground truth trajectory.
	} 
	\label{fig:simu traj} 
	\vspace{-2em}
\end{figure}
\section{Conclusions and Future Work}\label{seq:conclusion}
In this paper, we propose a novel kinematics and pose estimation method specialized for skid-steering robots, where multi-modal measurements are fused in a tightly-coupled sliding-window BA.  
In particular, in order to compensate for the complicated track-to-terrain interactions, the imperfectness of mechanical design, mass center changes, tire inflation changes, and terrain conditions, we explicitly model the kinematics of skid-steering robots by using both track ICRs and correction factors, which are online estimated to prevent performance reduction in the long-term mission of skid-steering robots.
To guide the estimator design, we conduct detailed observability analysis for the proposed algorithm under different setup conditions. Specifically, we show that the kinematic parameter vector $\boldsymbol \xi$ is observable under general motion when measurements from an IMU are added and odometer-to-camera extrinsic parameters are calibrated offline. In other situations, degenerate cases might be entered and reduced precision might be incurred. Extensive real-world experiments including ablation study and simulation tests are also provided, which demonstrate that the proposed method is able to compute skid-steering kinematic parameters online and yield accurate pose estimation results.

There are also limitations to the proposed method. Notably, ICR-based kinematic model is only valid when a vehicle is operated in low dynamics. Although low dynamics are common for skid-steering robots, the feasibility at high speed is also essential. Moreover, abrupt changes of the traversed terrain or robot mechanism status will cause significant changes in the kinematic parameters. In that case, estimators will have some latency to converge. In the future, it is worthwhile investigating a more general kinematic model for mobile robots through data-driven deep-learning methods. A general kinematic model is supposed to function well regardless of the moving speed, interacted terrains, and robot mechanism status. It is also interesting to actively detect the abrupt changes in terrains and robot status. State estimation can converge quickly by enlarging the uncertainty of kinematic parameters when abrupt changes are detected.
{\tiny }

%

{
	\bibliographystyle{IEEEtran} 
	\bibliography{main}
}
\appendix
\subsection{Notation}
In order to ease the reading and understanding, we summarize the notations used in the main paper and the this appendix, which is shown in Table~\ref{tb:notation}.
\begin{table*}[htp]
	\centering
	\caption{Notation Glossary}
	\resizebox{\textwidth}{!}{
		\begin{tabular}{cc|cc|cc} \hline
			Symbol & Meaning & Symbol & Meaning  & Symbol & Meaning  \\ \hline    
			$\lbrace \mathbf{G} \rbrace$ & Global frame &  $\lbrace \mathbf{C} \rbrace$ & Camera frame & $\lbrace \mathbf{I} \rbrace$ & IMU frame \\ 
			$\lbrace \mathbf{O}\rbrace$ & Odometer frame & $^{\mathbf A}{\mathbf p}_{\mathbf B}$  & Position of $\lbrace\mathbf{B}\rbrace$ in $\lbrace\mathbf{A}\rbrace$ & $^{\mathbf A}_{\mathbf B} \mathbf R$ & Rotation from $\lbrace\mathbf{B}\rbrace$ to $\lbrace\mathbf{A}\rbrace$ \\
			$^{\mathbf A}_{\mathbf B}\mathbf{q}$ & Quaternion equivalent of $^{\mathbf A}_{\mathbf B} \mathbf R$ & $\mathbf{I}$ & Identity matrix & $\mathbf{0}$ & Zero matrix \\ 
			$\mathbf{x}$ & State vector & $\hat{\mathbf x}$ & Estimate of $\mathbf x$ &  $\delta \mathbf x$ & Error state of  $\mathbf x$ \\
			$\breve{\mathbf z}$ & Inferred measurement value of $\mathbf z$ &  $\delta \boldsymbol{\theta}$ & Error state of rotation matrix & $\lfloor \mathbf{v} \rfloor$ &  Skew-symmetric matrix of 3D vector $\mathbf{v}$ \\
			${}^{\mathbf  O}\omega_z$ & Angular velocity along $z$ axis in odometer frame & ${}^{\mathbf  O}v_{x}$ & Linear velocity along $x$ axis in odometer frame &  $\boldsymbol{\chi}_{\mathbf{O}}$ & A sliding window of odometer poses \\ 
			$\boldsymbol \xi$ & Full ICR-based kinematic parameters (5D) & $\boldsymbol \xi_{ICR}$ / $\mathrm{\mathbf{ICR}}$  & ICR kinematic parameters (3D) & $\boldsymbol \xi_{\alpha}$ & Scale factors (2D) of encoder readings  \\
			${}^{\mathbf{G}}\mathbf{v}_{\mathbf{I}_k}$ & Velocity of IMU in global frame at time $t_k$ & $\mathbf{b}_{a}$ & Bias of accelerometer &  $\mathbf{b}_{\omega}$ & Bias of gyroscope \\
			$o_{l}$ and $o_r$  &  Linear velocity of left and right wheels &  $\alpha_l$ and $\alpha_r$  & Scale factor of $o_l$ and $o_r$ & $o_{lm}$ and $o_{rm}$ & Measurements of $o_l$ and $o_r$ \\
			$\mathbf{m}$ & Motion manifold parameters & $\mathcal{C}$ & Cost function & $\mathbf{M}$ & Observability Matrix  \\ 
			$\mathbf{n}$ & Noise vector & $\mathbf{J}$ & Jacobian matrix & $\bm \Lambda$ & Information matrix \\
			$\mathcal{O}_m$ & Set of odometer measurements & $\mathbf{\mathcal{W}}_{m}$ & Set of gyroscope measurements & $\mathbf{\mathcal{A}}_{m}$ & Set of accelerometer measurements \\ \hline
	\end{tabular}}
	\label{tb:notation}
\end{table*}

\subsection{Bundle Adjustment Optimization}
Our optimization process closely follows the design of~\cite{Eckenhoff2019IJRR,zhang2019v}. Specifically, 
as illustrated in Fig.~\ref{fig:factorgraph}, the sliding-window bundle adjustment (BA) in our estimation algorithm seeks to iteratively minimize a cost function corresponding to a combination of sensor measurement constraints, motion kinematic constraints, and marginalized constraints.
\begin{align} \label{eq_ap:cost}
	\mathcal{C} = \mathcal{C}_{P} + \mathcal{C}_{V} + \mathcal{C}_{I} + \mathcal{C}_{O} +\mathcal{C}_{M} 
\end{align}
In what follows, we describe each of the cost terms. Firstly, 
the marginalized term $\mathcal{C}_{P}$ is critical to consistently keep the algorithm computational complexity bounded, by probabilistically removing the old states in the sliding window.
For a constraint $\mathcal{C}(\mathbf{x}_r, \mathbf{x}_m)$ involved with the old states needed to be marginalized  $\mathbf{x}_m$ and the remaining states $\mathbf{x}_r$, we compute the Hessian and gradient matrices with respect to $\begin{bmatrix} \mathbf{x}_m ^\top & \mathbf{x}_r ^\top \end{bmatrix}^\top$, which are denoted as:
\begin{align}
	\begin{bmatrix}
		\boldsymbol{\Lambda}_{r r} & \boldsymbol{\Lambda}_{r m} \\
		\boldsymbol{\Lambda}_{m r} & \boldsymbol{\Lambda}_{m m}
	\end{bmatrix}
	,\quad 
	\begin{bmatrix}
		\mathbf{g}_{r} \\
		\mathbf{g}_{m}
	\end{bmatrix}
\end{align}
The marginalization can be conducted by computing the marginalized Hessian and gradient matrices, i.e., ${\bm \Lambda}_{marg} = \boldsymbol{\Lambda}_{r r}-\boldsymbol{\Lambda}_{r m} \boldsymbol{\Lambda}_{m m}^{-1} \boldsymbol{\Lambda}_{m r}$ and ${\mathbf{g}}_{marg} = \mathbf{g}_{r}-\mathbf{\Lambda}_{r m} \mathbf{\Lambda}_{m m}^{-1} \mathbf{g}_{m}$, which represent the uncertainty information for the remaining states $\mathbf{x}_r$ in the current sliding window~\cite{Eckenhoff2019IJRR}. 
Once marginalization is performed, the prior cost function can be formulated to ensure the remaining states are characterized by the computed uncertainties:
\begin{align}
	\mathcal{C}_{P}(\mathbf{x}_r) = \frac{1}{2}\big{|}\big{|}\mathbf{x}_r\boxminus \hat{\mathbf{x}}_r \big{|}\big{|}_{\bm \Lambda_{marg}}^2+ \mathbf{g}_{marg}^{\top}\left(\mathbf{x}_r\boxminus \hat{\mathbf{x}}_r\right) 
\end{align}
The “boxminus” operator $\boxminus$ denotes the generalized minus operation, since we need to perform computations on the manifold~\cite{barfoot2017state}.
%
%
For the marginalization, it should be noted that, as shown in Fig.~\ref{fig:factorgraph}, for limiting the computational complexity, we only leverage the constraints from IMU $\mathcal{C}_{I}$ and odometer $\mathcal{C}_{O}$ between the latest frame $k$ and the second latest frame $k-1$. After $\mathcal{C}_{I}$ and $\mathcal{C}_{O}$ are minimized in the optimization, 
the information contained in them and the related states will be marginalized into the prior cost term. 

The camera term $\mathcal{C}_{P}$, IMU term $\mathcal{C}_{I}$, and motion manifold term $\mathcal{C}_{M}$ used in this work are similar to that of existing literature~\cite{li2013high,Eckenhoff2019IJRR,zhang2019v} but with dedicated design for ground robots. In general, the camera cost term models the geometrical reprojection error of point features in the keyframes, the IMU term computes the error of IMU states between two consecutive keyframes, and the manifold cost term characterizes the motion smoothness across the whole sliding window. The exact cost terms from camera, IMU, and manifold we used will be elaborated in the following sections. Finally, $\mathcal{C}_{O}$ denotes 
the error induced by wheel odometer measurements. This term is a function of robot pose, measurement input, as well as skid-steering intrinsic parameters, and have been discussed in details in the paper.

It should also be noted that in this work, we assume that the IMU, the wheel odometers, and the camera are synchronized by hardware. Integration of IMU and odometer measurements between the time instants of captured images are required in the constraints $\mathcal{C}_{I}$ and $\mathcal{C}_{O}$. However, since different types of measurements come at varying frequencies, it is unlikely to get IMU/odometer measurements at the exact time instants when capturing the images. 
Thus, we perform the linear interpolations of IMU and odometer measurements at the image capturing time for performing integration. 

\begin{figure} [bt]
	\centering
	\includegraphics[width = \columnwidth]{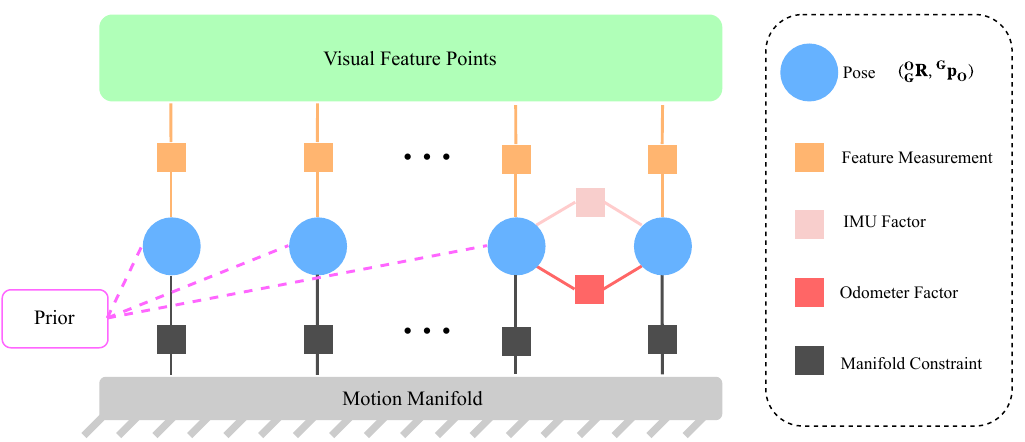}
	\captionsetup{justification=justified}
	\caption{In the proposed kinematics-constrained visual localization system for skid-steering robots, five different constraints are used in the sliding-window BA~\cite{zuo2019visual}: A prior encapsulates  the information about the current states due to marginalization of states and measurements (prior factor are related to all states which have marginalized measurements);
		Visual feature measurements connect the feature points in the map and the robot pose at the time when the image was recorded; 
		IMU integration factor summarizes the sequential IMU raw measurements between the two images (keyframes); 
		Odometry-induced kinematic factor summaries the sequential odometer measurements between the two images; 
		Motion manifold constraints enforce local smooth planar motions.
		Note that the IMU factor and the Odometry-induced kinematic factor only existing in the newest keyframe and the second newest keyframe. 
	} 
	\label{fig:factorgraph} 
	\vspace{-1em}
\end{figure}

\subsubsection{IMU Constraints}
The IMU provides readings of both accelerometer and gyroscope as follows:
\begin{subequations}
	\begin{align}
		\boldsymbol{\omega}_{m}&=\boldsymbol{\omega}_{\mathbf I}+\mathbf{b}_{\omega}+\mathbf{n}_{\omega} \\
		\mathbf{a}_{m}&=\mathbf{a}_{\mathbf{I}} - {}^{\mathbf I}_{\mathbf G} \mathbf{R} ^{\mathbf G} \mathbf{g}+\mathbf{b}_{a}+\mathbf{n}_{a}
	\end{align}
\end{subequations}
where $^{\mathbf G} \mathbf{g}$ is the known global gravity vector, $\mathbf{b}_{\omega}$ and $\mathbf{b}_{a}$ the time-varying gyroscope and accelerator bias vectors, and $\mathbf{n}_{\omega}$ and $\mathbf{n}_{a}$ denote white Gaussian measurement noise.
The IMU integration process is characterized by:
\begin{align}
	\hat{\mathbf{x}}_{I_k} \!\!=\!\! \left[ {}^{\mathbf G}\hat{\mathbf{p}}^T_{\mathbf{O}_k}, {}^{\mathbf G}_{\mathbf{O}_k}\hat{ \mathbf{q}}^T, {}^{\mathbf G}\hat{\mathbf{v}}^T_{\mathbf{I}_k}, \hat{\mathbf{b}}^T_{a_{k}}, \hat{\mathbf{b}}^T_{\omega_{k}} \right]^T \!\!=\!\! f(\hat{\mathbf{x}}_{I_{k-1}}, \mathbf{\mathcal{W}}_{m},\mathbf{\mathcal{A}}_{m}) 
\end{align}
where 
$\mathbf{\mathcal{W}}_{m},\mathbf{\mathcal{A}}_{m}$ are the gyroscope and accelerometer measurements during the time interval $t \in \left(t_{k-1}, t_k \right)$, and $f(\cdot)$ is the IMU integration function.
Since the IMU integration is widely  investigated in research communities~\cite{mourikis2007multi,li2013high,Eckenhoff2019IJRR}, and we here ignore the details on $f(\cdot)$.
%
%
The associated uncertainty matrix (i.e., linearized noise information matrix) of the prediction process $\bm \Lambda_{I}$ can also be obtained 
by linearizing the function $f(\cdot)$. 
As a result, the IMU cost term can be summarized by:
\begin{align}\label{eq:imu error}
	\mathcal{C}_{I}(\mathbf{x}_{I_k}, \mathbf{x}_{I_{k-1}})  = \left|\left| {\mathbf{x}}_{I_k} \boxminus f({\mathbf{x}}_{I_{k-1}},\mathbf{\mathcal{W}}_{m},\mathbf{\mathcal{A}}_{m}) \right| \right|_{\bm \Lambda_{I}}^2
\end{align}
which provides pose constraints between consecutive keyframes. We also note that, the IMU cost function requires odometer to IMU extrinsic parameters to transform states in odometer frame to IMU frame, which are also calibrated offline. After minimizing  Eq.~\ref{eq:imu error}, the states $\begin{bmatrix}
	{}^{\mathbf G}\hat{\mathbf{v}}^T_{\mathbf{I}_{k-1}},& \hat{\mathbf{b}}^T_{a_{k-1}},& \hat{\mathbf{b}}^T_{\omega_{k-1}}  \end{bmatrix}^T$ will be marginalized, and the contained information will be incorporated into the prior cost term.

\subsubsection{Motion Mainfold constraints}
\label{app:manifold}
Finally, 
since the skid-steer robot navigates on ground surfaces,
its trajectories can also be constrained by the prior knowledge about the shape of surface manifold. Specifically, we utilize our method presented in~\cite{zhang2019v} to approximate ground surfaces using quadratic polynomials, the following holds:
\begin{align}~\label{eq:manifold}
	& m_p({}^{\mathbf G}\mathbf{p}_{\mathbf O}) \!\!=\!\!\frac{1}{2} \! {}^\mathbf { G}\mathbf{p}_{\mathbf O_{xy}}^\top  \! \begin{bmatrix}
		a_1 & a_2\\ a_2 & a_3
	\end{bmatrix} \! {}^\mathbf { G}\mathbf{p}_{\mathbf O_{xy}} \!\!+\!\! \begin{bmatrix}
		b_1 \\ b_2
	\end{bmatrix}^\top \!\!  {}^\mathbf { G}\mathbf{p}_{\mathbf O_{xy}} \notag\\ & +\!\! {}^{\mathbf G}\mathbf{p}_{\mathbf{O}z} \!\!+ \!\! c 
	,{}^\mathbf{G}\mathbf{p}_{\mathbf O} = \begin{bmatrix}
		{}^\mathbf{ G}\mathbf{p}_{\mathbf O_{xy}}^\top & {}^ \mathbf { G}\mathbf{p}_{\mathbf O_z}^\top \end{bmatrix}^\top
\end{align}
where $\mathbf{m} = \left[  a_1, a_2, a_3, b_1, b_2, c \right]^\top$ to denote the manifold parameters. 
By utilizing the quadratic surface approximation, we are able to define the following cost function for both rotation and position terms~\cite{zhang2019v}:
\begin{align}
	\mathbf{m}_r({}^{\mathbf G}_\mathbf{O}\mathbf{R}, {}^{\mathbf G} \mathbf{p}_\mathbf{O} ) = \lfloor {}^{\mathbf G}_\mathbf{O}\mathbf{R} \mathbf{e}_3  \rfloor_{12} \frac{\partial m_p }{\partial {}^{\mathbf G} \mathbf{p}_\mathbf{O} }^\top \!=\! \mathbf{0} ,\text{and}\,m_p({}^{\mathbf G}\mathbf{p}_{\mathbf O}) \!\!=\!\!0 \notag
\end{align}
where $\lfloor \mathbf{v} \rfloor_{12}$ denotes the first and second rows of a symmetric matrix of the 3D vector $\mathbf{v}$. 
The above constraints reflect the fact that, the motion
manifold $\mathbf{m}$ has explicitly defined roll and pitch of a ground robot, which should be in consistent with the rotation ${}^{\mathbf G}_\mathbf{O}\mathbf{R}$.
Therefore, the motion manifold cost term for keyframes  in the sliding window can be written as:
\begin{align}
	\mathcal{C}_{M}({}^{\mathbf G}_{\mathbf{O}_i}{ \mathbf{R}}, {}^{\mathbf G}{\mathbf{p}}_{\mathbf{O}_i},  \mathbf{m}_k, \mathbf{m}_{k-1} ) \!=\! \Bigg{|} \Bigg{|} \begin{bmatrix}
		\mathbf{m}_{k} - \mathbf{m}_{k-1}\\
		m_p({}^{\mathbf G}\mathbf{p}_{\mathbf O_i}) \\ \mathbf{m}_r( {}^{\mathbf G}_{\mathbf{O}_i}\mathbf{R}, {}^{\mathbf G} \mathbf{p}_{\mathbf{O}_i}) 
	\end{bmatrix} \Bigg{|} \Bigg{|}_{\bm \Lambda_{m}}
\end{align}
for all $i \in \left[k-s+1, k\right]$. $\mathbf{m}_{k-1} $ and $\mathbf{m}_{k} $ denotes the manifold parameters characterize the motion manifold across the last and the current sliding window, respectively.
Moreover, ${\bm \Lambda_{m}}$ is the information matrix describing the uncertainties in both localization states and the surface manifold approximation itself, which is described in detail in~\cite{zhang2019v}.

\subsection{More Experimental Results}
\label{ap_sec:exp}

In our experiments, we used two testing skid-steering robots based on the commercially available Clearpath Jackal robot with both `localization' sensors and `ground-truth sensors' equipped.
All the experiments are conducted by first dataset collection and subsequently offline processing using an Intel Core i7-8700 @ 3.20GHz CPU, to allow repeatable comparison between different methods.

\subsubsection{Real-world Experiment}
In the first set of experiments, we focus on validating the effectiveness of the proposed skid-steering model as well as the localization algorithm. Specifically, we investigated the localization accuracy by estimating skid-steering kinematic parameters $\boldsymbol \xi$  online and compared that to the competing methods.
To demonstrate the generality of our method, we conducted experiments under various environmental conditions.
The experimental environments involved in our robotic data collection include (a) lawn, (b) cement brick, (c) wooden bridge, (d) muddy road, (e) asphalt road, (f) ceramic tiles, (g) carpet, and (h) wooden floor.

We note that since GPS signal is not always available in all tests (e.g., indoor tests), we use both final drift and root-mean-squared error (RMSE) of absolute translational error (ATE)~\cite{zhang2018tutorial} as our metrics. To make this possible, we started and terminated each experiment in the same position. 
In the Tables of experimental results,  we highlight the best pose estimations in bold, while the bad pose estimations (final drift or RMSE of ATE over 12m) by underlines.
\begin{table*}[t]
	\renewcommand{\arraystretch}{1.5}
	\caption{Final Drift for three different setups on 23 sequences which covers 8 types of terrain.}
	\label{tb:drift}
	\vspace{-1em}
	\begin{center}
		\resizebox{0.95\textwidth}{!}
		{
			\begin{tabular}{c|c|c|c|c|c|c|c|c|c|c|c|c|c|c}\cline{4-15}
				\multicolumn{3}{c|}{}&\multicolumn{4}{c|}{\textbf{VIO W/ $\boldsymbol \xi$ }}&\multicolumn{4}{c|}{\textbf{VO W/ $\mathrm{\mathbf{ICR}}$  }}&\multicolumn{4}{c}{\textbf{VIO W/O $\boldsymbol \xi$}}\\\cline{1-15}
				\textbf{Sequence}&\textbf{ Length(m)}& \textbf{Terrain}&\textbf{Norm(m)}&\textbf{x(m)}&\textbf{y(m)}&\textbf{z(m)}&\textbf{Norm(m)}&\textbf{x(m)}&\textbf{y(m)}&\textbf{z(m)}&\textbf{Norm(m)}&\textbf{x(m)}&\textbf{y(m)}&\textbf{z(m)}\\\hline
				SEQ1-CP02 & 232.30& (b) & \textbf{3.644} & 0.340& 3.604& 0.420&        4.070 & 0.393& 4.048& 0.146 &        6.246& 0.733& 4.841& 3.878 \\
				SEQ2-CP01& 193.63& (f) &    \textbf{0.800} & -0.544& 0.420& 0.409&     \clr{16.698} & -10.613& 12.886& 0.345  &    3.705& 2.779& -1.155& 2.162 \\
				SEQ3-CP01& 632.64& (b,f) & \textbf{7.150} & 6.446& 2.732& 1.457&       \clr{146.033} & 142.976& -26.395& 13.670 &         \clr{28.911}& -27.506& -6.272& 6.315 \\
				SEQ4-CP01& 629.96& (b,f) & \textbf{1.427} & 0.268& 0.815& 1.139&    \clr{139.484}& 135.732& -30.425& -10.344 &        \clr{35.958}& -33.928& -10.114& 6.291 \\
				SEQ5-CP01& 626.83& (b,f) &\textbf{ 8.109} & 7.589& 2.391& 1.563&      \clr{157.703}& 153.233& -37.044& -4.188 &        \clr{31.044}& -29.317& -7.951& 6.405 \\

				SEQ6-CP01& 212.59& (g) & \textbf{7.399} & 4.353& -5.959& 0.528&           11.269& -11.171& 1.073& -1.017 &          10.005& 8.170& -5.289& 2.317 \\
				SEQ7-CP01& 51.44& (a) &  0.206 &-0.144& -0.117& -0.090&      \textbf{0.201}&  -0.194& -0.045& 0.028 &        0.835& -0.302& -0.156& 0.763 \\
				SEQ8-CP01& 204.81& (e) &  0.766 &-0.271& -0.005& 0.716&         \textbf{0.626}& -0.575& 0.111& 0.223     &    2.218    & -0.440& 0.066& 2.173 \\
				SEQ9-CP01& 77.63& (c) &  0.319 & 0.013& 0.318& -0.020&        \textbf{0.271}&  -0.020& 0.270& 0.016 &      1.013& 0.246& 0.240& 0.953 \\
				SEQ10-CP01& 27.09& (a) &  0.204 & -0.114& -0.005& 0.170&     \textbf{ 0.077}&  -0.072& 0.027& 0.008 &        0.519&  -0.137& -0.154& 0.476 \\
				
				SEQ11-CP01& 270.41& (e,b) & \textbf{0.644} & -0.103& -0.387& 0.504&      1.298& -1.284& -0.162& 0.103 &     3.148& 0.262& -0.326& 3.120 \\
				SEQ12-CP01& 436.19& (e) & \textbf{0.734} & 0.116& -0.143& 0.710&      11.614& -2.261& -11.136& 2.403 &     7.084& 0.241& 5.484& 4.478 \\
				SEQ13-CP01& 28.64& (d) &\textbf{ 0.093} & -0.062& 0.054& 0.043&     0.161&  -0.098& 0.116& 0.053 &        0.350    & 0.059& 0.101& 0.330 \\
				SEQ14-CP01& 372.15& (b) & 10.016 & 9.903& 1.027& 1.099&            \textbf{3.838}&  -3.835& 0.071& 0.129 &        \clr{13.261} &12.707& 1.603& 3.440 \\
				SEQ15-CP02& 81.03& (h)  & 2.573 & -2.428& 0.787& -0.320&      \textbf{2.179}&  -1.078& 1.882& -0.206 &        2.564    &-2.268& 0.201& 1.179 \\
			
				SEQ16-CP02& 53.49& (h)  & 0.702 & -0.501& 0.468& -0.153&         \textbf{0.558}& -0.200& 0.515& -0.079 &    1.093    & -0.758& 0.111& 0.780 \\
				SEQ17-CP01& 110.55& (b) & 1.048 & -0.083& -1.039& 0.106&         \textbf{0.420}& 0.048& -0.409& -0.086 &    1.346& -0.286& -0.827& 1.023 \\
				SEQ18-CP01& 104.63& (h) & \textbf{0.488} &0.228& -0.404& 0.152&        0.584& 0.088& 0.530& 0.228 &        1.378    & 0.422& -0.769& 1.062 \\
				SEQ19-CP01& 214.66& (b,h) & \textbf{0.999} & -0.315& -0.685& 0.655&       1.743& 1.334& 1.121& 0.040 &           2.492 & -0.319& -1.392& 2.04 \\
				SEQ20-CP01& 254.30& (b,h) & \textbf{0.838} & -0.195& -0.038& 0.814&      2.584& 1.608& 2.022& 0.045 &        3.179 &-0.616& -2.036& 2.362 \\
				
				SEQ21-CP01& 629.16& (b,f) & \textbf{1.829} & 0.008 & 1.257 & 1.329 &     \clr{133.916}& 132.305& -12.902& 16.198 &           \clr{16.918}& 14.959& 5.358& 5.810 \\
				SEQ22-CP01& 633.53& (b,f) & \textbf{4.405} & 3.782 & 1.882 & 1.249 &        \clr{119.478}& 117.966& -14.111& -12.644 &         \clr{13.728}& 11.988& 3.693& 5.577 \\
				SEQ23-CP01& 651.94& (b) & 3.428 & -1.287 & -2.768 & 1.560 &        \textbf{1.379}& -0.413& 1.313& 0.095 &          5.197& -0.781& 0.888& 5.060
				\\
				\textbf{Mean}& & &\textbf{2.514}& 1.174& 0.183& 0.610& 32.878& \clr{28.429}&
				-4.637& 0.225& 8.356& -1.917& -0.602& 2.956
				\\\hline
			\end{tabular}
		}
	\end{center}
	\vspace{-1em}
\end{table*}

%
We  conducted experiments to show the benefits gained by modeling and estimating skid-steering parameters online. In this experiment, three sets of setup are compared, i.e., two provably observable methods and one baseline method. Specifically, those methods are 1) \textbf{VIO (visual-inertial odometry) W/ $\boldsymbol \xi$ :}using measurements from a monocular camera, an IMU, and odometer via the proposed estimator by estimating the full 5 skid-steering kinematic parameters $\boldsymbol \xi$ online; 2) \textbf{VO (visual odometry) W/ $\mathrm{\mathbf{ICR}}$ :} using monocular camera and odometer measurements (without an IMU), and performing localization by estimating the 3 ICR parameters $\boldsymbol{\xi}_{ICR}$ online; 3) \textbf{VIO W/O $\boldsymbol \xi$ :} using measurements from monocular camera, an IMU and odometer, and utilizing the traditional differential drive kinematics for localization without explicitly modeling $\boldsymbol \xi$. We note that, in traditional methods when $\boldsymbol \xi$ is not modeled,  differential drive kinematics model can be considered as one-parameter (i.e., $b$) approximation of skid-steering kinematics.

In Table.~\ref{tb:drift}, we show the final drift errors on 23 representative sequences, which cover all the eight types of terrains (a)-(g).
%
We also note that some sequences cover multiple types of terrains. 
Fig.~\ref{fig:visual_map} shows the trajectory and visual features estimated by the proposed method on sequence ``CP01-2019-05-27-14-50-49", in which the robot traversed the outdoors under terrain (b) and indoors under terrain (h).
Since the two robots were used for data collection, we use the notation ``CP01, CP02" to denote the robot names in Table.~\ref{tb:drift}. In addition, we highlight the results with severe drift (error of norm is over 12m) by underlines.

In some sequences where GPS signals were available across the entire data sequence, we also evaluated the root mean square errors (RMSE)~\cite{Bar-Shalom1988} of absolute translational error (ATE)~\cite{zhang2018tutorial}. To compute that, we interpolated the estimated poses to get the ones corresponding to the timestamp of the GPS measurements. The RMSE errors are shown in Table.I in our paper. The results demonstrate that estimating $\boldsymbol \xi$ is beneficial for trajectory tracking. 
Representative Trajectory estimates on representative sequences are shown in Fig.~\ref{fig:traj_gps}.

\begin{figure}[bth] 
	\centering
	\includegraphics[scale=.3]{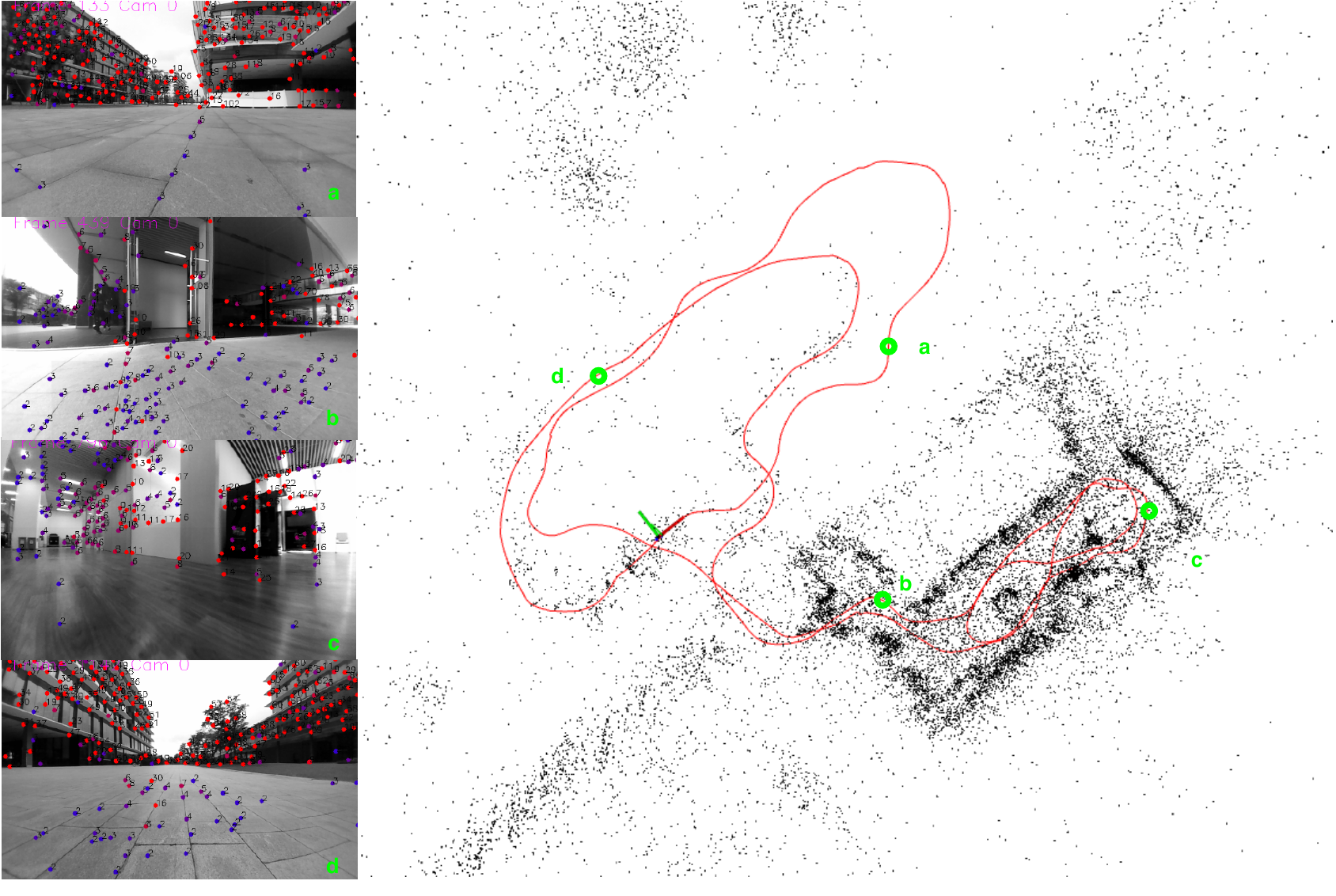} 
	\captionsetup{justification=justified}
	\caption{Skid-steering robots traversed outdoors and indoors~\cite{zuo2019visual}. The left part shows the representative images with visual features recorded at positions marked by green circles respectively. The right part shows the estimated trajectory red curve, and recovered 3D visual landmarks by black dots.} \label{fig:visual_map}
\end{figure}
\begin{figure*}[hbt]
	\centering 
	\subfigure[]{ 
		\includegraphics[width=1.6in]{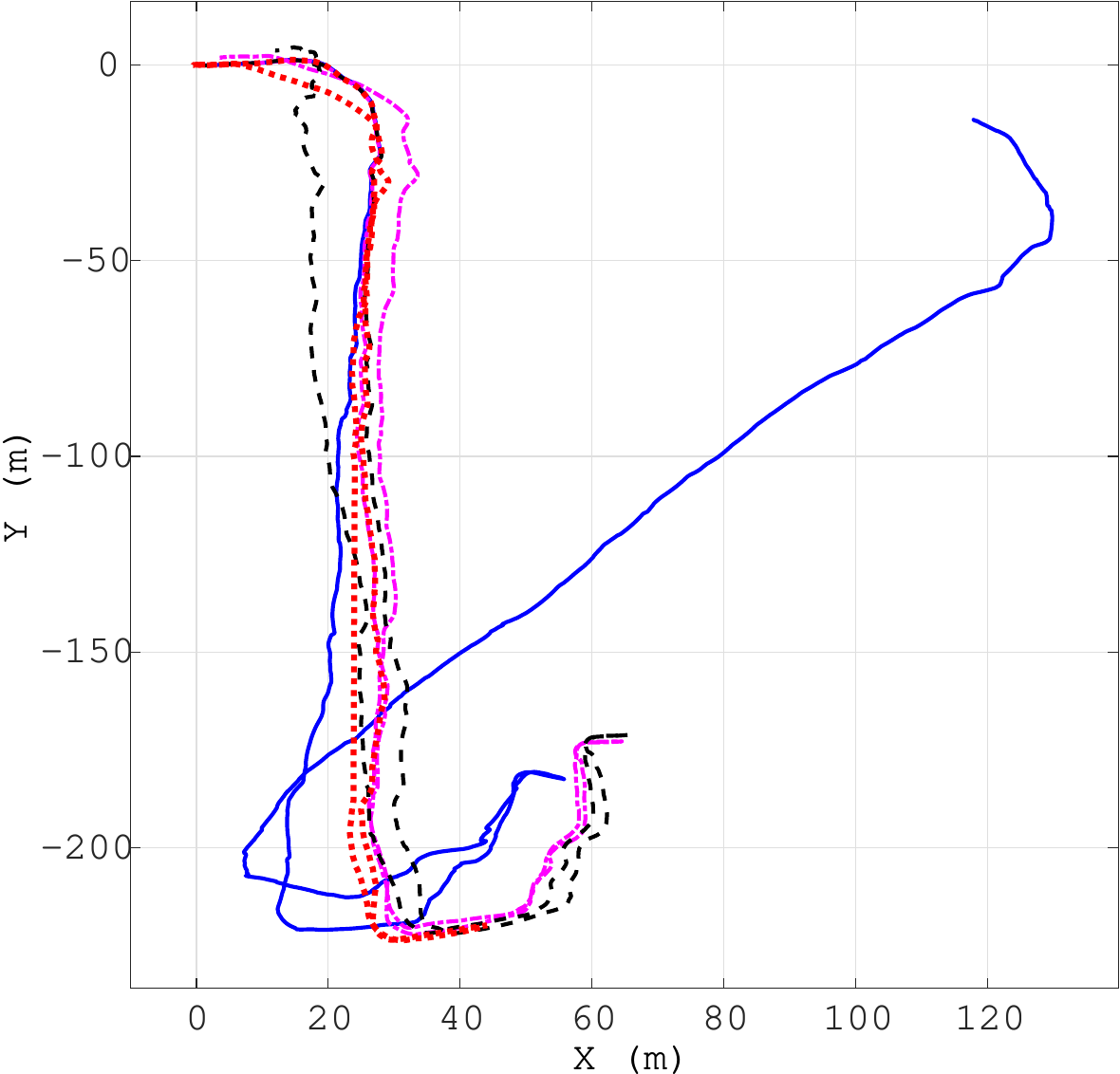} 
	} 
	\subfigure[]{ 
		\includegraphics[width=1.6in]{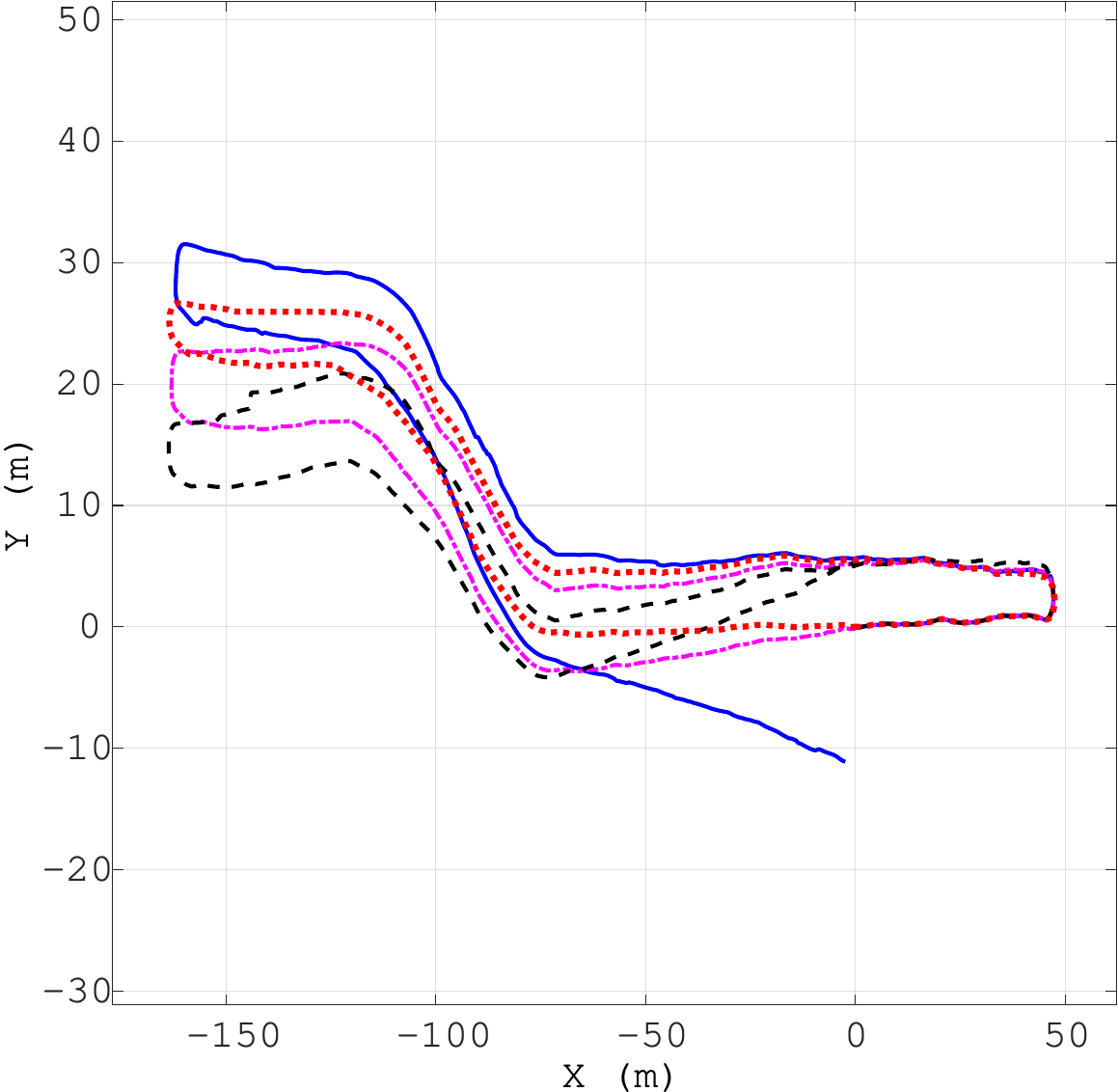} 
	}
	\subfigure[]{ 
		\includegraphics[width=1.6in]{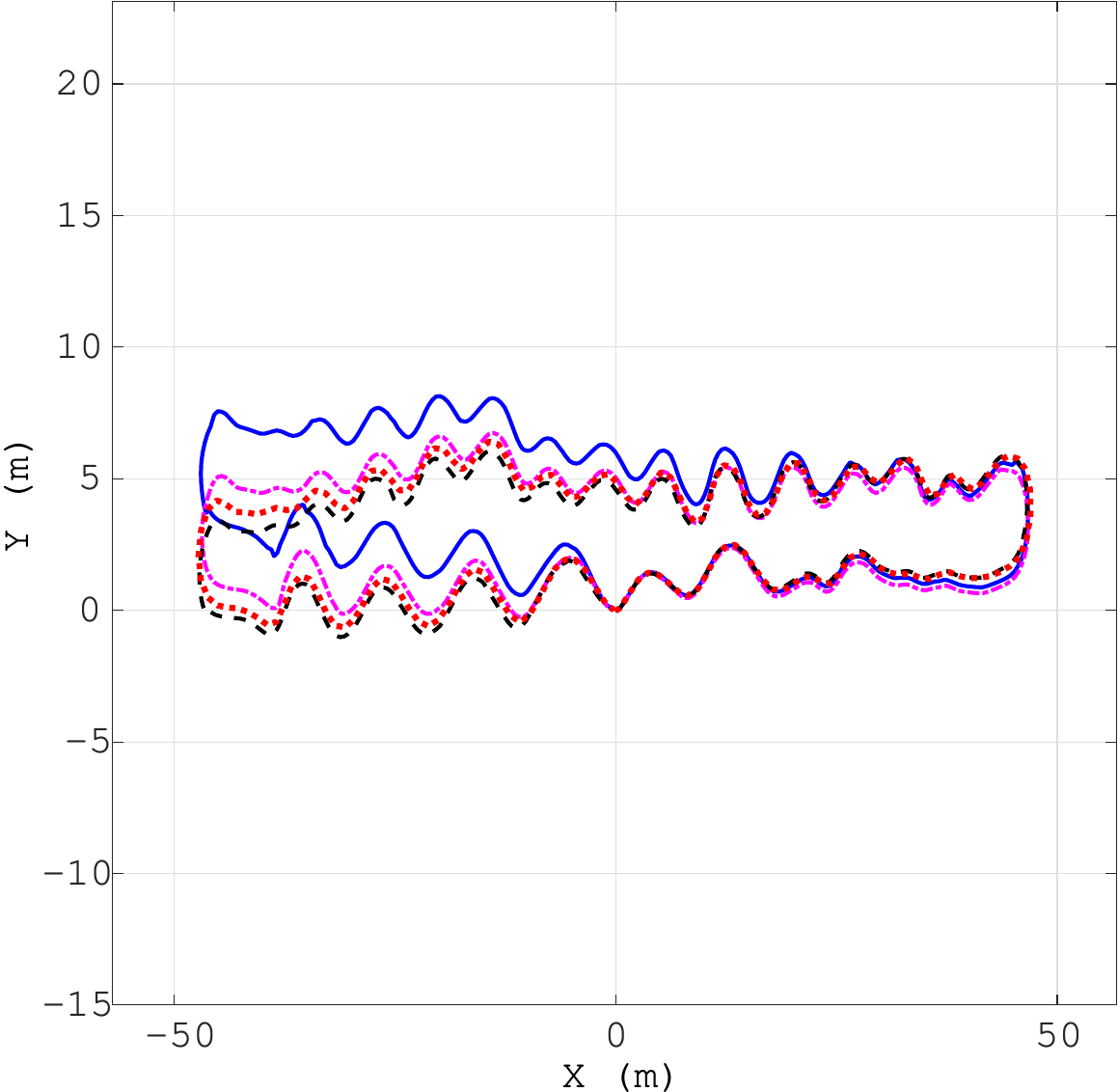} 
	} 
	\subfigure[]{ 
		\includegraphics[width=1.6in]{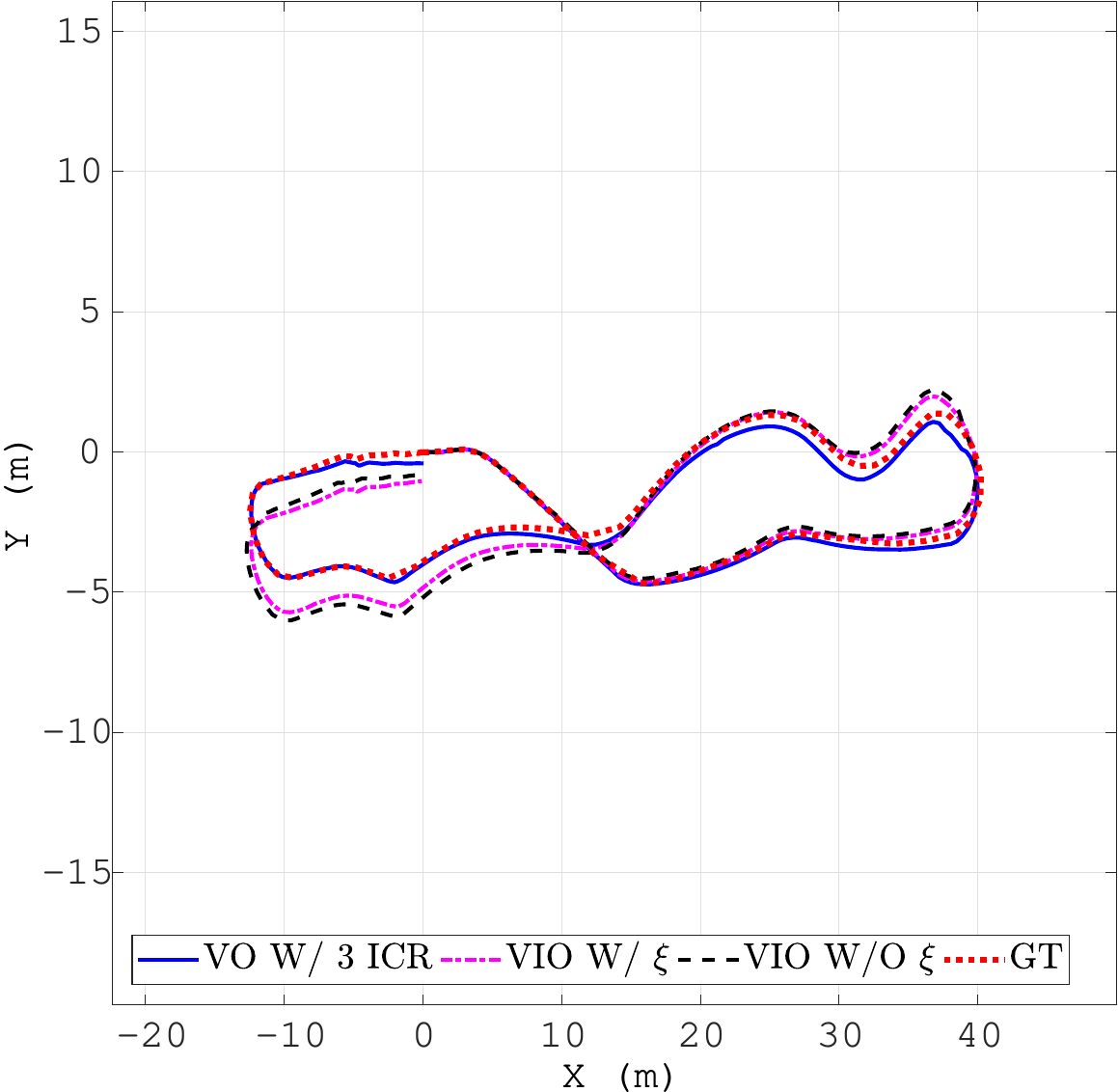} 
	}
	\captionsetup{justification=justified}     
	\caption{The trajectories of RTK-GPS (ground truth) and the estimated trajectory by the localization method: 1) VIO with online estimating the kinematic parameters $\boldsymbol \xi$; 2) VO with estimating $\boldsymbol{\xi}_{ICR}$ online; 3) VIO without estimating  $\boldsymbol \xi$.    
	}
	~\label{fig:traj_gps}
\end{figure*}

\subsubsection{Ablation Study}
In this section, we examine the advantages of online estimating the full kinematics parameters $\boldsymbol{\xi}$ in the kinematics-constrained VIO systems, which enables life-long high-precision localization for real-world robotic deployment. 
In fact, the mechanical parameters of a real robot can not be constants all the time. Some components might be of slow changes (e.g., height and width), and some drift relatively fast (e.g., weights or tire pressures). All those factors might lead to varying kinematic parameters, and we here verify the effectiveness of our method for handling them.

Specifically, we collected datasets under the following practically commonly-seen configurations for our skid-steering robot: (\romannumeral1) normal; (\romannumeral2) carrying a package with the weight around 3 kg; (\romannumeral3) under low tire pressure; (\romannumeral4) carrying a 3-kg package and with low tire pressure. In configurations (\romannumeral1) to (\romannumeral4), the actual kinematic parameters vary significantly and also deviate from our initial guess. 
Similar to the previous tests, three algorithms were conducted here by using the measurements from a camera, an IMU and odometer: 1) estimating 
$\boldsymbol{\xi}_{ICR}$;  2) estimating $\boldsymbol{\xi}$; 3) used fixed  $\boldsymbol{\xi}$ with a relatively good initial guess, obtained by the final estimate of running our online estimation algorithm.
%

%
We conducted experiments on 8 sequences named ABL-SEQ1 $\sim$ ABL-SEQ8, and each configuration corresponds to two sequences in ascending order (e.g., ABL-SEQ1 and ABL-SEQ2 correspond to the `normal' condition).
The evaluation methods used here are as same as the ones used in the previous section, which include both final drift and RMSE.
In Table.~\ref{tb:drift disturb}, we show the final drift of three different localization methods. On the other hand, the RMSE of ATE is given in Table.~\ref{tb:gpsseq_ate_disturb}. Additionally, RPE was shown in Fig.~\ref{fig:gps rpe disturb}.
Those results demonstrate that, when a robot is in normal mechanical condition, and the road condition is without large variance, there are minor differences between estimating the full 5 kinematic parameters $\boldsymbol{\xi}$ and online estimating only the 3 parameters $\boldsymbol{\xi}_{ICR}$, when good correction factors $\alpha_l, \alpha_r$ are given and kept constant. This is due to the fact that the correction factors reflect the transmission efficiency of the robot and are not subject to fast changes in the general case. However, if there are noticeable changes in the robotic mechanical condition, e.g., weight and center of mass change by carrying a large package or tire pressure changes after long-term usage, the correction factors $\alpha_l, \alpha_r$ will be changed significantly. In such cases, the overall estimation algorithm benefits significantly by online estimating $\boldsymbol \xi$. 
We also show the estimated trajectories compared with RTK-GPS measurement in Fig.~\ref{fig:traj_gps_disturb}, for the representative runs.
\begin{table*}[tbh]
	\renewcommand{\arraystretch}{1.5}
	\caption{Ablation Experiments Results: Final drift.}
	\label{tb:drift disturb}
	\begin{center}
		\resizebox{0.95\textwidth}{!}
		{
			\begin{tabular}{c|c|c|c|c|c|c|c|c|c|c|c|c|c|c|c}\cline{5-16}
				\multicolumn{4}{c|}{}&\multicolumn{4}{c|}{\textbf{VIO W/ $\mathrm{\mathbf{ICR}}$ }}&\multicolumn{4}{c|}{\textbf{VIO W/ $\boldsymbol{\xi}$}}&\multicolumn{4}{c}{\textbf{ VIO W/ Fixed $\boldsymbol{\xi}$ }}\\\cline{1-16}
				\textbf{Sequence}&\textbf{ Length(m)}& \textbf{Terrain}&\textbf{Config.}&\textbf{Norm(m)}&\textbf{x(m)}&\textbf{y(m)}&\textbf{z(m)}&\textbf{Norm(m)}&\textbf{x(m)}&\textbf{y(m)}&\textbf{z(m)}&\textbf{Norm(m)}&\textbf{x(m)}&\textbf{y(m)}&\textbf{z(m)}\\\hline
				ABL-SEQ1 & 167.30& (e) & (\romannumeral1) & 0.316& 0.038& 0.167& 0.266&            \textbf{0.304}& 0.038& 0.144& 0.265&            0.832& 0.029& -0.796& 0.240\\
				ABL-SEQ2 & 147.76 & (e)  & (\romannumeral1) & 0.349& -0.109& 0.235& 0.233&            \textbf{0.336}& -0.110& 0.216& 0.233&         0.564& -0.190& -0.495& 0.191\\
				
				ABL-SEQ3 & 152.23& (e) & (\romannumeral2) & 0.318&  -0.193&  -0.195& 0.162       & \textbf{0.311}&   -0.192&  -0.183&  0.162        & 0.819&   -0.240&  -0.771&  0.137 \\
				ABL-SEQ4 & 152.80& (e) & (\romannumeral2) & 0.406&  -0.240&  -0.252&  0.208     & \textbf{0.400}&  -0.240&  -0.243&  0.209        & 0.629&  -0.185&  -0.578&  0.165 \\

				ABL-SEQ5 &237.36 & (e)   & (\romannumeral3) & 8.700& 0.046 & -8.699& 0.144&         \textbf{7.222}& -0.069& -7.220& 0.143&        9.755& 0.037& -9.753& 0.163\\ 
				ABL-SEQ6 & 232.43 & (e)   & (\romannumeral3) & 8.102& -0.161& -8.098& 0.198&         \textbf{6.696}& -0.242& -6.689& 0.195&            8.846& -0.143& -8.842& 0.233\\
				ABL-SEQ7 & 232.54 & (e)   & (\romannumeral4) & 9.509& 0.189& -9.505& 0.196&         \textbf{7.771}& 0.053& -7.769& 0.199&             10.502& 0.161& -10.498& 0.243\\
				ABL-SEQ8 & 233.07& (e)   & (\romannumeral4) & 10.055& 0.256& -10.050& 0.182&           \textbf{8.204}& 0.080& -8.202& 0.164&         10.601& 0.173& -10.598& 0.198
				\\
				\textbf{Mean}& & & & 4.719& -0.022& -4.550& 0.199& \textbf{3.906}&   -0.085& -3.743&
				0.196&   5.319&   -0.045&  -5.291&  0.196 
				\\\hline
			\end{tabular}
		}
	\end{center}
\end{table*}
\begin{table*}[hbt]
	\renewcommand{\arraystretch}{1.5}
	\caption{Ablation Experiments Results: RMSE of ATE (m).}
	\label{tb:gpsseq_ate_disturb}
	\begin{center}
		{
			\begin{tabular}{C{2.6cm}C{1.2cm}C{1.2cm}C{1.2cm}C{1.2cm}C{1.2cm}C{1.2cm}C{1.2cm}C{1.2cm}C{1.2cm}C{1.2cm}C{1.2cm}} \toprule
				& ABL-SEQ1& ABL-SEQ2& ABL-SEQ3& ABL-SEQ4& ABL-SEQ5& ABL-SEQ6& ABL-SEQ7& ABL-SEQ8& \textbf{Mean} \\\hline
				{\textbf{VIO W/ $\mathrm{\mathbf{ICR}}$ }} &    0.24 & 0.17 & 0.14 & 0.16 & 2.20 & 2.20 & 2.54 &  2.54& 1.27 \\
				{\textbf{VIO W/ $\boldsymbol{\xi}$}} &    0.23 & 0.17 & 0.14 & 0.16 & 1.85 & 1.87 & 2.13 &  2.09&\textbf{1.08}   \\
				{\textbf{ VIO W/ Fixed $\boldsymbol{\xi}$ }} &    0.28 & 0.24 & 0.27 & 0.23 & 2.49 & 2.40 & 2.82 &  2.75&1.44  \\  \hline
			\end{tabular}
		}
	\end{center}
\end{table*}
\begin{figure}[hbt]
	\centering
	\includegraphics[width = \columnwidth]{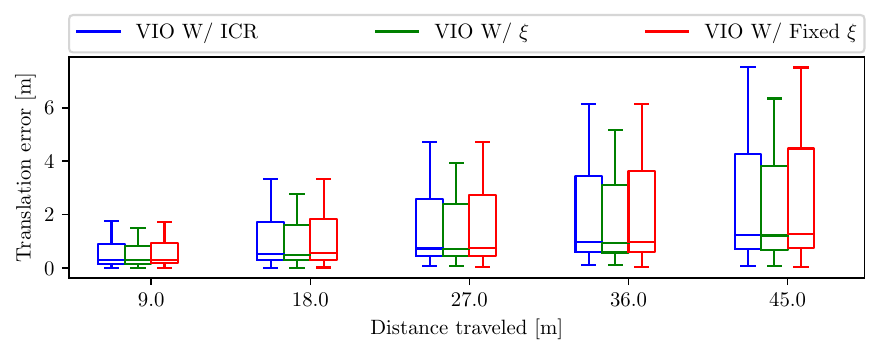}
	\captionsetup{justification=justified}
	\caption{Ablation Experiments Results: boxplot of the relative trajectory error statistics over all the sequences where RTK-GPS measurements are available. This plot will be best seen in color.
	} 
	\label{fig:gps rpe disturb} 
\end{figure}
\begin{table}[hbt]
	\centering
	\captionsetup{justification=justified}
	\caption{
		Ablation Experiments Results: Mean of RPE (m) for Different Segment Length.
	}
	{
		\begin{tabular}{C{2.0cm}C{1.5cm}C{1.5cm}C{1.50cm}} \toprule
			\textbf{Segment Length (m)} & {\textbf{VIO W/ $\mathrm{\mathbf{ICR}}$ }} & {\textbf{VIO W/ $\boldsymbol{\xi}$}} & {\textbf{ VIO W/ Fixed $\boldsymbol{\xi}$ }} \\ \hline
			9.00m &     0.53 & 0.48 &  0.54\\
			18.00m &    1.00 & 0.91 &  1.01  \\
			27.00m &     1.47 & 1.33 &  1.49  \\
			36.00m &     1.92 & 1.75 &  1.95 \\ 
			45.00m &     2.36 & 2.14 &  2.39 \\ \hline
		\end{tabular}    
	}
	\label{tb:gpsseq_rpe_disturb}
\end{table}
\begin{figure}[hbt]
	\centering 
	\subfigure[]{ 
		\includegraphics[width=1.6in]{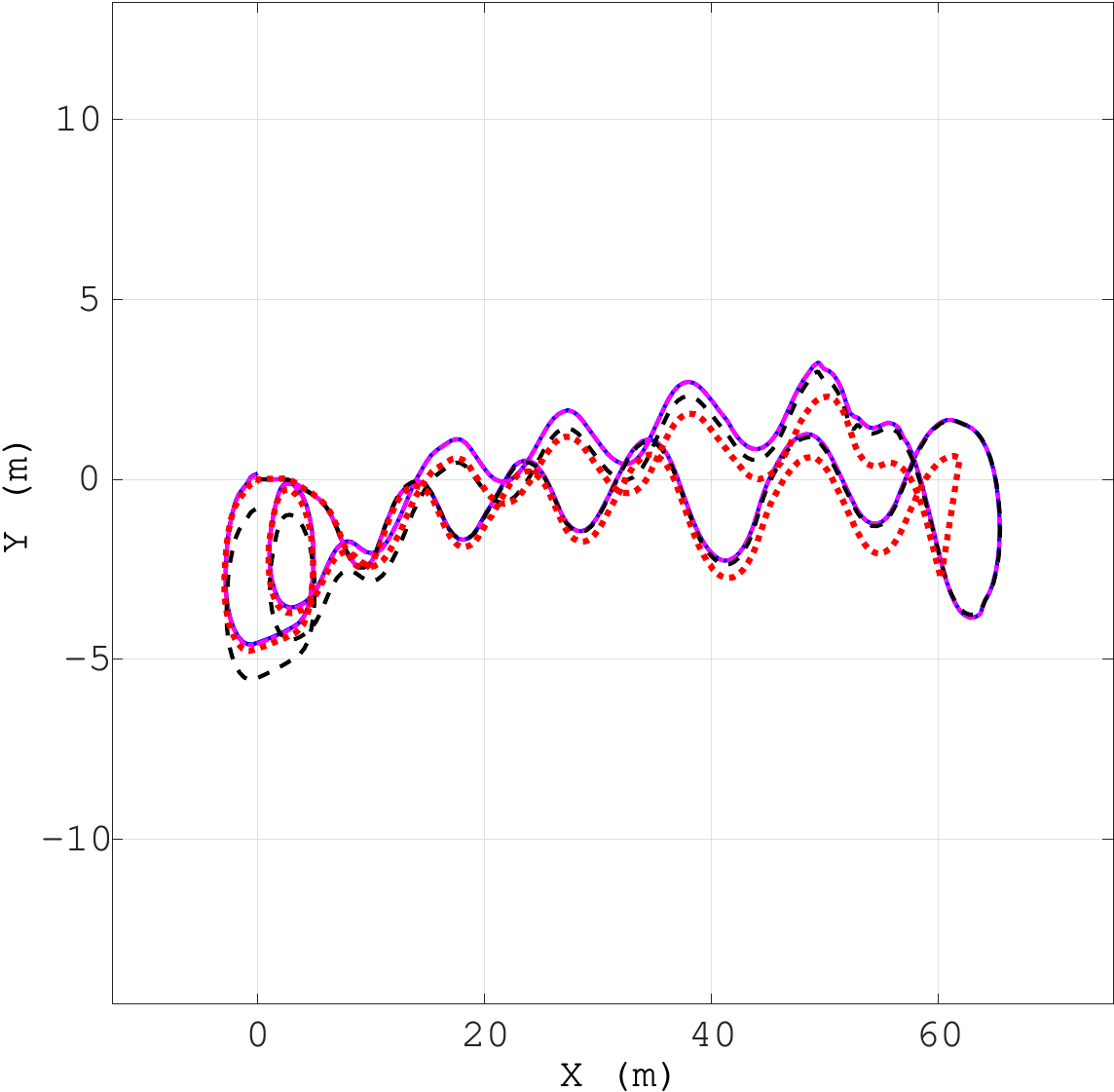} 
	} \hfill
	\subfigure[]{ 
		\includegraphics[width=1.6in]{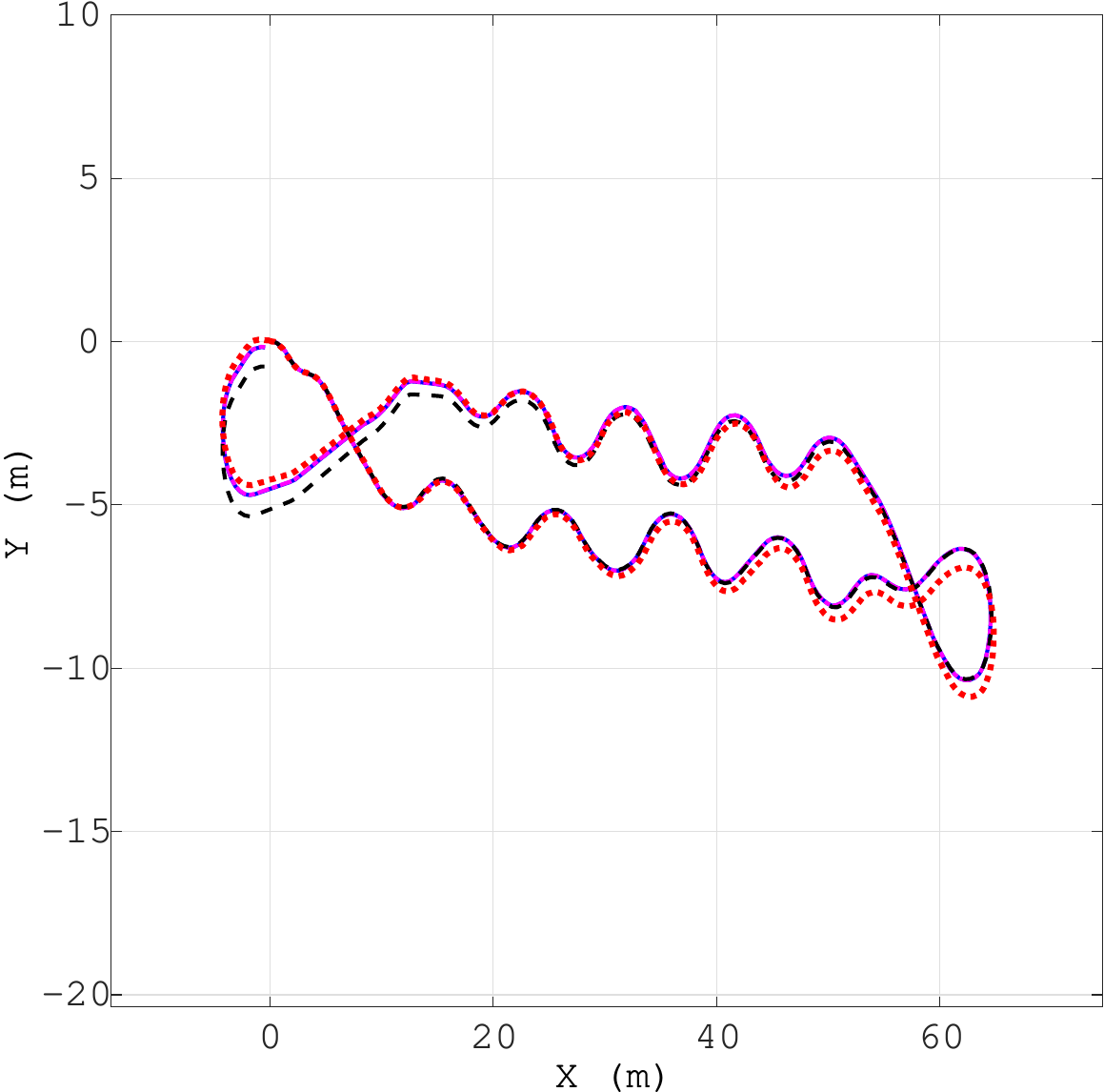} 
	} \hfill \\
	\subfigure[]{ 
		\includegraphics[width=1.6in]{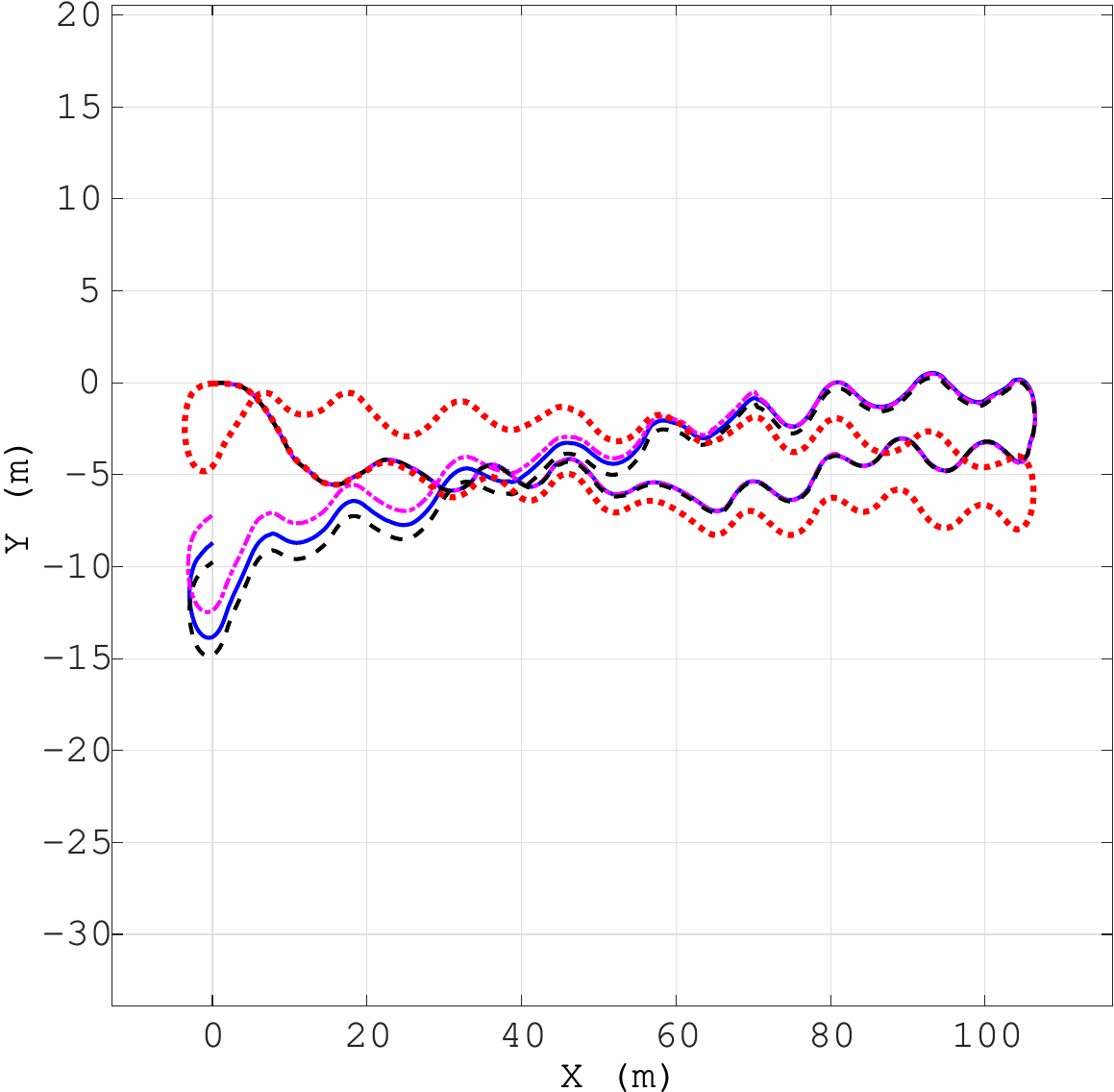} 
	} \hfill
	\subfigure[]{ 
		\includegraphics[width=1.6in]{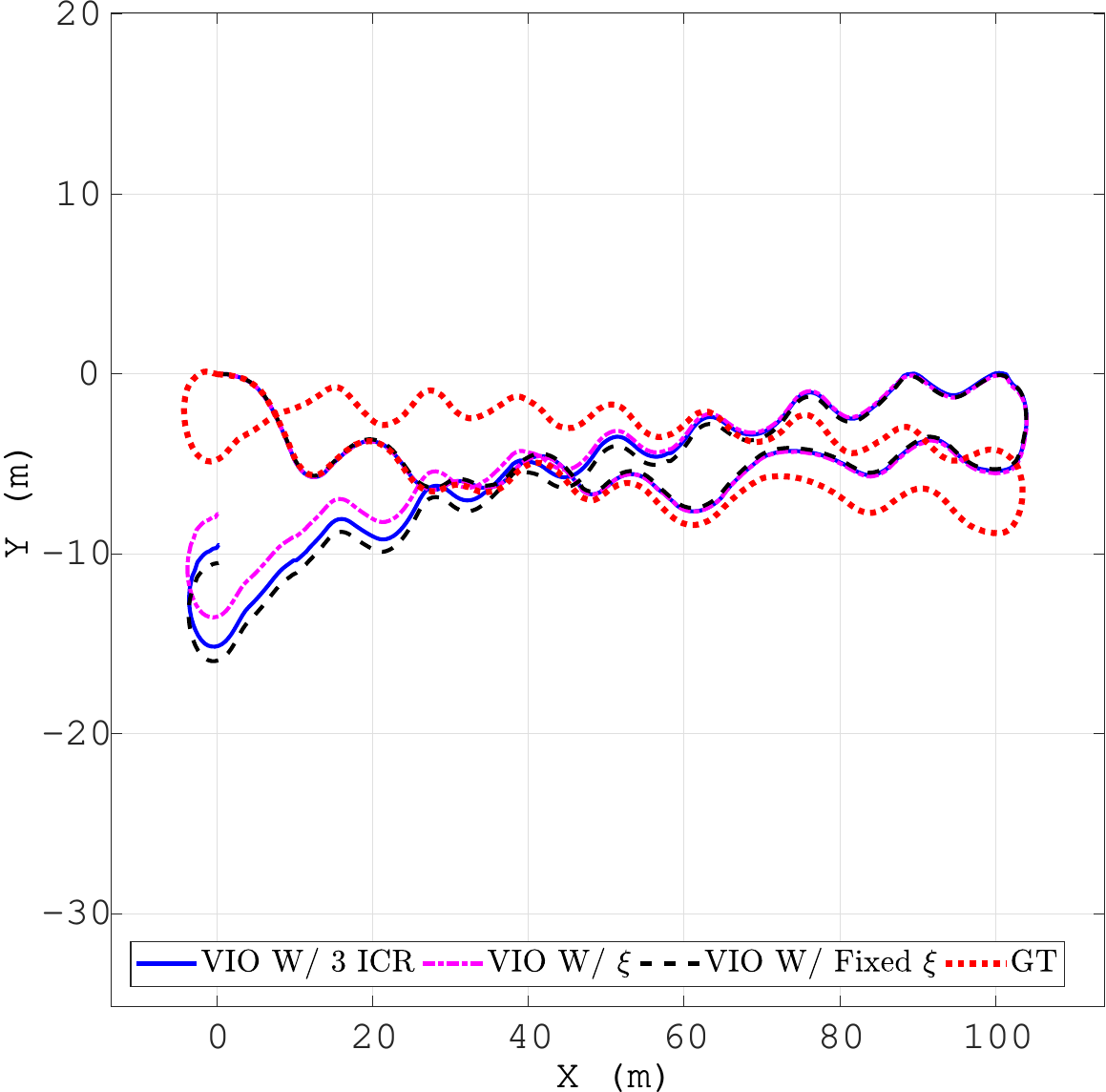} 
	}     
	\caption{In the ablation study, the trajectories of RTK-GPS (ground truth) and the estimated trajectory by the localization method: 1) VIO with online estimating the full kinematic parameters $\boldsymbol \xi$; 2) VIO with online estimating $\boldsymbol{\xi}_{ICR}$ only; 3) VIO with fixed $\boldsymbol \xi$.  
		The skid-steering robot is under four different conditions: (a) normal; (b) carrying a package with the weight around 3 kg; (c) under low tire pressure; (d) carrying a 3-Kg package and with low tire pressure.           
	}
	~\label{fig:traj_gps_disturb}
\end{figure}

\end{document}